\documentclass{article}

\usepackage[final,main]{neurips_2026}

\usepackage[utf8]{inputenc} %
\usepackage[T1]{fontenc}    %
\usepackage[hidelinks]{hyperref}       %
\usepackage{url}            %
\usepackage{booktabs}       %
\usepackage{amsfonts}       %
\usepackage{nicefrac}       %
\usepackage{microtype}      %
\usepackage{xcolor}         %
\usepackage{graphicx}
\usepackage{algorithm}
\usepackage{amsmath}
\usepackage{amssymb,mathrsfs}
\usepackage{amsthm}
\usepackage{float}
\usepackage{algpseudocode}
\usepackage{tikz}
\usetikzlibrary{arrows.meta, positioning}

\newcommand{\w}{\boldsymbol{\theta}}
\newcommand{\q}{\mathbf{z}}
\newcommand{\dir}{\mathbf{g}}
\newcommand{\mom}{\mathbf{m}}
\newcommand{\acc}{\mathbf{v}}
\newcommand{\binit}{v_{\mathrm{init}}}
\DeclareMathOperator{\polylog}{polylog}
\DeclareSymbolFont{orderletters}{U}{eur}{m}{n}
\DeclareMathSymbol{\orderOmega}{\mathord}{orderletters}{"0A}

\newtheorem{assumption}{Assumption}[section]
\newtheorem{thm}{Theorem}[section]
\newtheorem{lemma}{Lemma}[section]

\newtheorem{claim}{Claim}[section]

\newtheorem{proposition}{Proposition}[section]

\title{Adam under Generalized Smoothness with Second-Moment-Type Stochastic Gradients}

\author{%
  Ruinan Jin\\
  Mohamed bin Zayed University of Artificial Intelligence\\
  \texttt{jinruinan3@gmail.com}\\
  \And
  Difei Cheng\\
  Aerospace Information Technology University\\
  \And
  Ling Chen\\
  The Ohio State University\\
  \AND
  Jun Luo\\
  The Ohio State University\\
  \And
  Hao Zhou\\
  JD.com, Inc.\\
  \And
  Youzhi Zhang\thanks{Corresponding author.}\\
  Centre for Artificial Intelligence and Robotics,\\
  Hong Kong Institute of Science \& Innovation,\\
  Chinese Academy of Sciences\\
  \texttt{youzhi.zhang@cair-cas.org.hk}
}

\begin{document}

\maketitle

\begin{abstract}
It has been widely observed that Adam can remain stable even when the objective function deviates significantly from global smoothness. However, under the generalized smoothness framework, existing theoretical analyses typically rely on strong tail assumptions on stochastic gradients, such as almost-sure boundedness or sub-Gaussianity. Whether one can establish the convergence of Adam on generalized smooth objectives under only second moment information on the stochastic gradients, without imposing such strong concentration assumptions, was explicitly identified as an important open direction by \citet{DBLP:conf/nips/LiRJ23}. This paper gives an affirmative answer to this question under fairly general conditions. Specifically, we prove that such strong tail assumptions are not necessary. Building on the Adam self-normalization framework of \citet{jin2026adam}, which was developed for classical smoothness and bounded variance, we extend the stopping-time and de-preconditioning strategy to the \(L_0\)--\(L_p\) generalized smoothness condition and a generalized second moment ABC condition. This extension shows that, even when the stochastic-gradient condition provides only second moment information and may grow along the trajectory, the stochastic trajectory of Adam remains in a locally well-behaved smoothness region, with stretched-exponential tail decay under bounded variance and global smoothness. As a consequence, we establish high-probability convergence rate guarantees over the full range \(p<2\), with a confidence dependence of order \(\delta^{-1/2}\), while the stepsize prefactor depends on \(\delta\) only through a single logarithmic factor. Furthermore, we construct a hard instance proving that, under only second-moment information on the stochastic gradients, this \(\delta^{-1/2}\)-type confidence dependence is sharp. Finally, in the more favorable regime \(p<1\), we combine the above trajectory control with polynomial-growth estimates on rare events to further obtain convergence rate guarantees in expectation.
\end{abstract}
\section{Introduction}
\label{sec:intro}

Adaptive gradient methods are a standard tool in large-scale stochastic optimization, and Adam is arguably the most widely used example \citep{DBLP:journals/corr/KingmaB14,DBLP:conf/iclr/ReddiKK18,DBLP:journals/tmlr/DefossezBBU22,DBLP:conf/nips/LiRJ23}. Its practical robustness is usually attributed to two interacting mechanisms: momentum, which averages stochastic directions, and coordinatewise normalization, which rescales the update by a running average of squared gradients. These mechanisms are especially relevant when the objective is far from globally smooth. A common model for this behavior is \(L_0\)--\(L_p\) generalized smoothness \citep{DBLP:conf/nips/CrawshawLO0Z22,DBLP:conf/iclr/ZhangHSJ20,chen2023generalized,li2023convex,gorbunov2024methods}, where the local variation of the gradient is allowed to grow with the size of the gradient itself.

The main obstacle in this regime is not merely to prove a descent inequality. Since the local smoothness scale depends on the current gradient norm, one first has to show that the stochastic trajectory does not enter a region where the objective becomes too irregular. For Adam this is delicate: the algorithm normalizes stochastic gradients by \(\acc_t\), but \(\acc_t\) itself contains the raw squared stochastic gradients. Under bounded or sub-Gaussian noise, one can often control this feedback loop through direct concentration. Under a second-moment assumption alone, rare but very large stochastic-gradient spikes remain possible, so the raw gradient sequence need not have useful high-probability bounds.

This paper asks whether Adam can nevertheless be analyzed under such bounded-variance-type information. This is not merely a technical variant of existing results: after proving Adam guarantees under stronger bounded-gradient or sub-Gaussian assumptions, \citet{DBLP:conf/nips/LiRJ23} explicitly identified the bounded-variance setting as a challenging and important open direction for Adam under generalized smoothness. The answer we give is positive, but the result has a different structure from analyses based on strong tails. Our analysis builds on the self-normalized stopping-time framework of \citet{jin2026adam}, developed for Adam under classical smoothness and bounded variance, and extends it to generalized smoothness and generalized ABC second-moment growth. In this broader regime, self-normalization still controls the scaled quantities that actually move the iterate, not the raw stochastic gradients themselves. This distinction gives polylogarithmic trajectory localization, but controlling the adaptive denominator \(\acc_t\) in the final de-preconditioning step necessarily incurs a \(\delta^{-1/2}\) factor. Moreover, unlike the classical-smooth bounded-variance result of \citet{jin2026adam}, the local-smoothness and noise-growth restrictions here lead to a stepsize prefactor that depends on the confidence level only through a single logarithmic factor. We show through a hard instance that the resulting \(\delta^{-1/2}\) factor is not a proof artifact in the stated high-probability stationarity guarantee.

\paragraph{Contributions.}
For the horizon-dependent theoretical calibration \(\beta_2=1-1/T\), we prove the following results for Adam under \(L_0\)--\(L_p\) generalized smoothness and a generalized second-moment ABC condition, using the Adam self-normalization framework of \citet{jin2026adam} as a starting point.

\begin{itemize}
\item \emph{Trajectory localization without tail assumptions.} We extend the self-normalized stopping-time analysis to prove that, for any confidence level \(\delta\), both the auxiliary energy and the gradient scale along the Adam trajectory are bounded by \(\polylog(1/\delta)\) with probability at least \(1-\delta\). Under bounded variance and global smoothness, a fixed stepsize gives stretched-exponentially decaying tail probability.

\item \emph{High-probability convergence for \(p<2\).} The localization estimate yields a preconditioned energy bound. After converting it to an unweighted stationarity guarantee, we obtain
\[
\frac{1}{T}\sum_{t=1}^{T-1}\|\nabla f(\w_t)\|^2
\le
\widetilde{\mathcal{O}}\!\left(\frac{d}{\sqrt{\delta T}}+\frac{d^2}{T}\right)
\]
with probability at least \(1-\delta\). The additional \(\delta^{-1/2}\) factor comes from controlling the maximum adaptive denominator.

\item \emph{Sharpness of the confidence dependence.} We construct a one-dimensional hard instance showing that, under second-moment assumptions alone, the \(\delta^{-1/2}\) dependence in the high-probability stationarity rate cannot generally be replaced by a polylogarithmic dependence. In the construction, a rare spike in the stochastic gradient raises the adaptive denominator and keeps the iterates in a nonstationary ramp region for a constant fraction of the horizon.

\item \emph{Expectation convergence and hard instances.} In the more structured range \(p<1\), generalized smoothness implies a polynomial distance-growth estimate for the true gradient. This allows the rare-event terms to be integrated, giving
\[
\frac1T\sum_{t=1}^{T-1}\mathbb E\|\nabla f(\w_t)\|^2
\le
\mathcal{O}\!\left(\frac{d(1+\log T)}{\sqrt T}+\frac{d^2(1+\log T)^2}{T}\right).
\]
For \(1\le p<2\), we construct a fixed class of one-dimensional hard instances satisfying the same assumptions for which the calibrated Adam family has a worst-case expected squared-stationarity lower bound of \(\mathcal{\orderOmega}(T^{-1/3})\), uniformly over deterministic, possibly horizon-dependent, base stepsize prefactors. Thus the above expectation rate cannot extend uniformly to the full generalized ABC class for this Adam family (Proposition~\ref{prop:expected-stationarity-informal}).
\end{itemize}

\paragraph{Organization.}
Section~\ref{sec:related-work} discusses the relation to existing analyses of Adam, generalized smoothness, and high-probability stochastic optimization. Section~\ref{sec:preliminaries} formulates the problem, states the assumptions and Algorithm~\ref{alg:adam}, and fixes notation. Section~\ref{sec:results} states the main results. Section~\ref{sec:overview} outlines the upper- and lower-bound proofs, while full proofs are deferred to the appendix.

\section{Related Work}
\label{sec:related-work}
\citet{DBLP:conf/nips/LiRJ23} analyze Adam under generalized smoothness
with bounded or sub-Gaussian noise.
\citet{zhang2024convergence} and \citet{wang2024convergence} obtain
convergence rates under \((L_0,L_1)\)-type smoothness with affine noise
variance, for coordinatewise and scalar normalization, respectively.
Theorem~\ref{thm:coefficientwise-expectation-informal} is an expectation
result only for \(p<1\), with second moments that may grow with the
function gap, while Theorem~\ref{thm:coefficientwise-convergence-informal}
is a high-probability result for all \(p<2\).
The closest high-probability analysis under generalized smoothness is
\citet{DBLP:journals/corr/abs-2402-03982}, which allows exponents below
\(2\) and obtains a polylogarithmic dependence on \(1/\delta\) under an
exponential-moment condition on the noise.
This condition is much stronger than Assumption~\ref{ass:abc}, which
bounds only conditional second moments, so the two confidence
dependences are not directly comparable.
Under second moments alone, the factor \(\delta^{-1/2}\) cannot in
general be replaced by a polylogarithmic one
(Proposition~\ref{prop:delta-sharpness-informal}). \citet{jin2024comprehensive} establish
convergence under ABC second moments. \citet{jin2026adam} obtain
polynomial-confidence guarantees under bounded conditional variance.
We study generalized smoothness with generalized ABC growth, retaining
coordinatewise normalization and distinguishing high-probability
guarantees from expectation rates.
Here exponential memory controls the adaptive denominator's persistence.
Appendix~\ref{sec:extended-related-work}
discusses the broader literature on adaptive methods, noise models,
localization, clipping, and the comparison scopes of stochastic
optimization lower bounds.

\section{Preliminaries}\label{sec:preliminaries}

Throughout, let $f:\mathbb{R}^d \to \mathbb{R}$ be a differentiable function, and let $\{\zeta_t\}_{t\ge1}$ denote a sequence of independent random variables driving the stochastic-gradient oracle.  
All random quantities are defined on a common probability space $(\Omega,\mathscr{F},\mathbb{P})$.  
Expectation with respect to $\mathbb{P}$ is denoted by $\mathbb{E}[\cdot]$.

\subsection{Adaptive Moment Estimation (Adam)}
Adam \citep{DBLP:journals/corr/KingmaB14,DBLP:conf/iclr/ReddiKK18,wang2024closing,jin2024comprehensive} is an adaptive first-order optimization method widely employed in stochastic settings.  
It maintains a denominator-aware first-moment variable and an exponentially weighted average of squared gradients.  
These moment estimates are then used to construct coordinatewise step sizes that adapt to the local geometry of the stochastic noise.

\noindent Formally, let $f:\mathbb{R}^d\to\mathbb{R}$ be a differentiable function and let \(\mathsf G\) be a stochastic-gradient oracle driven by independent random variables \(\{\zeta_t\}_{t\ge1}\).  
At each iteration $t\ge1$, the stochastic gradient is denoted by
\[
\dir_t = \mathsf G(\w_t;\zeta_t).
\]
The algorithm maintains two auxiliary variables $\mom_t,\acc_t\in\mathbb{R}^d$ and computes the next iterate through a coordinatewise scaling of the momentum term.  
The complete procedure is summarized in Algorithm~\ref{alg:adam}.

\begin{algorithm}
\caption{Adam}\label{alg:adam}
\begin{algorithmic}[1]
\State \textbf{Input:} Stochastic-gradient oracle $\mathsf{G}$; base step size $\alpha>0$; initial point $\w_1\in\mathbb{R}^d$;
  $\mom_0=\mathbf{0}$; $\acc_0=\binit\,\mathbf{1}$ with $\binit>0$;
  $\beta_1\in[0,1)$; $\beta_{2}\in[0,1)$; $\epsilon>0$; number of iterations $T$.
\For{$t=1,\ldots,T-1$}
  \State Sample $\zeta_t$ and compute $\dir_t \gets \mathsf{G}(\w_t;\zeta_t)$.
  \State Update the second-moment estimate:
  \(
  \acc_t \gets \beta_{2}\acc_{t-1}+(1-\beta_{2})(\dir_t\odot \dir_t).
  \)
  \State Update the first-moment estimate:
  \(
  \mom_t \gets \beta_1 \mom_{t-1}+(1-\beta_1)\dir_t.
  \)
  \State Form the coordinatewise effective step size:
  \(
  \lambda_t \gets \alpha\cdot(\sqrt{\acc_t}+\epsilon)^{-1},
  \)
  where the inverse is taken element-wise.
  \State Update the iterate:
  \(
  \w_{t+1}\gets \w_t-\lambda_t\odot \mom_t.
  \)
\EndFor
\State \textbf{Output:} iterate sequence $(\w_t)_{t=1}^{T}$.
\end{algorithmic}
\end{algorithm}
\noindent All arithmetic operations in Algorithm~\ref{alg:adam} are executed componentwise, and the symbol \(\odot\) denotes the Hadamard product.
For notational convenience, we also define the initial coordinatewise preconditioning vector by
\(\lambda_0:=\alpha(\sqrt{\binit}+\epsilon)^{-1}\mathbf 1\), where \(\binit>0\) is the scalar initialization in \(\acc_0=\binit\mathbf 1\).
For a vector \(u\in\mathbb{R}^d\), we write \(u_i\) for its \(i\)-th coordinate; for a time-indexed vector \(u_t\in\mathbb{R}^d\), we write \(u_{t,i}\) for its \(i\)-th coordinate.
\paragraph{Theoretical calibration.}
The results use \(\alpha=\bar\alpha/\sqrt T\) and
\(\beta_2=1-1/T\) for an integer horizon \(T\ge10\), with fixed
\(\beta_1\in[0,1)\), \(\binit>0\) and \(\epsilon>0\).
This finite-horizon calibration, as in \citet{jin2026adam} and
\citet{ghadimi2013stochastic}, fixes both parameters once \(T\) is fixed.
By Lemma~\ref{lem:vt_comparable}, \(\acc_{t,i}\ge\binit/4\) and
\(\lambda_{t,i}\le2\lambda_{k,i}\) for \(0\le k\le t\le T-1\), so the
effective stepsizes are comparable along the whole run.
The lower bound on \(\acc_t\) comes from the retained initialization,
not from the stochastic gradients.
We measure stationarity by
\(\frac1T\sum_{t=1}^{T-1}\|\nabla f(\w_t)\|^2\), the average of the
squared true-gradient norm over the iterates output by
Algorithm~\ref{alg:adam}, not a single final iterate. Time-varying \(\beta_2\) schedules and
restarts are outside the scope of the theorems.

\paragraph{Filtration.}
Let $\mathscr{F}_t=\sigma(\zeta_1,\dots,\zeta_t)$ for $t\ge0$, with $\mathscr{F}_0=\{\emptyset,\Omega\}$ and $\mathscr{F}_\infty=\sigma\big(\cup_{t\ge0}\mathscr{F}_t\big)$.  
The filtration $\{\mathscr{F}_t\}$ encodes the information available up to time $t$.  
Each stochastic gradient $\dir_t$ is $\mathscr{F}_t$-measurable, while $\w_t$ is $\mathscr{F}_{t-1}$-measurable.  
The Euclidean norm $\|\cdot\|$ is used throughout. For nonnegative bases, a zeroth power is interpreted as one, including $0^0=1$.

\subsection{Analytic Assumptions}

We impose the following structural assumptions on the objective function and the stochastic gradient oracle.

\begin{assumption}[Lower boundedness]
\label{ass:nonneg}
The objective \(f:\mathbb R^d\to\mathbb R\) is bounded from below, i.e.,
\[
f^*:=\inf_{\w\in\mathbb R^d} f(\w)>-\infty .
\]
\end{assumption}

\begin{assumption}[\(L_0\)-\(L_p\) smoothness]
\label{ass:smooth}
The function \(f:\mathbb R^d\to\mathbb R\) is continuously differentiable.
There exist constants \(L_0>0\), \(L_p>0\), and a smoothness exponent
\(p\in[0,2)\) such that, for all \(u,u'\in\mathbb R^d\) satisfying
\(\|u-u'\|\le L_p^{-1}\),
\[
\|\nabla f(u)-\nabla f(u')\|
\le
\bigl(L_0+L_p\|\nabla f(u)\|^p\bigr)\|u-u'\|.
\]
\end{assumption}
When \(p=0\), this reduces to global smoothness up to constants. For \(p>0\), the local Lipschitz constant of the gradient is allowed to grow with the gradient norm. This captures objectives whose curvature can be much larger far from stationary regions, and it is precisely why a trajectory-localization argument is needed before applying descent estimates.

\begin{assumption}[Generalized second-moment ABC inequality]
\label{ass:abc}
Let \(\{\mathscr F_t\}_{t\ge0}\) be the natural filtration generated by the
algorithm, and assume that \(\w_t\) is \(\mathscr F_{t-1}\)-measurable.
For each \(t\ge1\), the stochastic gradient \(\dir_t\) is an unbiased estimator
of \(\nabla f(\w_t)\), namely
\[
\mathbb E[\dir_t\mid \mathscr F_{t-1}]
=
\nabla f(\w_t).
\]
Moreover, there exist constants \(A,B,C,\rho_1,\rho_2\ge0\) such that
\[
\mathbb E\!\left[\|\dir_t\|^2\mid \mathscr F_{t-1}\right]
\le
A\bigl(f(\w_t)-f^*\bigr)^{\rho_1}
+
B\|\nabla f(\w_t)\|^{\rho_2}
+
C .
\]
\end{assumption}
This assumption is a second-moment condition, not a tail condition. It permits the conditional variance of the stochastic gradient to depend on both the function gap and the gradient norm, but it does not rule out rare large stochastic-gradient values. This is the main distinction from bounded-gradient, sub-Gaussian, or exponentially concentrated oracle models.
By conditional unbiasedness,
\[
\mathbb E[\|\dir_t\|^2\mid\mathscr F_{t-1}]
=
\|\nabla f(\w_t)\|^2+
\mathbb E[\|\dir_t-\nabla f(\w_t)\|^2\mid\mathscr F_{t-1}].
\]
Hence Assumption~\ref{ass:abc} may equivalently be viewed as a variance-growth
condition up to the deterministic term \(\|\nabla f(\w_t)\|^2\). The classical
bounded-variance assumption
\[
\mathbb E[\|\dir_t-\nabla f(\w_t)\|^2\mid\mathscr F_{t-1}]\le \sigma^2
\]
is recovered in variance form by taking \(A=B=0\) and \(C=\sigma^2\). In the
second-moment form used above, it corresponds to
\(A=0\), \(B=1\), \(\rho_2=2\), and \(C=\sigma^2\).
\section{Results}
\label{sec:results}
Throughout this section, we write
\begin{equation}
\Psi(x):=f(x)-f^*+1.
\label{eq:auxiliary-objective}
\end{equation}
We also use the standard auxiliary sequence
\begin{equation}
\q_1:=\w_1,
\qquad
\q_t:=\frac{\w_t-\beta_1\w_{t-1}}{1-\beta_1},
\qquad t\ge2,
\label{eq:auxiliary-sequence}
\end{equation}
which removes the leading momentum term from the recursion.
The results are organized around the two quantities that must be controlled in Adam: the trajectory itself and the adaptive denominator. Formal auxiliary lemmas, intermediate propositions, and complete proofs are deferred to the appendix.

The main statements below give the convergence orders. The order notation
suppresses fixed problem and initialization parameters other than the
dimension \(d\); \(\widetilde{\mathcal{O}}\)
also suppresses logarithmic factors in \(1/\delta\).
Appendix~\ref{app:formal-results} gives the complete inequalities;
their constants are explicit but not optimized.

To display the confidence dependence of the stepsize, define
\begin{equation}
\eta_{\mathrm{gen}}
:=\mathbf1_{\{p>0\}}
+\sqrt{A+[B-1]_++B\mathbf1_{\{\rho_2\ne2\}}},
\qquad [x]_+:=\max\{x,0\}.
\label{eq:stepsize-excess-main}
\end{equation}

\begin{thm}[High-probability convergence, informal]
\label{thm:coefficientwise-convergence-informal}
Assume Assumptions~\ref{ass:nonneg}--\ref{ass:abc}. Let \(T\ge10\) and
\(\beta_2=1-1/T\). For every \(0<\delta<1\), an admissible prefactor
\(\bar\alpha>0\) can be chosen with
\begin{equation}
\bar\alpha^{-1}
=d\Bigl[\mathcal{O}(1)
+\mathbf 1_{\{\eta_{\mathrm{gen}}>0\}}\,\mathcal{O}\!\left(\log(1/\delta)\right)\Bigr],
\label{eq:stepsize-order-split}
\end{equation}
where \(\eta_{\mathrm{gen}}\) is defined in \eqref{eq:stepsize-excess-main}.
Adam with \(\alpha=\bar\alpha/\sqrt T\) then satisfies, with probability
at least \(1-\delta\),
\[
\max_{1\le t\le T}\Psi(\q_t)
=\mathcal{O}\!\left((1+\log(1/\delta))^{4\mathbf 1_{\{\eta_{\mathrm{gen}}=0\}}}\right),\qquad
\frac1T\sum_{t=1}^{T-1}\|\nabla f(\w_t)\|^2
=\widetilde{\mathcal{O}}\!\left(\frac{d}{\sqrt{\delta T}}+\frac{d^2}{T}\right).
\]
Here \(\Psi\) and \(\q_t\) are defined in
\eqref{eq:auxiliary-objective}--\eqref{eq:auxiliary-sequence}.
\end{thm}
The complete inequalities and admissible stepsizes are given in
Theorem~\ref{thm:coefficientwise-convergence} of
Appendix~\ref{app:formal-results}, where the parameter scale \(\mathcal P\), the calibration scale \(\mathcal R\) and the confidence height \(U_\delta\) are defined. The decomposition
\eqref{eq:stepsize-order-split} describes the upper endpoint of that
stepsize interval; its two order constants are independent of
\(T,\delta,d\), and for \(\eta_{\mathrm{gen}}>0\) the height \(U_\delta\) is independent of \(\delta\).

\begin{claim}[Smooth bounded-variance specialization, informal]
\label{clm:squared-standard-rate-informal}
For \(p=0,A=0,B=1,\rho_2=2\), the same result with a fixed prefactor
\(\bar\alpha>0\) satisfying \(\bar\alpha^{-1}=\mathcal{O}(d)\), as in
\eqref{eq:stepsize-order-split}, gives
\[
\frac1T\sum_{t=1}^{T-1}\mathbb E\|\nabla f(\w_t)\|^2=\mathcal{O}(dT^{-1/2}+d^2T^{-1}).
\]
\end{claim}
Here \(\eta_{\mathrm{gen}}=0\) in \eqref{eq:stepsize-excess-main},
so the logarithmic term in \eqref{eq:stepsize-order-split} vanishes.
The same fixed stepsize prefactor is therefore admissible at every
confidence level. Integrating the quantile bound from
Theorem~\ref{thm:coefficientwise-convergence} therefore gives
\[
\begin{aligned}
\frac1T\sum_{t=1}^{T-1}\mathbb E\|\nabla f(\w_t)\|^2
\le{}&\mathcal{O}(dT^{-1/2})\int_0^1
\frac{(1+\log(16/\delta))^4}{\sqrt\delta}\,d\delta\\
&+\mathcal{O}(d^2T^{-1})\int_0^1(1+\log(16/\delta))^8\,d\delta
=\mathcal{O}(dT^{-1/2}+d^2T^{-1}).
\end{aligned}
\]
The constants in the order terms are independent of \(T,\delta,d\),
and both integrals are finite by the substitution \(u=\log(1/\delta)\).
Claim~\ref{clm:squared-standard-rate} gives the explicit inequality;
its full proof is in Appendix~\ref{app:proof-standard-specialization}.

By Cauchy--Schwarz, Claim~\ref{clm:squared-standard-rate-informal} implies
\[
\frac1T\sum_{t=1}^{T-1}\mathbb E\|\nabla f(\w_t)\|
\le\left(\frac1T\sum_{t=1}^{T-1}\mathbb E\|\nabla f(\w_t)\|^2\right)^{1/2}
=\mathcal{O}(\sqrt d\,T^{-1/4}+dT^{-1/2}).
\]
For the regularized Adam calibration considered here, this recovers the
expected-gradient-norm rate of \citet[Theorem~2]{wang2024closing}.
The implication between these moment criteria is strictly one-way:
an expected norm bound alone does not control the corresponding squared
moment. The algorithm in \citet{wang2024closing} uses an update without
the denominator offset. The squared-gradient expectation rate in
Claim~\ref{clm:squared-standard-rate-informal} also agrees with the rate
obtained by integrating the confidence bound of
\citet[Theorem~1]{jin2026adam}, whose stepsize prefactor is independent of
the confidence level.
Under global smoothness and bounded coordinatewise variance, assumptions
much stronger than Assumptions~\ref{ass:smooth}--\ref{ass:abc},
\citet{li2025adamw} obtain an expected \(\ell_1\) rate without
logarithmic factors; Appendix~\ref{sec:extended-related-work} relates it
to Claim~\ref{clm:squared-standard-rate-informal}, whose standard
specialization carries no \(\log T\) factor.

\begin{thm}[Expected convergence for \(p<1\), informal]
\label{thm:coefficientwise-expectation-informal}
Assume Assumptions~\ref{ass:nonneg}--\ref{ass:abc} with \(p<1\), and
let \(T\ge10\), \(\beta_2=1-1/T\). A prefactor \(\bar\alpha_T>0\) with
\(\bar\alpha_T^{-1}=d[\mathcal{O}(1)+\mathbf 1_{\{\eta_{\mathrm{gen}}>0\}}\mathcal{O}(\log T)]\)
can be chosen so that Adam with \(\alpha=\bar\alpha_T/\sqrt T\) satisfies
\[
\frac1T\sum_{t=1}^{T-1}\mathbb E\|\nabla f(\w_t)\|^2
=\mathcal{O}\!\left(\frac{(1+\log T)^{\mathbf 1_{\{\eta_{\mathrm{gen}}>0\}}}d}{\sqrt T}
+\frac{(1+\log T)^{2\mathbf 1_{\{\eta_{\mathrm{gen}}>0\}}}d^2}{T}\right).
\]
\end{thm}
Theorem~\ref{thm:coefficientwise-expectation} in
Appendix~\ref{app:formal-results} gives the complete inequality and the
polynomial-confidence calibration of the prefactor.

\subsection{Rate obstructions}
\begin{samepage}
\begin{proposition}[Confidence dependence, informal]
\label{prop:delta-sharpness-informal}
Fix \(\bar\alpha\in(0,1]\), \(\binit>0\), and \(\epsilon>0\).
For sufficiently small \(\delta>0\) and sufficiently large \(T\)
depending on \(\delta\), there are one-dimensional instances in the
classical specialization \(p=0\), \(A=0\), \(B=1\), \(\rho_2=2\)
of Assumptions~\ref{ass:nonneg}--\ref{ass:abc} for which Adam with
\(\beta_1=0\), \(\beta_2=1-1/T\), and
\(\alpha=\bar\alpha/\sqrt T\) satisfies
\[
\frac1T\sum_{t=1}^{T-1}|\nabla f(\w_t)|^2
=\mathcal{\orderOmega}\!\left(\frac1{\sqrt{\delta T}}\right)
\quad\text{with probability }\mathcal{\orderOmega}(\delta).
\]
The problem constants, initial objective gaps, and constants hidden in
\(\mathcal{\orderOmega}\) are uniform in \(T,\delta\).
This obstruction also holds for fixed prefactors admitted by
Theorem~\ref{thm:coefficientwise-convergence}.
\end{proposition}
Proposition~\ref{prop:delta_sharpness} gives the precise probability,
parameter ranges, and comparison with the calibrated upper bound.
\end{samepage}

\begin{samepage}
\begin{proposition}[Expected stationarity for \(1\le p<2\), informal]
\label{prop:expected-stationarity-informal}
Fix \(p\in[1,2)\), \(\beta_1\in[0,1)\), \(\binit>0\), and
\(\epsilon>0\). There is a fixed class \(\mathcal C_p\) of
one-dimensional instances satisfying
Assumptions~\ref{ass:nonneg}--\ref{ass:abc}, with common parameters
including \(A=0\) and \(\rho_2=3\), and uniform bounds on
\(|\theta_1|\), \(f(\theta_1)-f^*\), and \(|f'(\theta_1)|\),
such that Adam with
\(\beta_2=1-1/T\) and \(\alpha=\bar\alpha_T/\sqrt T\) satisfies,
for all sufficiently large \(T\),
\[
\inf_{\bar\alpha_T>0}\;
\sup_{(f,\mathsf G,\theta_1)\in\mathcal C_p}
\frac1T\sum_{t=1}^{T-1}\mathbb E|f'(\theta_t)|^2
\ge\mathcal{\orderOmega}(T^{-1/3}).
\]
The infimum is over deterministic, possibly horizon-dependent,
prefactors; the class and implicit constant are independent of
\(T,\bar\alpha_T\). Consequently, the
\(\operatorname{polylog}(T)/\sqrt T\) expectation rate cannot hold
uniformly for this Adam family on the full generalized ABC class.
\end{proposition}
Proposition~\ref{prop:expected_stationarity_lower_bound} specifies the
common class parameters, the lower-bound coefficient, and the horizon range.
\end{samepage}

\section{Proof Sketch}
\label{sec:overview}

We summarize the main ideas behind the proofs. The central difficulty is to obtain high-probability trajectory control without assuming that the raw stochastic gradients concentrate. Following the self-normalized stopping-time viewpoint of \citet{jin2026adam}, the proof focuses on quantities that are normalized by Adam's adaptive denominator, and only later converts the resulting preconditioned estimates into the usual stationarity measure.
Two growth terms enter every descent estimate. Under Assumption~\ref{ass:smooth} the curvature coefficient \(L_0+L_p\|\nabla f(\w_t)\|^p\) varies along the random trajectory, and under Assumption~\ref{ass:abc} the conditional second moment is bounded by \(C+A(f(\w_t)-f^*)^{\rho_1}+B\|\nabla f(\w_t)\|^{\rho_2}\), which depends on the same trajectory. Bounding both requires localization of the iterates, while the descent estimate that yields localization requires both bounds. We break it with the augmented energy \(\mathcal L_t=\Psi(\q_t)+E_t\) from \eqref{eq:overview-energy}, stopped at the exit index \(\sigma_U\) from \eqref{eq:coefficientwise-stopping-index}. Before exit, Lemmas~\ref{lem:gradient-gap-growth} and~\ref{lem:barf_comparison} turn both growth terms into deterministic functions of \(U\), the scales of \eqref{eq:coefficientwise-local-scales}--\eqref{eq:coefficientwise-calibration-scale}, and the pathwise logarithmic bounds of Lemma~\ref{lem:stopped-normalized-energies} together with the stopped martingale estimates close with conditional second moments only (Lemma~\ref{lem:stopped-energy-localization}). The auxiliary sequence, the logarithmic energy lemma, and the de-preconditioning step (Appendix~\ref{app:de-preconditioning}) follow \citet{jin2026adam}; the augmented energy with its gradient-energy telescope, the height-dependent calibration \(\mathcal R(U)\), the conditional-to-direct comparison (Lemma~\ref{lem:cond_to_direct}) with Freedman's inequality \eqref{eq:freedman-tool} before the crossing index \(\tau\) from \eqref{eq:coefficientwise-log-crossing}, the integration of the tail bounds over the confidence level for \(p<1\) (Appendix~\ref{app:proof-expectation}), and the hard instances of Section~\ref{subsec:lower-bound-sketch} are specific to the present setting.

\subsection{Upper Bounds}
\label{subsec:upper-bound-sketch}

\paragraph{Auxiliary sequence and descent.}
The auxiliary sequence \(\q_t\) from \eqref{eq:auxiliary-sequence}
removes the direct momentum lag:
\[
\q_{t+1}-\q_t=-\lambda_t\odot\dir_t
+\frac{\beta_1}{1-\beta_1}
(\lambda_{t-1}-\lambda_t)\odot\mom_{t-1}.
\]
The second term records the changing denominator. Pair the auxiliary
objective \(\Psi\) from \eqref{eq:auxiliary-objective} with the
predictable gradient energy
\begin{equation}
E_t=\sum_i\lambda_{t-1,i}(\nabla f(\w_t))_i^2,\qquad
\mathcal L_t=\Psi(\q_t)+E_t.
\label{eq:overview-energy}
\end{equation}
The gradient-energy difference in \eqref{eq:overview-energy}
cancels the leading true-gradient
part of the same-sample denominator error. The remaining localized
descent has negative drift \(-E_t/8\), three centered terms, and
seven nonnegative residuals; \eqref{eq:localized-augmented-descent}
gives the explicit decomposition. Its coefficients depend on the
curvature and noise envelope along the trajectory.

\paragraph{Stopped preconditioned energy.}
The first ingredient turns objective localization into gradient and
noise control.
\begin{lemma}[Function-gap control, informal]
\label{lem:overview-gap-control}
Under Assumptions~\ref{ass:nonneg}--\ref{ass:smooth}, with
\(\Psi\) from \eqref{eq:auxiliary-objective},
\[
\|\nabla f(x)\|
=\mathcal{O}\!\left(\Psi(x)^{\max\{1,(2-p)^{-1}\}}\right).
\]
If \(\|x-y\|\le L_p^{-1}\), then
\[
\Psi(x)=\mathcal{O}\!\left(\Psi(y)^{\max\{1,p/(2-p)\}}\right).
\]
The constants depend only on \(L_0,L_p,p\).
\end{lemma}
The complete bounds are Lemmas~\ref{lem:gradient-gap-growth}
and~\ref{lem:barf_comparison} in Appendix~\ref{app:auxiliary}.
The first follows by taking a short step against the gradient
and comparing its decrease with \(f(x)-f^*\). The second integrates
local smoothness between nearby points.

For \(U\ge1\), stop at
\begin{equation}
\sigma_U=\inf\{1\le t\le T:\Psi(\q_t)>U\},\qquad
\inf\varnothing=\infty.
\label{eq:overview-stopping}
\end{equation}
The calibration keeps \(\w_t,\q_t\) within the comparison radius.
Lemma~\ref{lem:overview-gap-control} then bounds the true gradient
before \(\sigma_U\) from \eqref{eq:overview-stopping}
by a polynomial in \(U\); the ABC condition
gives a polynomial conditional second-moment envelope.
The indicator of \(t<\sigma_U\) is measurable before query \(t\),
so stopping preserves the centered terms' zero conditional means.
Summing the descent reduces localization to controlling the
residuals and martingales in the energy \eqref{eq:overview-energy}.

\paragraph{Self-normalization without raw-gradient tails.}
Large oracle outputs are controlled through their normalized
contribution to the update.
\begin{lemma}[Normalized energy, informal]
\label{lem:overview-normalized-energy}
For Algorithm~\ref{alg:adam} with \(T\ge10\), \(\beta_2=1-1/T\),
and \(\alpha=\bar\alpha/\sqrt T\), for any stopping index \(\sigma\),
\[
\sum_{t<T\wedge\sigma}\left(
\|\lambda_t\odot\dir_t\|^2+
\|\lambda_t\odot\mom_t\|^2+
\|\lambda_{t-1}\odot\mom_{t-1}\|^2\right)\le
\mathcal{O}(\bar\alpha^2)\sum_i
\log\left(1+\frac{\sum_{t<T\wedge\sigma}\dir_{t,i}^2}
{T\binit}\right).
\]
The bound is pathwise.
\end{lemma}
Lemma~\ref{lem:stopped-normalized-energies} is the formal statement.
The accumulator comparison in Lemma~\ref{lem:vt_comparable}
allows logarithmic telescoping: a large sample raises both numerator
and denominator, so its cumulative charge grows logarithmically.
Expanding momentum into its geometric weights gives the same
control for current and lagged momentum. Positive stepsize
variation is bounded separately by \eqref{eq:direct-variation-bound}.

Before \(\sigma_U\) from \eqref{eq:overview-stopping}, the ABC envelope
and Markov's inequality control the logarithm on the right.
Lemma~\ref{lem:cond_to_direct} and the direct variation bound
\eqref{eq:direct-variation-bound} control the predictable compensator
of the positive stepsize variation.
Freedman's inequality controls the centered descent terms.
For the momentum-variation martingale, a predictable logarithmic
crossing includes the first crossing update, whose size is paid
for by the preceding normalized momentum energy.
Thus no tail assumption on the raw gradient is needed.

\paragraph{From stopped energy to localization.}
The normalized-energy bounds close the stopped descent.
\begin{lemma}[Localization and preconditioned energy, informal]
\label{lem:overview-localization}
Under Assumptions~\ref{ass:nonneg}--\ref{ass:abc} and the calibration
of Theorem~\ref{thm:coefficientwise-convergence-informal}, with
probability at least \(1-\delta/2\),
\[
\max_{1\le t\le T}\Psi(\q_t)+\sum_{t=1}^{T-1}E_t
=\mathcal{O}\!\left((1+\log(1/\delta))^{4\mathbf 1_{\{\eta_{\mathrm{gen}}=0\}}}\right),
\]
where \(\Psi,\q_t,E_t\) are defined in
\eqref{eq:auxiliary-objective}, \eqref{eq:auxiliary-sequence},
and \eqref{eq:overview-energy}.
\end{lemma}
The precise stopped inequality and calibration appear in
Lemma~\ref{lem:stopped-energy-localization}.
For the stopping height \(U\) in \eqref{eq:overview-stopping},
take the order in Lemma~\ref{lem:overview-localization}; if \(\eta_{\mathrm{gen}}>0\),
the logarithmic factor in \eqref{eq:stepsize-order-split} absorbs the logarithms. Absorb the centered
sums and cumulative residuals into this height:
\[
\mathcal L_{\sigma_U\wedge T}
+\frac18\sum_{t<\sigma_U\wedge T}E_t\le U/2.
\]
The energy and stopping time are from
\eqref{eq:overview-energy}--\eqref{eq:overview-stopping}.
Since \(\mathcal L_t\ge\Psi(\q_t)\), a first exit before or at
\(T\) would make the left side exceed \(U\), a contradiction.
The stopping never activates, and the same inequality bounds the
full preconditioned gradient energy.

\paragraph{From localization to stationarity.}
On the localization event, a second Markov bound on the stopped
centered-noise energy gives,
with additional failure probability at most \(\delta/2\),
\[
\max_{t<T,i}\acc_{t,i}
\le \binit+\frac2T\sum_{s=1}^{T-1}\|\nabla f(\w_s)\|^2
+\frac{\operatorname{poly}(U)}{\delta},
\]
where \(U\) is the height from \eqref{eq:overview-stopping}
and the polynomial has coefficients depending only on the fixed
problem parameters. Using this maximum to remove the coordinatewise
weights in \(E_t\) from \eqref{eq:overview-energy}, and applying
Lemma~\ref{lem:overview-localization}, yields the schematic bound
\[
\frac1T\sum_{t=1}^{T-1}\|\nabla f(\w_t)\|^2
\le \mathcal{O}\!\left(\frac{U}{\bar\alpha\sqrt T}\right)
\left(1+
\sqrt{\frac1T\sum_{s=1}^{T-1}\|\nabla f(\w_s)\|^2}
+\sqrt{\frac{\operatorname{poly}(U)}{\delta}}\right).
\]
Young's inequality absorbs the square-root gradient average into
the left side. Substituting this height and the prefactor
in \eqref{eq:stepsize-order-split} gives
Theorem~\ref{thm:coefficientwise-convergence-informal};
Appendix~\ref{app:de-preconditioning} supplies the complete inequalities.
The square root of the noise-energy threshold introduces
the \(\delta^{-1/2}\) confidence factor.

\paragraph{Expectation bound for \(p<1\).}
The confidence-dependent calibration requires control of the most
extreme trajectories. The crucial structural input is a distance
bound on the true gradient.
\begin{lemma}[Polynomial distance growth, informal]
\label{lem:overview-distance-growth}
Under Assumption~\ref{ass:smooth}, if \(0\le p<1\), then for every
reference point \(x_{\mathrm{ref}}\),
\[
\|\nabla f(x)\|
=\mathcal{O}\!\left(
(1+\|x-x_{\mathrm{ref}}\|)^{1/(1-p)}
\right),\qquad x\in\mathbb R^d.
\]
The constant depends only on \(L_0,L_p,p\) and
\(\|\nabla f(x_{\mathrm{ref}})\|\).
\end{lemma}
Lemma~\ref{lem:gradient_distance_poly} gives the complete inequality.
Its proof applies the concavity of \(u^{1-p}\) to gradient increments
along a partition of the segment from \(x_{\mathrm{ref}}\) to \(x\).
This converts local generalized smoothness into a global polynomial
distance estimate.

With the prefactor \(\bar\alpha_T\) from
Theorem~\ref{thm:coefficientwise-expectation-informal}, Adam's normalized
displacement satisfies
\[
\|\w_{t+1}-\w_t\|\le2\sqrt d\,\bar\alpha_T,\qquad
\max_{t\le T}\|\w_t-\w_1\|\le2\sqrt d\,\bar\alpha_T T,
\]
as shown in \eqref{eq:coefficientwise-momentum-displacement}.
Combining this with Lemma~\ref{lem:overview-distance-growth},
at \(x_{\mathrm{ref}}=\w_1\), gives
\[
\frac1T\sum_{t=1}^{T-1}\|\nabla f(\w_t)\|^2
=\mathcal{O}\!\left((1+\sqrt d\,\bar\alpha_T T)^{2/(1-p)}\right).
\]
Fix the prefactor at the polynomial-confidence cutoff prescribed by
Theorem~\ref{thm:coefficientwise-expectation}. The same run satisfies
the high-probability bound at every larger confidence level.
Integrating those quantiles costs only
\(\mathcal{O}(d(1+\log T)/\sqrt T+d^2(1+\log T)^2/T)\).
The displayed pathwise bound controls all remaining quantiles:
their interval has polynomially small length, chosen so that its
product with this fixed power of \(T\) is \(o(T^{-1/2})\).
This gives Theorem~\ref{thm:coefficientwise-expectation-informal}.
The exact cutoff, integral bound, and rare-event term appear in
\eqref{eq:coefficientwise-tail-calibration}--\eqref{eq:coefficientwise-expectation-explicit}.

\subsection{Hard Instance}
\label{subsec:lower-bound-sketch}

\paragraph{Hard instance for confidence dependence.}
For Proposition~\ref{prop:delta-sharpness-informal}, fix
\(\beta_1=0\) and \(\bar\alpha\in(0,1]\), with \(T,\delta\) in the
proposition's range.
The one-dimensional objective has derivative \(h=f'\) that vanishes
at \(\theta_1=0\) and has a positive plateau to its left. Its scales are
\begin{equation}
\chi_\star=(\delta T)^{-1/4},\qquad
H=\sqrt{\frac{T}{2\delta}},\qquad
L=16\bar\alpha\sqrt8\,(\delta T)^{1/4}.
\label{eq:overview-confidence-scales}
\end{equation}
With the scales in \eqref{eq:overview-confidence-scales},
the plateau occupies \([-L,-\bar\alpha/8]\), with height
\(\chi_\star\). The oracle adds independent noise equal to \(+H\)
or \(-H\), each with probability \(\delta/T\), and zero otherwise.
Its noise variance is one.

With probability \(\mathcal{\orderOmega}(\delta)\), there is exactly
one positive spike in the first half of the run and no other nonzero
noise before \(T\). That spike moves the iterate left into the
plateau and raises the accumulator to order \(\delta^{-1}\).
Because \(\beta_2=1-1/T\), its contribution persists throughout
the horizon. Later steps follow the positive true gradient, but
their leftward displacements are
\(\mathcal{O}(\bar\alpha\delta^{1/4}T^{-3/4})\). The width \(L\) in
\eqref{eq:overview-confidence-scales} keeps a constant fraction of
the iterates on the plateau. Their squared gradients have size
\((\delta T)^{-1/2}\), giving the confidence obstruction.
Figure~\ref{fig:main-confidence-dynamics} shows both effects;
Proposition~\ref{prop:delta_sharpness} and
Appendix~\ref{app:confidence-lower-proof} give the complete construction.
\raggedbottom
\begin{figure}[H]
\centering
\begin{minipage}[t]{0.49\linewidth}
\centering
\textbf{\footnotesize (a) Entry and slow descent}\par\smallskip
\begin{tikzpicture}[x=0.94cm,y=0.68cm,font=\fontsize{8}{9}\selectfont,
  line cap=round,>=Stealth]
\path[use as bounding box] (-0.10,-1.24) rectangle (7.00,2.85);
\fill[blue!6] (2.15,0) rectangle (5.00,1.55);
\draw[->,black!65] (0.10,0) -- (6.70,0) node[below] {$\theta$};
\draw[->,black!65] (0.35,-0.03) -- (0.35,2.35) node[above] {$h(\theta)=f'(\theta)$};
\draw[very thick,blue!65!black] (0.40,0) -- (1.15,0) -- (2.15,1.55)
  -- (5.00,1.55) -- (5.85,0) -- (6.45,0);
\draw[black!40,dashed] (0.35,1.55) -- (2.15,1.55);
\node[left] at (0.35,1.55) {$\chi_\star$};
\foreach \x/\lab in {1.15/{-2L},2.15/{-L},5.00/{-\bar\alpha/8},5.85/{0}}
  {\draw[black!55] (\x,0.06) -- (\x,-0.06);
   \node[below] at (\x,-0.07) {$\lab$};}
\node[align=center] at (3.55,2.27) {$\theta_{k+1}\in[-\bar\alpha,-\bar\alpha/4]$};
\draw[black!40,dashed] (4.25,0.12) -- (4.25,1.88);
\fill[orange!85!black] (5.85,0) circle (1.5pt);
\draw[->,thick,orange!85!black] (5.83,0.25)
  to[out=125,in=20] (4.25,0.58);
\node[orange!85!black] at (5.05,0.30) {$g_k=H$};
\draw[->,thick,orange!85!black] (4.25,0.54) -- (3.02,0.54);
\foreach \x in {4.25,4.02,3.83,3.67}
  \fill[orange!85!black] (\x,0.54) circle (1.15pt);
\node at (3.60,-0.95)
  {after entry: total additional travel $\le L/16$};
\end{tikzpicture}
\end{minipage}\hfill
\begin{minipage}[t]{0.49\linewidth}
\centering
\textbf{\footnotesize (b) Persistent denominator}\par\smallskip
\begin{tikzpicture}[x=0.94cm,y=0.68cm,font=\fontsize{8}{9}\selectfont,
  line cap=round,>=Stealth]
\path[use as bounding box] (-0.10,-1.24) rectangle (7.00,2.85);
\draw[->,black!65] (0.45,0) -- (6.60,0) node[right] {$t$};
\draw[->,black!65] (0.45,-0.03) -- (0.45,2.48) node[above] {$v_t$};
\draw[black!40,dotted] (0.45,0.58) -- (6.15,0.58);
\node[below right] at (4.70,0.57) {$1/(8\delta)$};
\draw[->,very thick,orange!85!black] (1.70,0.15) -- (1.70,1.83);
\node[above left] at (1.66,1.84) {$1/(2\delta)$};
\draw[very thick,blue!65!black,dashed,domain=1.70:6.10,samples=50]
  plot (\x,{1.8*exp(-0.75*(\x-1.70)/4.40)});
\node[align=center,fill=white,inner sep=2pt] at (4.50,1.84)
  {$v_t\ge\beta_2^{t-k}/(2\delta)$};
\draw[black!45,->] (4.50,1.62) -- (4.50,1.20);
\draw[black!45] (1.70,0.06) -- (1.70,-0.06);
\node[below] at (1.70,-0.07) {$k$};
\draw[black!45] (6.10,0.06) -- (6.10,-0.06);
\node[below] at (6.10,-0.07) {$T-1$};
\node at (3.75,-0.95)
  {$|\theta_{t+1}-\theta_t|\le
    \bar\alpha\sqrt8\,\chi_\star\sqrt{\delta/T}$ for $k+1\le t<T$};
\end{tikzpicture}
\end{minipage}
\caption{Confidence obstruction for \(\beta_1=0\), with
\(\chi_\star,H,L\) from \eqref{eq:overview-confidence-scales}.
On the event of one positive spike at time \(k\) and no other nonzero
noise, Adam enters the shaded plateau and drifts slowly left. The
dashed curve is a lower envelope of \(v_t\), which bounds the
subsequent travel; both panels are schematic.}
\label{fig:main-confidence-dynamics}
\end{figure}
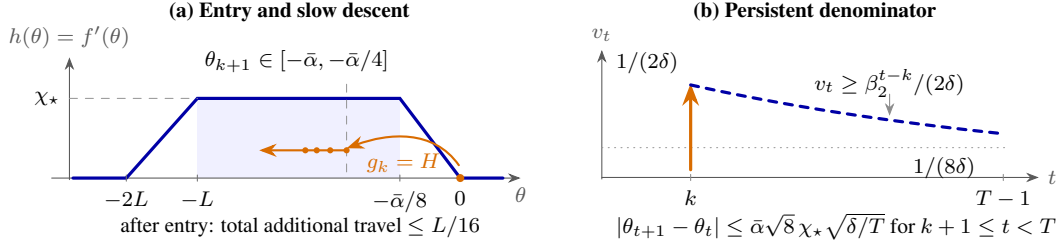

\paragraph{Hard instance for expectation when \(1\le p<2\).}
Proposition~\ref{prop:expected-stationarity-informal} uses a fixed
class containing a deterministic quadratic and a countable
staircase family. For \(f_{\mathrm q}(x)=x^2/2\), started at \(1\),
each early displacement is at most \(\alpha/\epsilon\).
Thus a long initial segment stays in \([1/2,1]\), with squared
gradient at least \(1/4\). This yields the \(T^{-1/3}\)
lower bound when \(\bar\alpha_T\) is below a fixed multiple of
\(T^{-1/6}\).

For larger prefactors, use a family indexed by integers \(n\ge1\),
started at \(\theta_1=0\). Its continuous staircase derivative has
levels that double across blocks of width two, with
\begin{equation}
h_n:=f'_n,\qquad h_n(x)=2^{n+1}\quad(x\ge2n).
\label{eq:overview-staircase-gradient}
\end{equation}
At a point with \(h=h_n(x)\ge2\), the oracle returns \(-h\)
with probability \(1-h^{-1}\), and \(2h^2-h\) with probability
\(h^{-1}\). Its mean is \(h\), and its second moment is
\(4h^3-3h^2\). These instances share the generalized ABC
parameters in Proposition~\ref{prop:expected-stationarity-informal}.

For \(h_n\) in \eqref{eq:overview-staircase-gradient}, let
\(\widehat\tau_n\) be the first index at which the iterate is at or beyond \(2n\)
on the deterministic path that always selects the negative branch.
The stochastic run follows this path on the event
\begin{equation}
E_n:=\{\dir_t=-h_n(\theta_t)\text{ for every }1\le t<\widehat\tau_n\}.
\label{eq:overview-staircase-event}
\end{equation}
On \(E_n\) from \eqref{eq:overview-staircase-event}, negative momentum drives the
iterate right, against the positive true gradient.
Bounding the time spent in each block controls the accumulator and
places an arrival before \(T\).
Using the terminal height in \eqref{eq:overview-staircase-gradient}
and the event in \eqref{eq:overview-staircase-event}, the single arrival term gives
\[
\frac1T\sum_{t=1}^{T-1}\mathbb E|f'_n(\theta_t)|^2
\ge \frac{4^{n+1}}T\,\mathbb P(E_n).
\]
Choosing \(n\) balances the growing gradient against the probability
cost of this path; with the quadratic case, this yields the uniform
\(T^{-1/3}\) lower bound over deterministic prefactors
(Figure~\ref{fig:expectation-dynamics} and
Appendix~\ref{app:expectation-lower-proof}).

\section{Conclusion}
Adam admits high-probability convergence guarantees under
\(L_0\)--\(L_p\) generalized smoothness and generalized ABC second-moment
growth. Self-normalization and stopped augmented energy control the
trajectory, while the final de-preconditioning step gives the
\(\delta^{-1/2}\) confidence dependence. The expectation theorem for
\(p<1\) and the fixed-class lower bound for \(1\le p<2\) distinguish
the two regimes for the calibrated Adam family.

\section*{Acknowledgments}
This research is supported by the InnoHK funding.

\bibliographystyle{plainnat}
\bibliography{sample}
\clearpage
\flushbottom
\appendix
\section{Extended Related Work}
\label{sec:extended-related-work}

\subsection{Adam and adaptive-gradient convergence}
Adam combines momentum with coordinatewise second-moment normalization
\citep{DBLP:journals/corr/KingmaB14}.
\citet{DBLP:conf/iclr/ReddiKK18} identify convergence failures and
introduce AMSGrad, while \citet{chen2019on} give convergence conditions
for a class of Adam-type methods.
\citet{DBLP:journals/tmlr/DefossezBBU22} analyze Adam and AdaGrad under
bounded stochastic gradients. \citet{zhang2022adam} study Adam without
changing its update rule, including the role of its parameter choices
and sampling scheme. \citet{wang2024closing} establish finite-time
complexity bounds under bounded variance, and
\citet{jin2024comprehensive} develop an ABC-based framework for
finite-time and asymptotic Adam convergence. These analyses address
how the adaptive denominator interacts with the gradient used in the
same update.
Last-iterate gradient convergence has also been studied for distributed
stochastic gradient descent with momentum \citep{cheng2026lastiterate}.

\citet{li2025adamw} prove that AdamW satisfies
\(\frac1T\sum_{t=1}^{T}\mathbb E\|\nabla f(\w_t)\|_1=\mathcal{O}(\sqrt d\,T^{-1/4})\)
with no logarithmic factor. Their analysis assumes global smoothness and
coordinatewise bounded conditional variance. These assumptions are much
stronger than Assumptions~\ref{ass:smooth}--\ref{ass:abc}: they exclude
curvature growth, and they bound the variance by a constant instead of
letting the second moment grow with powers of the function gap and of the
gradient norm.
Their problem class is therefore contained in the case
\(p=0\), \(A=0\), \(B=1\), \(\rho_2=2\) of the present setting, and
their criterion is the expected \(\ell_1\) norm. For these reasons a
comparison of the two results has limited meaning; the only common
quantity is the exponent of \(T\) on that subclass. There,
\(\|\nabla f(\w_t)\|_1\le\sqrt d\,\|\nabla f(\w_t)\|\) and the
expected-norm bound obtained above from
Claim~\ref{clm:squared-standard-rate-informal} by Cauchy--Schwarz give
\[
\frac1T\sum_{t=1}^{T-1}\mathbb E\|\nabla f(\w_t)\|_1
=\mathcal{O}(dT^{-1/4}+d^{3/2}T^{-1/2}).
\]
This recovers the exponent \(T^{-1/4}\), with a dimension factor larger
by \(\sqrt d\) than in \citet{li2025adamw} and with constants that are
not optimized. The logarithms in
Theorem~\ref{thm:coefficientwise-convergence-informal} depend only on
\(1/\delta\) and are integrated out in
Claim~\ref{clm:squared-standard-rate-informal}, so this specialization
carries no \(\log T\) factor.

AdaGrad provides a complementary perspective on adaptive scaling.
\citet{duchi2011adaptive} develop adaptive subgradient methods for online
learning and stochastic optimization. In nonconvex optimization,
\citet{li2019adaptive} study adaptive stepsizes, and
\citet{ward2020adagrad} give convergence guarantees for AdaGrad-Norm.
\citet{faw2022power} allow unbounded gradients and affine variance;
\citet{wang2023convergence} use potential arguments under relaxed
stochastic assumptions. \citet{jin2024stability} apply stopping-time
techniques to AdaGrad under global smoothness, with distinct assumptions
for their finite-time and asymptotic conclusions.
\citet{jiang2025provable} compare coordinatewise AdaGrad and SGD through
upper and lower complexity bounds with coordinatewise smoothness and
noise parameters. Scalar normalization, coordinatewise accumulation,
and exponential moving averages produce different denominator
recursions, so their cancellation and memory estimates enter the
analyses differently.

\subsection{Generalized smoothness and second-moment growth}
\citet{DBLP:conf/iclr/ZhangHSJ20} motivate gradient-dependent smoothness
through the behavior of gradient clipping.
\citet{DBLP:conf/nips/CrawshawLO0Z22} study generalized SignSGD under
unbounded smoothness. \citet{chen2023generalized} develop a
generalized-smooth nonconvex framework, while \citet{li2023convex}
analyze convex and nonconvex optimization under generalized
smoothness. \citet{gorbunov2024methods} study clipping, acceleration,
and adaptivity for convex \((L_0,L_1)\)-smooth objectives.
The distinctions among local gradient inequalities, coordinatewise
conditions, and Hessian growth bounds matter when transferring a
convergence guarantee between these settings.

For Adam, \citet{DBLP:conf/nips/LiRJ23} establish guarantees under
relaxed smoothness with bounded or sub-Gaussian noise and identify
bounded variance as an open direction.
\citet{DBLP:journals/corr/abs-2402-03982} analyze Adam under relaxed
smoothness and an exponential-moment noise envelope.
\citet{wang2024provable} study random-reshuffling Adam for finite sums
under non-uniform smoothness and growth conditions.
\citet{wang2024convergence} analyze scalar-normalized Adam under
non-uniform smoothness and affine noise variance, including comparisons
with momentum SGD.
\citet{zhang2024convergence} obtain RMSProp and Adam guarantees under
coordinatewise generalized smoothness and coordinatewise affine
second-moment growth, with the regularizer inside the square root.
Here Assumption~\ref{ass:smooth} is a Euclidean local gradient
inequality, Assumption~\ref{ass:abc} allows powers of the function
gap and gradient norm, and Algorithm~\ref{alg:adam} uses
\(\sqrt{\acc_t}+\epsilon\).

The ABC condition belongs to a broader family of moment-growth models.
\citet{gower2019sgd} use expected smoothness to analyze SGD under
general sampling. \citet{khaled2023better} develop a nonconvex
second-moment model combining the function gap, squared gradient norm,
and a constant noise term. The usual linear ABC form is contained in
Assumption~\ref{ass:abc} by taking \(\rho_1=1\) and \(\rho_2=2\).
Allowing other exponents accommodates nonlinear growth of the oracle's
second moment while retaining conditional unbiasedness. This condition
controls second moments without requiring almost-sure boundedness or
an exponential-moment bound.

\subsection{High-probability guarantees under weak moments}
\citet{ghadimi2013stochastic} establish stochastic first-order
stationarity guarantees for smooth nonconvex objectives.
High-probability bounds depend on both the noise model and the
algorithm: \citet{liu2023high} analyze stochastic gradient methods
under sub-Gaussian noise, and \citet{madden2024high} treat norm
sub-Weibull noise. Martingale inequalities such as
\citet{freedman1975tail} control sums through conditional variance and
bounded increments. In adaptive analyses, these tools can be applied
to normalized quantities even when raw oracle values are unbounded.

Clipping gives another route under weak moment assumptions.
\citet{gorbunov2020clipping} obtain accelerated clipped methods for
convex stochastic optimization with heavy-tailed noise.
\citet{cutkosky2021heavy} give high-probability nonconvex guarantees
using clipped normalized momentum under moment bounds on the
stochastic gradient. \citet{nguyen2023improved} analyze clipped
methods under centered heavy-tailed noise, and
\citet{sadiev2023unbounded} treat stochastic optimization and
variational inequalities with bounded noise moments of order between
one and two. \citet{koloskova2023revisiting} characterize stochastic
clipping bias and establish convergence guarantees and lower bounds
for clipped SGD. Their mechanisms and assumptions differ from the
unclipped, same-sample coordinatewise normalization in
Algorithm~\ref{alg:adam}.

\subsection{Self-normalization and comparison scopes}
\citet{jin2026adam} connect Adam's second-moment normalization to
high-probability stationarity under classical smoothness and bounded
conditional variance. Their stopping-time and de-preconditioning
viewpoint separates control of the normalized trajectory from control
of the accumulator used in the final gradient conversion.
For generalized smoothness and ABC growth, the localized curvature
and noise envelope depend on the objective height. The augmented
energy in the present analysis pairs the true-gradient contribution
with a telescope before bounding the centered noise, yielding the
explicit squared-stationarity estimate in
Theorem~\ref{thm:coefficientwise-convergence}.

Lower bounds must be compared at the same oracle model, stationarity
measure, and algorithmic scope.
\citet{arjevani2023lower} establish oracle-complexity lower bounds for
smooth nonconvex stochastic optimization under bounded-variance
and related oracle models.
Proposition~\ref{prop:delta_sharpness} concerns the confidence
factor for the calibrated Adam family with horizon-length memory.
Proposition~\ref{prop:expected_stationarity_lower_bound} concerns
expected averaged squared gradients on a fixed generalized ABC
class with \(\rho_2=3\), uniformly over deterministic base stepsize
prefactors. Its benchmark is the square-root-horizon expectation
rate in Theorem~\ref{thm:coefficientwise-expectation}; its quantifiers
concern this Adam family rather than all stochastic first-order
algorithms.

\section{Formal Results}
\label{app:formal-results}
This appendix gives the complete statements corresponding to
Section~\ref{sec:results}. We use the auxiliary objective and sequence
in \eqref{eq:auxiliary-objective}--\eqref{eq:auxiliary-sequence}.
\subsection{Parameter scales and convergence bounds}
The two gradient-growth comparison quantities used throughout are
\begin{align}
K&=\max\left\{(4L_p+1)^{\max\{1,(2-p)^{-1}\}},\sqrt{4L_0}\right\},
\label{eq:gradient-growth-constant}\\
C_{\mathrm q}&=1+\frac{K}{L_p}
+\frac{L_0+L_pK^p}{2L_p^2}.
\label{eq:gap-comparison-constant}
\end{align}
Both depend only on $L_0,L_p,p$.
The notation \([z]_+\) means \(\max\{z,0\}\).

With \(\Psi\) from \eqref{eq:auxiliary-objective}, the constants \(L_0,L_p\) of
Assumption~\ref{ass:smooth} and \(C\) of Assumption~\ref{ass:abc}, the fixed
parameter scale used in the high-probability analysis is
\begin{equation}
\begin{aligned}
\mathcal P={}&3+\frac{7\beta_1}{1-\beta_1}
+\frac{\beta_1^2}{(1-\beta_1)^2}\\
&+\Psi(\w_1)+\|\nabla f(\w_1)\|+L_0+L_p+C+\sqrt C\\
&+\epsilon+\binit^{-1}+\binit^{-3/2}.
\end{aligned}
\label{eq:coefficientwise-parameter-scale}
\end{equation}
For an auxiliary-objective height \(U\ge1\), use \(K\) and
\(C_{\mathrm q}\) from
\eqref{eq:gradient-growth-constant}--\eqref{eq:gap-comparison-constant}
together with \(A,B,\rho_1,\rho_2\) of Assumption~\ref{ass:abc},
to define the gradient radius and the growing part of the noise variance by
\begin{align}
G_{\mathrm q}(U)&=KC_{\mathrm q}^{\max\{1,(2-p)^{-1}\}}
U^{\max\{1,p/(2-p)^2\}},\notag\\
\mathcal V_{\mathrm q}(U)&=AC_{\mathrm q}^{\rho_1}
U^{\rho_1\max\{1,p/(2-p)\}}
+\sup_{0\le z\le G_{\mathrm q}(U)}[Bz^{\rho_2}-z^2]_+.
\label{eq:coefficientwise-local-scales}
\end{align}
Using \(\mathcal P\) from \eqref{eq:coefficientwise-parameter-scale}
and the local scales in \eqref{eq:coefficientwise-local-scales}, define
the calibration scale
\begin{equation}
\begin{aligned}
\mathcal R(U)={}&\mathcal P+L_p[G_{\mathrm q}(U)^p-1]_+
+\mathcal V_{\mathrm q}(U)+\sqrt{\mathcal V_{\mathrm q}(U)}\\
&+\sqrt{\frac{2\mathbf1_{\{p>0\}}G_{\mathrm q}(U)^2}{U+1}}.
\end{aligned}
\label{eq:coefficientwise-calibration-scale}
\end{equation}
These definitions display all parameter dependence. The scale
\(\mathcal P\) is fixed, while \(G_{\mathrm q}(U),\mathcal V_{\mathrm q}(U),\mathcal R(U)\)
also depend on the height \(U\). None depends on the horizon or
confidence except through that height. In the proof,
\(C+\mathcal V_{\mathrm q}(U)\) bounds the conditional noise variance before exit
(the exit index \(\sigma_U\) is defined in
\eqref{eq:coefficientwise-stopping-index}).
\begin{table}[h]
\centering
\small
\caption{Quantities used in the upper-bound analysis.}
\label{tab:coefficientwise-notation}
\begin{tabular}{lll}
\toprule
Symbol & Role & Definition\\
\midrule
\(\Psi\) & shifted objective & \eqref{eq:auxiliary-objective}\\
\(\q_t\) & auxiliary sequence & \eqref{eq:auxiliary-sequence}\\
\(\lambda_{t,i}\) & coordinatewise stepsize & Algorithm~\ref{alg:adam}\\
\(K,\ C_{\mathrm q}\) & gradient-growth comparison & \eqref{eq:gradient-growth-constant}--\eqref{eq:gap-comparison-constant}\\
\(\mathcal P\) & fixed parameter scale & \eqref{eq:coefficientwise-parameter-scale}\\
\(G_{\mathrm q}(U),\ \mathcal V_{\mathrm q}(U)\) & gradient radius, growing noise variance & \eqref{eq:coefficientwise-local-scales}\\
\(\mathcal R(U)\) & calibration scale & \eqref{eq:coefficientwise-calibration-scale}\\
\(\mathcal X\) & domination scale & \eqref{eq:coefficientwise-scale-domination}\\
\(U_\delta\) & confidence height & \eqref{eq:coefficientwise-confidence-height}\\
\(\eta_{\mathrm{gen}}\) & stepsize excess factor & \eqref{eq:stepsize-excess-main}\\
\(\Delta_{t,i}\) & stepsize variation & \eqref{eq:coefficientwise-step-variation-definition}\\
\(\sigma_U\) & exit time above height \(U\) & \eqref{eq:coefficientwise-stopping-index}\\
\(E_t,\ \mathcal L_t\) & gradient energy, Lyapunov process & \eqref{eq:coefficientwise-energy-processes}\\
\(Z_n\) & stopped logarithmic energy & \eqref{eq:stopped-logarithmic-energy}\\
\(\tau\) & logarithmic-energy crossing time & \eqref{eq:coefficientwise-log-crossing}\\
\bottomrule
\end{tabular}
\end{table}

\begin{thm}[High-probability convergence]
\label{thm:coefficientwise-convergence}
Assume Assumptions~\ref{ass:nonneg}--\ref{ass:abc}, and use the scales
in \eqref{eq:coefficientwise-parameter-scale}--\eqref{eq:coefficientwise-calibration-scale}. Let \(T\ge10\),
\(\beta_2=1-1/T\), and \(\alpha=\bar\alpha/\sqrt T\) in Algorithm~\ref{alg:adam}.
For \(0<\delta<1\), with \(\eta_{\mathrm{gen}}\) from \eqref{eq:stepsize-excess-main}, set
\begin{equation}
U_\delta=2^{20}\mathcal P^8(1+\log(16/\delta))^{4\mathbf 1_{\{\eta_{\mathrm{gen}}=0\}}}.
\label{eq:coefficientwise-confidence-height}
\end{equation}
With \(\mathcal R\) defined in \eqref{eq:coefficientwise-calibration-scale},
for every
\begin{equation}
0<\bar\alpha\le2^{-100}d^{-1}(1+\log(16/\delta))^{-\mathbf 1_{\{\eta_{\mathrm{gen}}>0\}}}
\mathcal R(U_\delta)^{-48},
\label{eq:coefficientwise-stepsize}
\end{equation}
with probability at least \(1-\delta\), where \(\Psi\) and \(\q_t\) are
defined in \eqref{eq:auxiliary-objective}--\eqref{eq:auxiliary-sequence},
\begin{align}
\max_{1\le t\le T}\Psi(\q_t)&\le U_\delta,
\notag\\
\frac1T\sum_{t=1}^{T-1}\|\nabla f(\w_t)\|^2
&\le\frac{8U_\delta(\epsilon+\sqrt{\binit})}{\bar\alpha\sqrt T}
+\frac{32U_\delta^2}{\bar\alpha^2T}\notag\\
&\quad+\frac{16U_\delta\sqrt{C+\mathcal V_{\mathrm q}(U_\delta)}}
{\bar\alpha\sqrt{\delta T}}.
\label{eq:coefficientwise-rate}
\end{align}
In particular, the upper endpoint of the displayed stepsize interval gives
\(\widetilde{\mathcal{O}}(d(\delta T)^{-1/2}+d^2T^{-1})\) stationarity, where the suppressed
factors are powers of \(1+\log(1/\delta)\). If \(\eta_{\mathrm{gen}}>0\), then, by
\eqref{eq:coefficientwise-confidence-height},
\(U_\delta=2^{20}\mathcal P^8\) does not depend on \(\delta\), so the localization
bound is \(\mathcal{O}(1)\), and \(\delta\) enters \eqref{eq:coefficientwise-stepsize}
only through the single factor \(1+\log(16/\delta)\).
\end{thm}
At the upper endpoint of \eqref{eq:coefficientwise-stepsize} with
\(\eta_{\mathrm{gen}}=0\) (see \eqref{eq:stepsize-excess-main}), the three
coefficients in \eqref{eq:coefficientwise-rate}, with \(\mathcal P\) from
\eqref{eq:coefficientwise-parameter-scale}, are
\(2^{123}d\,\mathcal P^{56}(1+\log(16/\delta))^4(\epsilon+\sqrt{\binit})\),
\(2^{245}d^2\mathcal P^{112}(1+\log(16/\delta))^8\) and
\(2^{124}d\,\mathcal P^{56}(1+\log(16/\delta))^4\sqrt C\); the second term is
dominated by the first once
\(T\ge2^{244}d^2\mathcal P^{112}(1+\log(16/\delta))^8(\epsilon+\sqrt{\binit})^{-2}\).
When \(\eta_{\mathrm{gen}}>0\), the factor \(\mathcal P^{48}(1+\log(16/\delta))^4\)
in the first and third coefficients, and its square in the second
coefficient and in this threshold, are replaced by
\(\mathcal R(2^{20}\mathcal P^8)^{48}(1+\log(16/\delta))\) and its square, with
\(\mathcal R\) from \eqref{eq:coefficientwise-calibration-scale} evaluated at the
height \(2^{20}\mathcal P^8\) of \eqref{eq:coefficientwise-confidence-height}, and
\(\sqrt C\) by \(\sqrt{C+\mathcal V_{\mathrm q}(2^{20}\mathcal P^8)}\).
These constants are not optimized.
The quantities \(U_\delta\) and \(\mathcal V_{\mathrm q}\) in
\eqref{eq:coefficientwise-rate} are defined in
\eqref{eq:coefficientwise-confidence-height} and
\eqref{eq:coefficientwise-local-scales}, respectively.

\begin{claim}[Smooth bounded-variance specialization]
\label{clm:squared-standard-rate}
Under the assumptions of Theorem~\ref{thm:coefficientwise-convergence},
let \(p=0,A=0,B=1,\rho_2=2\) in Assumptions~\ref{ass:smooth}
and~\ref{ass:abc}, and let \(\mathcal P\) be as in
\eqref{eq:coefficientwise-parameter-scale}. For every fixed
\(0<\bar\alpha\le2^{-100}d^{-1}\mathcal P^{-48}\),
\begin{equation}
\begin{aligned}
\frac1T\sum_{t=1}^{T-1}\mathbb E\|\nabla f(\w_t)\|^2
\le{}&\frac{2^{23}\mathcal P^8}{\bar\alpha\sqrt T}
\sum_{j=0}^{4}\binom4j(1+\log16)^{4-j}j!
\bigl(\epsilon+\sqrt{\binit}+2^{j+2}\sqrt C\bigr)\\
&+\frac{2^{45}\mathcal P^{16}}{\bar\alpha^2T}
\sum_{j=0}^{8}\binom8j(1+\log16)^{8-j}j!.
\end{aligned}
\label{eq:coefficientwise-standard-mean}
\end{equation}
In particular, at \(\bar\alpha=2^{-100}d^{-1}\mathcal P^{-48}\) the left-hand side
is \(\mathcal{O}(dT^{-1/2}+d^2T^{-1})\).
\end{claim}
Here \(\eta_{\mathrm{gen}}=0\) in \eqref{eq:stepsize-excess-main},
\(\mathcal V_{\mathrm q}(U)=0\) and
\(\mathcal R(U)=\mathcal P\) for all \(U\ge1\), by
\eqref{eq:coefficientwise-local-scales}--\eqref{eq:coefficientwise-calibration-scale}.
Thus the same run satisfies \eqref{eq:coefficientwise-rate}, with
\(U_\delta\) from \eqref{eq:coefficientwise-confidence-height}, at every
confidence level, and integrating that bound gives the claim.

\begin{thm}[Expected convergence for \(p<1\)]
\label{thm:coefficientwise-expectation}
Assume Assumptions~\ref{ass:nonneg}--\ref{ass:abc} with \(p<1\), and let
\(T\ge10\), \(\beta_2=1-1/T\) in Algorithm~\ref{alg:adam}. Choose
\begin{equation}
\begin{aligned}
A_{\mathrm{tail}}&>\frac12+\frac2{1-p},\qquad
\delta_T=T^{-A_{\mathrm{tail}}},\qquad
\alpha=\bar\alpha_T/\sqrt T,\\
\bar\alpha_T&=2^{-100}d^{-1}(1+\log(16/\delta_T))^{-\mathbf 1_{\{\eta_{\mathrm{gen}}>0\}}}
\mathcal R(U_{\delta_T})^{-48},
\end{aligned}
\label{eq:coefficientwise-tail-calibration}
\end{equation}
where \(\eta_{\mathrm{gen}}\), \(\mathcal R\) and \(U_\delta\) are defined in
\eqref{eq:stepsize-excess-main}, \eqref{eq:coefficientwise-calibration-scale} and
\eqref{eq:coefficientwise-confidence-height}. Then
\begin{equation}
\begin{aligned}
\frac1T\sum_{t=1}^{T-1}\mathbb E\|\nabla f(\w_t)\|^2
\le{}&\frac{8(\epsilon+\sqrt{\binit})}{\bar\alpha_T\sqrt T}
\int_{\delta_T}^1 U_\delta\,d\delta
+\frac{32}{\bar\alpha_T^2T}\int_{\delta_T}^1 U_\delta^2\,d\delta\\
&+\frac{16}{\bar\alpha_T\sqrt T}
\int_{\delta_T}^1
\frac{U_\delta\sqrt{C+\mathcal V_{\mathrm q}(U_\delta)}}{\sqrt\delta}\,d\delta\\
&+\delta_T(1+2T)^{2/(1-p)}\\
&\qquad{}\times
\left[(1+\|\nabla f(\w_1)\|)^{1-p}+(1-p)(L_0+L_p)\right]^{2/(1-p)}.
\end{aligned}
\label{eq:coefficientwise-expectation-explicit}
\end{equation}
The integrands use only \(U_\delta\) and \(\mathcal V_{\mathrm q}\) from
\eqref{eq:coefficientwise-confidence-height} and
\eqref{eq:coefficientwise-local-scales}; the cutoff and prefactor are specified by
\eqref{eq:coefficientwise-tail-calibration}. In particular,
\[
\frac1T\sum_{t=1}^{T-1}\mathbb E\|\nabla f(\w_t)\|^2
=\begin{cases}
\mathcal{O}\!\left(\dfrac{d(1+\log T)}{\sqrt T}
+\dfrac{d^2(1+\log T)^2}{T}\right), & \eta_{\mathrm{gen}}>0,\\[2ex]
\mathcal{O}\!\left(\dfrac{d}{\sqrt T}+\dfrac{d^2}{T}\right), & \eta_{\mathrm{gen}}=0.
\end{cases}
\]
\end{thm}
The horizon enters these orders only through \(\bar\alpha_T^{-1}\). The
integrals in \eqref{eq:coefficientwise-expectation-explicit} are at most their
values over \((0,1]\), which are finite and independent of \(T\), and the last
term is \(o(T^{-1/2})\) because \(A_{\mathrm{tail}}>1/2+2/(1-p)\). By
\eqref{eq:coefficientwise-tail-calibration}, \(\bar\alpha_T^{-1}\) is linear in
\(1+\log(16/\delta_T)=1+\log16+A_{\mathrm{tail}}\log T\) when
\(\eta_{\mathrm{gen}}>0\) (see \eqref{eq:stepsize-excess-main}), whereas
\(\bar\alpha_T^{-1}=2^{100}d\,\mathcal P^{48}\), with \(\mathcal P\) from
\eqref{eq:coefficientwise-parameter-scale},
does not depend on \(T\) when \(\eta_{\mathrm{gen}}=0\).

\subsection{Lower bounds}
\begin{proposition}[Sharpness of the \(\delta^{-1/2}\) dependence]
\label{prop:delta_sharpness}
Fix \(\bar\alpha\in(0,1]\), \(\binit>0\), and \(\epsilon>0\).
For every integer \(T\ge4\) and every
\[
0<\delta\le\min\left\{\frac18,\frac1{8(\binit+\epsilon^2)}\right\},
\qquad \delta T\ge1,
\]
there is a one-dimensional lower-bounded \(C^1\) objective and an
unbiased oracle satisfying Assumptions~\ref{ass:nonneg}--\ref{ass:abc}
with \(p=0\), \(A=0\), \(B=1\), \(\rho_1=0\),
\(\rho_2=2\), \(C=1\),
and smoothness constants independent of \(T,\delta\), such that
\[
\theta_1=0,\qquad f'(\theta_1)=0,\qquad
f(\theta_1)-f^*\le32\sqrt8\,\bar\alpha,
\]
and Adam (Algorithm~\ref{alg:adam}) with \(\beta_1=0\), \(\beta_2=1-1/T\), and
\(\alpha=\bar\alpha/\sqrt T\) satisfies
\[
\mathbb P\!\left(
\frac1T\sum_{t=1}^{T-1}
|\nabla f(\w_t)|^2
\ge
\frac{1}{4\sqrt{\delta T}}
\right)
\ge
\frac\delta4.
\]
Whenever \(\delta T\ge(8/\bar\alpha)^4\), the same construction
(the objective in \eqref{eq:confidence-lower-ramp} and the oracle noise in
\eqref{eq:confidence-lower-noise}, with scales from
\eqref{eq:confidence-lower-scales})
satisfies the common smoothness constants \(L_0=L_p=1\).
\end{proposition}

The polynomial confidence obstruction holds within this common classical
parameter class, including prefactors admitted by
Theorem~\ref{thm:coefficientwise-convergence}. Indeed, each fixed
\[
0<\bar\alpha\le
2^{-100}\left(100+\epsilon+\binit^{-1}+\binit^{-3/2}\right)^{-48}
\]
satisfies its classical stepsize condition \eqref{eq:coefficientwise-stepsize} for every constructed member
with \(\delta T\ge(8/\bar\alpha)^4\).
Reparameterizing the construction of Proposition~\ref{prop:delta_sharpness} by \(\delta=8\eta\) gives, for
sufficiently small \(\eta>0\) and sufficiently large horizons,
\[
\mathbb P\!\left(
\frac1T\sum_{t=1}^{T-1}|f'(\theta_t)|^2
\ge\frac1{8\sqrt2\sqrt{\eta T}}
\right)\ge2\eta.
\]
Thus a uniform \(\polylog(1/\eta)/\sqrt T\) guarantee at confidence
\(1-\eta\) is impossible for these fixed calibrated prefactors.
The comparison concerns the same averaged squared-gradient statistic and
the \(\beta_1=0\) subfamily of Algorithm~\ref{alg:adam}.

\begin{proposition}[A lower bound on expected stationarity for \(1\le p<2\)]
\label{prop:expected_stationarity_lower_bound}
Fix \(p\in[1,2)\), \(\beta_1\in[0,1)\), \(\binit>0\), and
\(\epsilon>0\). There exists a fixed class \(\mathcal C_p\) of
one-dimensional objectives, stochastic-gradient oracles, and initial points
satisfying Assumptions~\ref{ass:nonneg}--\ref{ass:abc} with
\[
L_0=L_p=1,\qquad A=0,\quad B=4,\quad C=1,\quad
\rho_1=0,\quad\rho_2=3,
\]
and
\[
f^*=0,\qquad |\theta_1|\le1,\qquad
|f'(\theta_1)|\le2,\qquad f(\theta_1)-f^*\le2,
\]
such that Adam (Algorithm~\ref{alg:adam}) with \(\beta_2=1-1/T\) and
\(\alpha=\bar\alpha_T/\sqrt T\) satisfies, for every sufficiently large integer
\(T\),
\begin{equation}
\label{eq:expected_stationarity_lower_bound}
\inf_{\bar\alpha_T>0}\;
\sup_{(f,\mathsf G,\theta_1)\in\mathcal C_p}
\frac1T\sum_{t=1}^{T-1}\mathbb E|f'(\theta_t)|^2
\ge
\frac{\epsilon(1-\beta_1)^2}
     {2^{19}(1+\sqrt{\binit}+\epsilon)^2}\,T^{-1/3}.
\end{equation}
The class, whose members are given in
\eqref{eq:expected-lower-quadratic-instance},
\eqref{eq:expected-lower-staircase-objective}, and
\eqref{eq:expected-lower-staircase-oracle}, is independent of \(T\) and \(\bar\alpha_T\); the horizon threshold
depends only on \(\beta_1,\binit,\epsilon\).
\end{proposition}

The infimum in Proposition~\ref{prop:expected_stationarity_lower_bound}
is over deterministic scaled base stepsizes for the stated Adam recursion.
Its stationarity measure is the same expected average of squared true-gradient
norms as in Theorem~\ref{thm:coefficientwise-expectation}.
For every fixed \(k\ge0\),
\[
\frac{T^{-1/3}}{\log^k(eT)/\sqrt T}
=\frac{T^{1/6}}{\log^k(eT)}\longrightarrow\infty.
\]
Consequently, no uniform bound of order
\(\log^k(eT)/\sqrt T\) holds for this Adam family on the full class in the
proposition. The comparison is with the square-root-horizon expectation rate
available when \(p<1\) (see \eqref{eq:coefficientwise-expectation-explicit}); the lower bound concerns the Adam family and the
generalized ABC class containing \(\rho_2=3\).

\section{Gradient growth and martingale tools}
\label{app:auxiliary}

Throughout this section, \(\Psi\) is the auxiliary objective in
\eqref{eq:auxiliary-objective}.

\begin{lemma}[Gradient growth from the function gap]
\label{lem:gradient-gap-growth}
Under Assumptions~\ref{ass:nonneg}--\ref{ass:smooth}, with
\(K\) defined in \eqref{eq:gradient-growth-constant},
\begin{equation}
\|\nabla f(x)\|
\le K\Psi(x)^{\max\{1,(2-p)^{-1}\}}
\qquad(x\in\mathbb R^d).
\label{eq:gradient-gap-bound}
\end{equation}
\end{lemma}
\begin{proof}
The conclusion is immediate when \(\nabla f(x)=0\). Otherwise take
\[
y=x-\frac{\nabla f(x)}
{L_0+L_p\max\{\|\nabla f(x)\|^p,\|\nabla f(x)\|\}}.
\]
This displacement is at most \(1/L_p\). Integrating the local
gradient inequality of Assumption~\ref{ass:smooth} along the segment from \(x\) to \(y\) gives
\[
f(y)\le f(x)+\langle\nabla f(x),y-x\rangle
+\frac{L_0+L_p\|\nabla f(x)\|^p}{2}\|y-x\|^2.
\]
The chosen reciprocal stepsize is at least the local smoothness
coefficient, so lower boundedness (Assumption~\ref{ass:nonneg}) implies
\[
f(x)-f^*\ge
\frac{\|\nabla f(x)\|^2}
{2[L_0+L_p\max\{\|\nabla f(x)\|^p,\|\nabla f(x)\|\}]}.
\]
Either \(\|\nabla f(x)\|^2\le4L_0(f(x)-f^*)\), or the last
display yields
\[
\|\nabla f(x)\|^2\le
4L_p(f(x)-f^*)\max\{\|\nabla f(x)\|^p,\|\nabla f(x)\|\}.
\]
Distinguishing the two entries of this maximum proves
\[
\|\nabla f(x)\|\le
\max\left\{
\sqrt{4L_0(f(x)-f^*)},\,
[4L_p(f(x)-f^*)]^{1/(2-p)},\,
4L_p(f(x)-f^*)
\right\}.
\]
Each entry is at most
\(K\Psi(x)^{\max\{1,(2-p)^{-1}\}}\), by the definition of \(K\)
in \eqref{eq:gradient-growth-constant} and \(\Psi(x)\ge1\) (see
\eqref{eq:auxiliary-objective}).
\end{proof}

\begin{lemma}[Comparison of nearby function gaps]
\label{lem:barf_comparison}
Under Assumptions~\ref{ass:nonneg}--\ref{ass:smooth}, if
\(\|x-y\|\le1/L_p\), then, with \(\Psi\) from \eqref{eq:auxiliary-objective}
and \(C_{\mathrm q}\) from
\eqref{eq:gap-comparison-constant},
\begin{equation}
\Psi(x)\le C_{\mathrm q}
\Psi(y)^{\max\{1,p/(2-p)\}}.
\label{eq:nearby-gap-bound}
\end{equation}
\end{lemma}
\begin{proof}
The integrated local gradient inequality of Assumption~\ref{ass:smooth} and
the gradient bound \eqref{eq:gradient-gap-bound} of
Lemma~\ref{lem:gradient-gap-growth}, with \(K\) from
\eqref{eq:gradient-growth-constant}, give
\begin{align*}
\Psi(x)\le{}&\Psi(y)
+\frac K{L_p}\Psi(y)^{\max\{1,(2-p)^{-1}\}}\\
&+\frac{L_0+L_pK^p}{2L_p^2}
\Psi(y)^{p\max\{1,(2-p)^{-1}\}}.
\end{align*}
Since \(\Psi(y)\ge1\), all three powers are bounded by the power
\(\max\{1,p/(2-p)\}\). Indeed this exponent is one for \(p\le1\)
and is \(p/(2-p)\) for \(p\ge1\). Substituting the definition of
\(C_{\mathrm q}\) in \eqref{eq:gap-comparison-constant} gives
\eqref{eq:nearby-gap-bound}.
\end{proof}

\begin{lemma}[Polynomial distance growth for \(p<1\)]
\label{lem:gradient_distance_poly}
Under Assumption~\ref{ass:smooth}, for \(0\le p<1\) and any
\(x,x_{\mathrm{ref}}\in\mathbb R^d\),
\begin{align}
(1+\|\nabla f(x)\|)^{1-p}
&\le(1+\|\nabla f(x_{\mathrm{ref}})\|)^{1-p}
+(1-p)(L_0+L_p)\|x-x_{\mathrm{ref}}\|.
\label{eq:distance-gradient-exact}
\end{align}
Consequently,
\begin{align*}
\|\nabla f(x)\|\le{}&
\left[(1+\|\nabla f(x_{\mathrm{ref}})\|)^{1-p}
+(1-p)(L_0+L_p)\right]^{1/(1-p)}\\
&\qquad\times(1+\|x-x_{\mathrm{ref}}\|)^{1/(1-p)}.
\end{align*}
\end{lemma}
\begin{proof}
Partition the segment from \(x_{\mathrm{ref}}\) to \(x\) into finitely
many subsegments of length at most \(1/L_p\), with consecutive endpoints
\(z_{j-1},z_j\). The reverse triangle inequality, combined with
Assumption~\ref{ass:smooth} on each subsegment, gives
\[
\|\nabla f(z_j)\|-\|\nabla f(z_{j-1})\|
\le (L_0+L_p)(1+\|\nabla f(z_{j-1})\|)^p
\|z_j-z_{j-1}\|.
\]
The increasing concave function \(u\mapsto u^{1-p}\), on \(u>0\),
therefore satisfies
\begin{align*}
&(1+\|\nabla f(z_j)\|)^{1-p}
 -(1+\|\nabla f(z_{j-1})\|)^{1-p}\\
&\quad\le(1-p)(1+\|\nabla f(z_{j-1})\|)^{-p}
 \bigl(\|\nabla f(z_j)\|-\|\nabla f(z_{j-1})\|\bigr)\\
&\quad\le(1-p)(L_0+L_p)\|z_j-z_{j-1}\|.
\end{align*}
Summing over the partition proves \eqref{eq:distance-gradient-exact}.
The second bound follows by factoring \(1+\|x-x_{\mathrm{ref}}\|\)
and taking the positive power \(1/(1-p)\).
\end{proof}

\begin{lemma}[Finite-horizon accumulator comparison]
\label{lem:vt_comparable}
For Algorithm~\ref{alg:adam} with \(T\ge10\) and
\(\beta_2=1-1/T\), every \(0\le t\le T-1\) satisfies
\[
\acc_{t,i}\ge\frac14\left(\binit+
\frac1T\sum_{k=1}^t\dir_{k,i}^2\right).
\]
Moreover, the coordinatewise stepsizes \(\lambda_{t,i}\) of
Algorithm~\ref{alg:adam} satisfy \(\lambda_{t,i}\le2\lambda_{k,i}\) for
\(0\le k\le t\le T-1\).
\end{lemma}
\begin{proof}
Unrolling the second-moment recursion of Algorithm~\ref{alg:adam} with
\(\beta_2=1-1/T\) gives
\[
\acc_{t,i}=(1-1/T)^t\binit+
\frac1T\sum_{k=1}^t(1-1/T)^{t-k}\dir_{k,i}^2.
\]
The elementary bound \((1-1/T)^T\ge1/4\) holds for \(T\ge2\):
the function \(u\mapsto u\log(1-1/u)\) is increasing for \(u>1\),
since \(\log(1-z)>-z/(1-z)\) for \(0<z<1\), and its value at
\(u=2\) is \(\log(1/4)\). Thus every weight in the display is
at least \(1/4\). The same recursion started at \(k\) gives
\(\acc_{t,i}\ge\acc_{k,i}/4\); taking square roots and retaining
\(\epsilon>0\) proves the comparison of the stepsizes \(\lambda_{t,i}\)
defined in Algorithm~\ref{alg:adam}.
\end{proof}

\begin{lemma}[Conditional-to-direct moment comparison]
\label{lem:cond_to_direct}
For nonnegative random variables \(Z_t\) adapted to a filtration
\(\{\mathscr F_t\}\) (in the applications, the one of
Assumption~\ref{ass:abc}) and a finite horizon
\(n\), write
\[
Y_n=\sum_{t=1}^n\mathbb E[Z_t\mid\mathscr F_{t-1}].
\]
For every real \(s\ge1\),
\[
\|Y_n\|_{L^s}\le s\left\|\sum_{t=1}^n Z_t\right\|_{L^s},\qquad
\|X\|_{L^s}:=(\mathbb E|X|^s)^{1/s}.
\]
The same result applies to a predictable truncation by multiplying
\(Z_t\) by its \(\mathscr F_{t-1}\)-measurable inclusion indicator.
\end{lemma}
\begin{proof}
The case \(s=1\) follows from the tower property. For \(s>1\), fix a
nonnegative dual random variable \(W\) with
\(\|W\|_{L^{s/(s-1)}}=1\), and let
\(M_t=\mathbb E[W\mid\mathscr F_t]\). Then
\[
\mathbb E[WY_n]=\sum_{t=1}^n\mathbb E[M_{t-1}Z_t]
\le\mathbb E\left[\left(\max_{0\le t\le n}M_t\right)
\sum_{u=1}^n Z_u\right].
\]
H\"older's inequality, Doob's maximal inequality, and contraction of
conditional expectation give
\[
\mathbb E[WY_n]
\le\left\|\max_{0\le t\le n}M_t\right\|_{L^{s/(s-1)}}
\left\|\sum_{u=1}^n Z_u\right\|_{L^s}
\le s\left\|\sum_{u=1}^n Z_u\right\|_{L^s}.
\]
Taking the supremum over the dual variables proves the claim. If needed,
apply the argument first to truncated random variables and pass to the
limit by monotone convergence; when the right side is infinite the
inequality holds in the extended sense.
\end{proof}

We use the scalar Freedman inequality \citep{freedman1975tail} in the
following form. For bounded real martingale differences \(D_t\) (in the
applications, with respect to the filtration \(\{\mathscr F_t\}\) of
Assumption~\ref{ass:abc}) and
\(0<\delta<1\),
\begin{align}
\mathbb P\Bigg(
\sum_{t=1}^n D_t>{}&
\sqrt{2\left\|\sum_{t=1}^n
\mathbb E[D_t^2\mid\mathscr F_{t-1}]\right\|_{L^\infty}
\log(1/\delta)}\notag\\
&+\frac23\max_{1\le t\le n}\|D_t\|_{L^\infty}
\log(1/\delta)
\Bigg)\le\delta.
\label{eq:freedman-tool}
\end{align}
When both bounds vanish, the martingale is identically zero and the
strict upper-tail event is empty.

\section{Proof of Theorem~\ref{thm:coefficientwise-convergence}}

Throughout this proof, \(K,C_{\mathrm q}\) are defined in
\eqref{eq:gradient-growth-constant}--\eqref{eq:gap-comparison-constant},
\(\mathcal P\) in \eqref{eq:coefficientwise-parameter-scale}, and
\(G_{\mathrm q},\mathcal V_{\mathrm q},\mathcal R\) in
\eqref{eq:coefficientwise-local-scales}--\eqref{eq:coefficientwise-calibration-scale}.
The auxiliary objective \(\Psi\) is defined in \eqref{eq:auxiliary-objective},
the auxiliary sequence \(\q_t\) in \eqref{eq:auxiliary-sequence}, and the
effective stepsizes \(\lambda_{t,i}\) in Algorithm~\ref{alg:adam}.

\paragraph{Proof framework.}
Figure~\ref{fig:formal-proof-dependencies} records the dependencies
among the auxiliary estimates and convergence results.
The stopped concentration estimates close the localized energy descent;
ordinary-gradient conversion then proves
Theorem~\ref{thm:coefficientwise-convergence}.
Its two expectation consequences use the calibrations and integration
arguments in Appendices~\ref{app:proof-standard-specialization}
and~\ref{app:proof-expectation}, respectively.
For \(p<1\), the distance bound in
Lemma~\ref{lem:gradient_distance_poly} supplies the pathwise tail contribution.
\begin{figure}[H]
\centering
\begin{tikzpicture}[
  x=1cm,y=1cm,
  flowbox/.style={draw=black!75,rounded corners=2pt,line width=0.45pt,
    align=center,text width=3.25cm,minimum height=1.08cm,inner sep=4pt,
    font=\fontsize{8.4}{10.1}\selectfont},
  flowarrow/.style={-{Stealth[length=1.6mm]},line width=0.55pt},
  completion/.style={flowarrow,dashed}]
\node[flowbox,text width=4.1cm] (assumptions) at (0,0)
  {Assumptions~\ref{ass:nonneg}--\ref{ass:abc}\\
   Algorithm~\ref{alg:adam} and calibration
   \eqref{eq:coefficientwise-stepsize}};
\node[flowbox] (gap) at (-4.45,-1.7)
  {Gradient and nearby-gap bounds\\
   Lemmas~\ref{lem:gradient-gap-growth}--\ref{lem:barf_comparison}};
\node[flowbox] (normalized) at (0,-1.7)
  {Normalized estimates\\
   Lemmas~\ref{lem:vt_comparable}, \ref{lem:stopped-normalized-energies}\\
   Variation bound \eqref{eq:direct-variation-bound}};
\node[flowbox] (distance) at (4.45,-1.7)
  {Polynomial growth with distance\\
   Lemma~\ref{lem:gradient_distance_poly}, \(p<1\)};
\node[flowbox] (descent) at (-4.45,-3.55)
  {Localized descent\\\eqref{eq:localized-augmented-descent}\\
   Gradient bound \eqref{eq:uniform-stopped-gradient}};
\node[flowbox] (concentration) at (0,-3.55)
  {Four stopped estimates\\
   Appendix~\ref{app:stopped-concentration}\\
   Lemma~\ref{lem:cond_to_direct} and \eqref{eq:freedman-tool}};
\node[flowbox] (pathwise) at (4.45,-3.55)
  {Pathwise gradient bound\\
   Distance growth \eqref{eq:distance-gradient-exact}\\
   Displacement \eqref{eq:coefficientwise-momentum-displacement}};
\node[flowbox,text width=3.7cm] (closure) at (0,-5.3)
  {Stopped energy and localization\\
   Lemma~\ref{lem:stopped-energy-localization}\\
   Preconditioned energy \eqref{eq:closed-preconditioned-energy}};
\node[flowbox,text width=3.7cm] (hp) at (0,-7.05)
  {Ordinary-gradient conversion\\
   Appendix~\ref{app:de-preconditioning}\\
   Theorem~\ref{thm:coefficientwise-convergence}};
\node[flowbox] (standard) at (-4.45,-7.05)
  {Standard specialization\\
   Fixed prefactor, all-\(\delta\) integration\\
   Appendix~\ref{app:proof-standard-specialization}};
\node[flowbox] (cutoff) at (4.45,-7.05)
  {One run calibrated at \eqref{eq:coefficientwise-tail-calibration}\\
   Retained quantiles and pathwise tail\\
   Appendix~\ref{app:proof-expectation}};
\node[flowbox,text width=3.7cm] (standardmean) at (-4.45,-8.85)
  {Standard expected stationarity\\
   Claim~\ref{clm:squared-standard-rate}\\
   Complete bound \eqref{eq:coefficientwise-standard-mean}};
\node[flowbox,text width=3.7cm] (generalmean) at (4.45,-8.85)
  {Expected stationarity, \(p<1\)\\
   Theorem~\ref{thm:coefficientwise-expectation}\\
   Complete bound \eqref{eq:coefficientwise-expectation-explicit}};
\draw[flowarrow] (assumptions.south west) -- (gap.north);
\draw[flowarrow] (assumptions) -- (normalized);
\draw[flowarrow] (assumptions.south east) -- (distance.north);
\draw[flowarrow] (gap) -- (descent);
\draw[flowarrow] (normalized) -- (concentration);
\draw[flowarrow] (distance) -- (pathwise);
\draw[flowarrow] (normalized.south east) -- (pathwise.north west);
\draw[flowarrow] (descent.east) -- (concentration.west);
\draw[flowarrow] (descent.south) -- (closure.north west);
\draw[flowarrow] (concentration) -- (closure);
\draw[flowarrow] (closure) -- (hp);
\draw[flowarrow] (hp.west) -- (standard.east);
\draw[flowarrow] (hp.east) -- (cutoff.west);
\draw[flowarrow] (standard) -- (standardmean);
\draw[flowarrow] (cutoff) -- (generalmean);
\draw[completion] (pathwise.south) -- (cutoff.north);
\end{tikzpicture}
\caption{Formal dependencies for Theorems~\ref{thm:coefficientwise-convergence}
and~\ref{thm:coefficientwise-expectation} and
Claim~\ref{clm:squared-standard-rate}.
Solid arrows connect the localized estimates, stopped concentration,
energy closure, and gradient conversion. The standard branch takes
\(p=0,A=0,B=1,\rho_2=2\). The two integration branches
use one fixed run under their respective calibrations; the dashed arrow
supplies the deterministic contribution of extreme quantiles for \(p<1\).}
\label{fig:formal-proof-dependencies}
\end{figure}
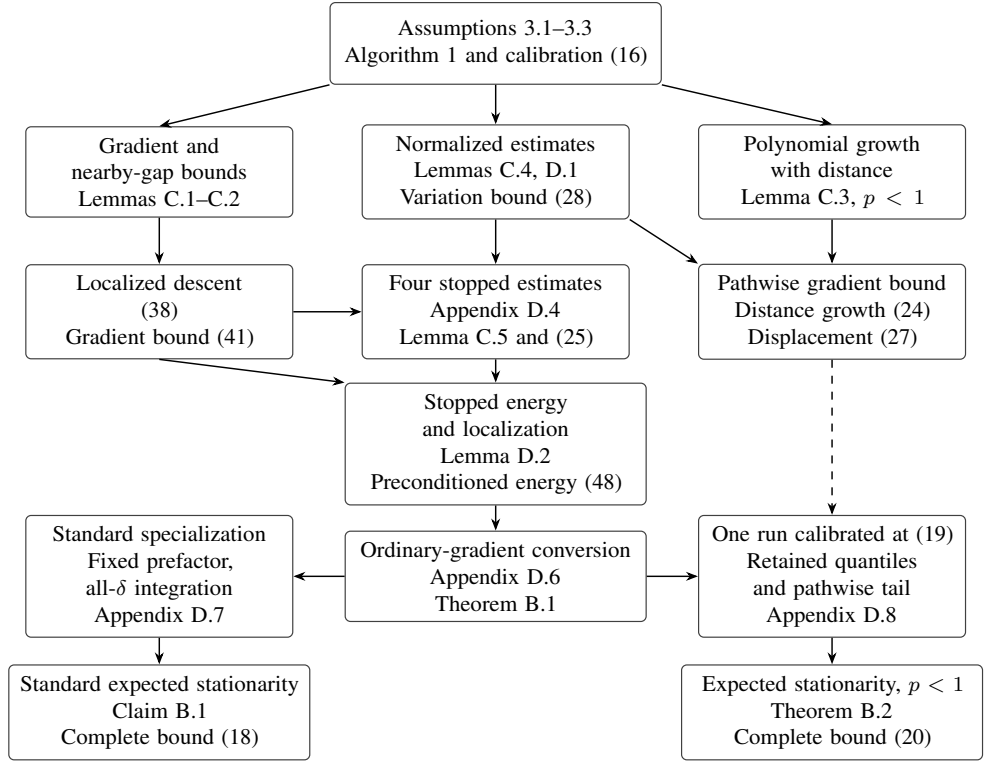

\subsection{Deterministic normalized estimates}

With the effective stepsizes \(\lambda_{t,i}\) and accumulators
\(\acc_{t,i}\) of Algorithm~\ref{alg:adam}, let
\begin{equation}
\widetilde\lambda_{t-1,i}
=\frac{\bar\alpha}{\sqrt{T-1}(\sqrt{\acc_{t-1,i}}+\epsilon)},
\qquad \Delta_{t,i}=\widetilde\lambda_{t-1,i}-\lambda_{t,i}.
\label{eq:coefficientwise-step-variation-definition}
\end{equation}
The accumulator recursion gives, for all \(t<T\),
\[
\acc_{t,i}\ge\binit/4,\qquad
\lambda_{t,i}\le\frac{2\bar\alpha}{\sqrt{T\binit}},\qquad
|\lambda_{t,i}\dir_{t,i}|\le\bar\alpha,
\qquad 0\le\Delta_{t,i}\le2\lambda_{t-1,i}.
\]
Moreover, \(\lambda_{t,i}\le2\lambda_{k,i}\) for \(k\le t<T\).
Expanding the momentum recursion therefore yields
\begin{equation}
\|\lambda_t\odot\mom_t\|\le2\sqrt d\,\bar\alpha.
\label{eq:coefficientwise-momentum-displacement}
\end{equation}
For \(\Delta_{t,i}\) from
\eqref{eq:coefficientwise-step-variation-definition}, the exact decomposition
\[
\Delta_{t,i}=\lambda_{t-1,i}-\lambda_{t,i}
+\bar\alpha\left(\frac1{\sqrt{T-1}}-\frac1{\sqrt T}\right)
\frac1{\sqrt{\acc_{t-1,i}}+\epsilon}
\]
implies, for every \(n\le T-1\),
\begin{align}
\sum_{t=1}^n\sum_i\Delta_{t,i}
&\le\frac{d\bar\alpha}{\sqrt T(\sqrt{\binit}+\epsilon)}
+\frac{2d\bar\alpha(T-1)}{\sqrt{\binit}}
\left(\frac1{\sqrt{T-1}}-\frac1{\sqrt T}\right)
\notag\\
&\le\frac{2d\bar\alpha}{\sqrt{T\binit}}.
\label{eq:direct-variation-bound}
\end{align}
For the last step use
\((T-1)(1/\sqrt{T-1}-1/\sqrt T)<1/(2\sqrt T)\).

\begin{lemma}[Stopped normalized energies]
\label{lem:stopped-normalized-energies}
For Algorithm~\ref{alg:adam} with \(T\ge10\),
\(\beta_2=1-1/T\), and \(\alpha=\bar\alpha/\sqrt T\),
let \(\sigma\) be a stopping index taking values in
\(\{1,2,\ldots\}\cup\{\infty\}\). For every \(1\le n\le T\),
\begin{align}
&\sum_{t<n\wedge\sigma}\left(
\|\lambda_t\odot\dir_t\|^2
+\|\lambda_t\odot\mom_t\|^2
+\|\lambda_{t-1}\odot\mom_{t-1}\|^2\right)\notag\\
&\qquad\le132\bar\alpha^2\sum_i
\log\left(1+\frac{\sum_{t<n\wedge\sigma}\dir_{t,i}^2}{T\binit}\right).
\label{eq:stopped-normalized-energies}
\end{align}
\end{lemma}
\begin{proof}
For the stopping index \(\sigma\), put
\begin{equation}
Z_n=\sum_i\log\left(1+
\frac{\sum_{t<n\wedge\sigma}\dir_{t,i}^2}{T\binit}\right).
\label{eq:stopped-logarithmic-energy}
\end{equation}
The accumulator lower bound
\(\acc_{t,i}\ge(\binit+T^{-1}\sum_{k\le t}\dir_{k,i}^2)/4\),
followed by the logarithmic sum inequality, gives
\begin{align}
\sum_{t<n\wedge\sigma}\|\lambda_t\odot\dir_t\|^2
&\le4\bar\alpha^2Z_n,\notag\\
\sum_{t<n\wedge\sigma}\|\lambda_t\odot\mom_t\|^2,\quad
\sum_{t<n\wedge\sigma}\|\lambda_{t-1}\odot\mom_{t-1}\|^2
&\le64\bar\alpha^2Z_n.
\label{eq:normalized-three-energies}
\end{align}
For the momentum bounds, fix \(i\). Since \(\mom_0=\mathbf 0\),
\(\mom_{t,i}=(1-\beta_1)\sum_{k=1}^t\beta_1^{t-k}\dir_{k,i}\), and the
weights \((1-\beta_1)\beta_1^{t-k}\) sum to at most one. Jensen's
inequality and \(\lambda_{t,i}\le2\lambda_{k,i}\) for \(k\le t\) give
\[
(\lambda_{t,i}\mom_{t,i})^2
\le(1-\beta_1)\sum_{k=1}^t\beta_1^{t-k}(\lambda_{t,i}\dir_{k,i})^2
\le4(1-\beta_1)\sum_{k=1}^t\beta_1^{t-k}(\lambda_{k,i}\dir_{k,i})^2.
\]
Sum over \(t<n\wedge\sigma\) and \(i\), exchange the sums over \(t\)
and \(k\), and use \((1-\beta_1)\sum_{t\ge k}\beta_1^{t-k}\le1\):
\[
\sum_{t<n\wedge\sigma}\|\lambda_t\odot\mom_t\|^2
\le4\sum_{k<n\wedge\sigma}\|\lambda_k\odot\dir_k\|^2
\le16\bar\alpha^2Z_n\le64\bar\alpha^2Z_n.
\]
Because \(\mom_0=\mathbf 0\), the lagged sum
\(\sum_{t<n\wedge\sigma}\|\lambda_{t-1}\odot\mom_{t-1}\|^2\) consists of
the terms of the current sum with \(t\le(n\wedge\sigma)-2\), so it obeys
the same bound. All three bounds hold pathwise, hence also at the random
truncation index \(n\wedge\sigma\).
Adding the three inequalities in \eqref{eq:normalized-three-energies}
gives \eqref{eq:stopped-normalized-energies}.
\end{proof}

\subsection{The localized augmented-energy inequality}

Fix \(0<\delta<1\), put \(U=U_\delta\) as defined in
\eqref{eq:coefficientwise-confidence-height}, and let \(\bar\alpha\)
satisfy \eqref{eq:coefficientwise-stepsize}. With \(\Psi\) from
\eqref{eq:auxiliary-objective} and \(\q_t\) from \eqref{eq:auxiliary-sequence}, define
\begin{equation}
\sigma_U=\inf\{1\le t\le T:\Psi(\q_t)>U\},
\qquad \inf\varnothing=\infty,
\label{eq:coefficientwise-stopping-index}
\end{equation}
and write, with \(\lambda_{t-1,i}\) from Algorithm~\ref{alg:adam},
\begin{equation}
h_t=\nabla f(\w_t),\quad \xi_t=\dir_t-h_t,\quad
E_t=\sum_i\lambda_{t-1,i}h_{t,i}^2,\quad
\mathcal L_t=\Psi(\q_t)+E_t.
\label{eq:coefficientwise-energy-processes}
\end{equation}
For \(\mathcal R\) defined in \eqref{eq:coefficientwise-calibration-scale},
the calibration \eqref{eq:coefficientwise-stepsize} gives
\begin{equation}
2\le\mathcal R(U)\le\mathcal X,\qquad
\mathcal X:=(2^{100}d\bar\alpha\,(1+\log(16/\delta))^{\mathbf 1_{\{\eta_{\mathrm{gen}}>0\}}})^{-1/48}.
\label{eq:coefficientwise-scale-domination}
\end{equation}
The domination scale \(\mathcal X\) of \eqref{eq:coefficientwise-scale-domination}
depends only on \(d\), \(\bar\alpha\) and \(\delta\), with the stepsize
excess factor \(\eta_{\mathrm{gen}}\) defined in \eqref{eq:stepsize-excess-appendix};
the lower bound
\(\mathcal X\ge2\) follows from \(\mathcal P\ge3\)
(see \eqref{eq:coefficientwise-parameter-scale}) and
\(\mathcal R(U)\ge\mathcal P\).
Thus every nonnegative summand in
\eqref{eq:coefficientwise-calibration-scale} is bounded by the same
upper bound. For every \(0\le k\le48\),
\begin{equation}
\begin{aligned}
&d\bar\alpha\,(1+\log(16/\delta))^{\mathbf 1_{\{\eta_{\mathrm{gen}}>0\}}}
\mathcal X^{k}\\
&\qquad=2^{-100}\mathcal X^{k-48}\le2^{-100}.
\end{aligned}
\label{eq:stepsize-power-absorption}
\end{equation}
Since \(d\ge1\) and \(1+\log(16/\delta)\ge1\), the same bound holds with
\(d\bar\alpha\,(1+\log(16/\delta))^{\mathbf 1_{\{\eta_{\mathrm{gen}}>0\}}}\) replaced by
\(d\bar\alpha\), by \(\sqrt d\,\bar\alpha\) or by \(\bar\alpha\).

In both cases \(\eta_{\mathrm{gen}}=0\) and \(\eta_{\mathrm{gen}}>0\), the confidence
height \eqref{eq:coefficientwise-confidence-height}, \(\mathcal P\le\mathcal R(U)\) and
\eqref{eq:coefficientwise-scale-domination} give
\begin{equation}
\begin{aligned}
(1+\log(16/\delta))^{\mathbf 1_{\{\eta_{\mathrm{gen}}=0\}}}&\le U^{1/4},\\
U^{1/4}&\le2^5\mathcal P^2(1+\log(16/\delta))\\
&\le2^5\mathcal X^{2}(1+\log(16/\delta)).
\end{aligned}
\label{eq:confidence-height-facts}
\end{equation}
Consequently, for every \(0\le k\le48\) and every \(0\le b\le3/4\),
\begin{equation}
d\bar\alpha\,(1+\log(16/\delta))
\mathcal X^{k}U^b\le2^{-100}U.
\label{eq:confidence-log-absorption}
\end{equation}
Indeed, if \(\eta_{\mathrm{gen}}>0\), then \eqref{eq:confidence-log-absorption} is
\eqref{eq:stepsize-power-absorption} multiplied by \(U^b\le U\). If
\(\eta_{\mathrm{gen}}=0\), then \eqref{eq:stepsize-power-absorption} holds with
\(d\bar\alpha\) alone, and the first line of \eqref{eq:confidence-height-facts} gives
\((1+\log(16/\delta))U^b\le U^{b+1/4}\le U\).
In the closing steps below, the factor \(1+\log(16/\delta)\) appears with
power at most one and is absorbed through \eqref{eq:confidence-log-absorption}.

The stepsize \eqref{eq:coefficientwise-stepsize} and
the domination \eqref{eq:coefficientwise-scale-domination} ensure all
local-distance requirements, because
\[
(1+6\frac{\beta_1}{1-\beta_1})\sqrt d L_p\bar\alpha
\le\mathcal X^{2}d\bar\alpha<1.
\]
For the displacement of \(\w_t\), the momentum expansion more precisely
gives \(\|\lambda_t\odot\mom_t\|\le(1+\beta_1)\sqrt d\,\bar\alpha
\le(1+6\frac{\beta_1}{1-\beta_1})\sqrt d\,\bar\alpha\); this also covers small values of
\(\beta_1\). The displacements of \(\q_t\) (defined in
\eqref{eq:auxiliary-sequence}) and the distance between
\(\q_t\) and \(\w_t\) satisfy the same local-radius requirement.
In particular, before \(\sigma_U\) (see
\eqref{eq:coefficientwise-stopping-index}), Lemmas~\ref{lem:barf_comparison}
and~\ref{lem:gradient-gap-growth}, with the definitions
\eqref{eq:gradient-growth-constant}--\eqref{eq:gap-comparison-constant}
and \eqref{eq:coefficientwise-local-scales}, give
\[
\Psi(\w_t)\le C_{\mathrm q}U^{\max\{1,p/(2-p)\}},\qquad
\|h_t\|,\|\nabla f(\q_t)\|\le G_{\mathrm q}(U).
\]
With \(G_{\mathrm q}\) and \(\mathcal R\) defined in
\eqref{eq:coefficientwise-local-scales}--\eqref{eq:coefficientwise-calibration-scale},
the corresponding smoothness coefficient is at most
\[
L_0+L_p\max\{1,G_{\mathrm q}(U)^p\}\le\mathcal R(U)\le \mathcal X.
\]
Conditional unbiasedness and the variance envelope in
\eqref{eq:coefficientwise-local-scales} give, for the noise \(\xi_t\) from
\eqref{eq:coefficientwise-energy-processes},
\begin{align*}
\mathbb E_{t-1}\|\xi_t\|^2&=\mathbb E_{t-1}\|\dir_t\|^2-\|h_t\|^2\\
&\le C+\mathcal V_{\mathrm q}(U)\le\mathcal R(U)
\le \mathcal X.
\end{align*}
Here and below \(\mathbb E_{t-1}\) denotes conditional expectation
given \(\mathscr F_{t-1}\).

For \(t<\sigma_U\), with \(\mathcal L_t\) and \(E_t\) from
\eqref{eq:coefficientwise-energy-processes}, the following inequality holds:
\begin{equation}
\mathcal L_{t+1}-\mathcal L_t
\le-\tfrac18 E_t+D_{t,1}+D_{t,2}+D_{t,3}
+\sum_{j=1}^7P_{t,j},
\label{eq:localized-augmented-descent}
\end{equation}
where, with \(\Delta_{t,i}\) from
\eqref{eq:coefficientwise-step-variation-definition},
\begin{equation}
\begin{aligned}
D_{t,1}&=\mathbb E_{t-1}\sum_i\lambda_{t,i}h_{t,i}\dir_{t,i}
-\sum_i\lambda_{t,i}h_{t,i}\dir_{t,i},\\
D_{t,2}&=\sum_i h_{t,i}^2(\mathbb E_{t-1}\Delta_{t,i}-\Delta_{t,i}),\\
D_{t,3}&=\frac{\beta_1}{1-\beta_1}\sum_i[(\lambda_{t-1,i}-\lambda_{t,i})
-\mathbb E_{t-1}(\lambda_{t-1,i}-\lambda_{t,i})]
\nabla_i f(\q_t)\mom_{t-1,i},
\end{aligned}
\label{eq:coefficientwise-centered-terms}
\end{equation}
and, with the domination scale \(\mathcal X\) from
\eqref{eq:coefficientwise-scale-domination},
\begin{equation}
\begin{aligned}
P_{t,1}&=\biggl(\frac{\beta_1^2\mathcal X}{2(1-\beta_1)^2}\\
&\qquad+\frac{\beta_1^2\mathcal X^{2}\bar\alpha}{4(1-\beta_1)^2\sqrt{\binit}}\biggr)
\|\lambda_{t-1}\odot\mom_{t-1}\|^2,\\
P_{t,2}&=\frac{3\mathcal X}{2}\|\lambda_t\odot\dir_t\|^2,\\
P_{t,3}&=\frac{\beta_1^2}{(1-\beta_1)^2}\mathcal X\sum_i(\lambda_{t-1,i}-\lambda_{t,i})^2\mom_{t-1,i}^2,\\
P_{t,4}&=\frac{140d\mathcal X^{2}\bar\alpha}{\sqrt{\binit}}
\|\lambda_t\odot\mom_t\|^2,\\
P_{t,5}&=\mathcal X\sum_i\mathbb E_{t-1}\Delta_{t,i},\\
P_{t,6}&=\frac{8\beta_1(1+\beta_1)}{(1-\beta_1)^2}
\frac{T}{T-1}\sum_i\biggl[
\frac{\lambda_{t-1,i}^2\mom_{t-1,i}^2}{\bar\alpha}\\
&\qquad\times\Bigl(|h_{t,i}|+\mathcal X^{1/2}+\frac{\epsilon}{\sqrt T+\sqrt{T-1}}\Bigr)\biggr],\\
P_{t,7}&=\frac{\beta_1}{T(1-\beta_1)}\sum_i\lambda_{t-1,i}
|\nabla_i f(\q_t)\mom_{t-1,i}|.
\end{aligned}
\label{eq:coefficientwise-seven-residuals}
\end{equation}

To prove \eqref{eq:localized-augmented-descent}, start from the exact
identity, with \(\q_t\) from \eqref{eq:auxiliary-sequence},
\[
\q_{t+1}-\q_t=-\lambda_t\odot\dir_t
+\frac{\beta_1}{1-\beta_1}(\lambda_{t-1}-\lambda_t)\odot\mom_{t-1}.
\]
Generalized smoothness at \(\q_t\), and, for \(h_t\) from
\eqref{eq:coefficientwise-energy-processes},
\(\|\nabla f(\q_t)-h_t\|
\le \frac{\beta_1}{1-\beta_1}\mathcal X\|\lambda_{t-1}\odot\mom_{t-1}\|\), give
\begin{align*}
\Psi(\q_{t+1})-\Psi(\q_t)
\le&-\sum_i\lambda_{t,i}h_{t,i}\dir_{t,i}
+\frac{\beta_1}{1-\beta_1}\sum_i(\lambda_{t-1,i}-\lambda_{t,i})
\nabla_i f(\q_t)\mom_{t-1,i}\\
&+\frac{\beta_1^2\mathcal X}{2(1-\beta_1)^2}\|\lambda_{t-1}\odot\mom_{t-1}\|^2\\
&+\frac{3\mathcal X}{2}\|\lambda_t\odot\dir_t\|^2\\
&+\frac{\beta_1^2\mathcal X}{(1-\beta_1)^2}\| (\lambda_{t-1}-\lambda_t)\odot\mom_{t-1}\|^2.
\end{align*}
Center the first complete adaptive product. By the definitions in
\eqref{eq:coefficientwise-step-variation-definition},
\(\lambda_t=\widetilde\lambda_{t-1}-\Delta_t\).
Conditional Cauchy--Schwarz and \(\Delta_{t,i}\le2\lambda_{t-1,i}\) give
\[
|h_{t,i}|\mathbb E_{t-1}[\Delta_{t,i}|\dir_{t,i}|]
\le\tfrac12\lambda_{t-1,i}h_{t,i}^2
+\mathbb E_{t-1}\Delta_{t,i}\,\mathbb E_{t-1}\dir_{t,i}^2.
\]
The predictable first product is therefore at most
\(-E_t/2+\sum_i(h_{t,i}^2+\mathcal X)\mathbb E_{t-1}\Delta_{t,i}\),
with \(E_t\) from \eqref{eq:coefficientwise-energy-processes}.
The realized part of the latter sum has the exact decomposition
\begin{align*}
\sum_i h_{t,i}^2\Delta_{t,i}
=&E_t-E_{t+1}
+\sum_i(\widetilde\lambda_{t-1,i}-\lambda_{t-1,i})h_{t,i}^2\\
&+\sum_i\lambda_{t,i}(h_{t+1,i}^2-h_{t,i}^2).
\end{align*}
For \(T\ge10\), the middle term is at most \(E_t/16\).
Put \(q=h_{t+1}-h_t\), so
\(\|q\|\le \mathcal X\|\lambda_t\odot\mom_t\|\) with \(\mathcal X\) from
\eqref{eq:coefficientwise-scale-domination}. Expanding the last term
and using \(\lambda_t\le2\lambda_{t-1}\) gives
\[
\sum_i\lambda_{t,i}(2h_{t,i}q_i+q_i^2)
\le\tfrac18 E_t+
\frac{34\mathcal X^{2}\bar\alpha}{\sqrt{T\binit}}
\|\lambda_t\odot\mom_t\|^2.
\]
Consequently the realized sum is bounded by
\(E_t-E_{t+1}+E_t/4+P_{t,4}\), with \(E_t\) and \(P_{t,4}\)
defined in \eqref{eq:coefficientwise-energy-processes} and
\eqref{eq:coefficientwise-seven-residuals}.

For the lagged-momentum product, centering gives \(D_{t,3}\) from
\eqref{eq:coefficientwise-centered-terms}.
The deterministic difference
\(\lambda_{t-1}-\widetilde\lambda_{t-1}\) has magnitude at most
\(\lambda_{t-1}/T\), yielding \(P_{t,7}\) in
\eqref{eq:coefficientwise-seven-residuals}.
For \(\Delta_{t,i}\) from \eqref{eq:coefficientwise-step-variation-definition},
direct rationalization gives
\[
\mathbb E_{t-1}\Delta_{t,i}
\le\frac{T}{T-1}
\frac{|h_{t,i}|+\mathcal X^{1/2}+\epsilon/(\sqrt T+\sqrt{T-1})}{\bar\alpha}
\lambda_{t-1,i}^2.
\]
Since \(\q_t\), \(\mom_{t-1}\) and \(\lambda_{t-1}\) are
\(\mathscr F_{t-1}\)-measurable, the remaining part is
\[
\frac{\beta_1}{1-\beta_1}\sum_i\mathbb E_{t-1}\Delta_{t,i}\,
\nabla_i f(\q_t)\mom_{t-1,i}.
\]
Fix \(i\). Apply the arithmetic--geometric mean inequality to
\(\tfrac14\lambda_{t-1,i}^{1/2}|\nabla_i f(\q_t)|\) and
\(\frac{2\beta_1}{1-\beta_1}\lambda_{t-1,i}^{-1/2}
\mathbb E_{t-1}\Delta_{t,i}|\mom_{t-1,i}|\), whose product is half of
the left side below. Then use
\((\mathbb E_{t-1}\Delta_{t,i})^2\le2\lambda_{t-1,i}\mathbb E_{t-1}\Delta_{t,i}\),
which follows from \(0\le\Delta_{t,i}\le2\lambda_{t-1,i}\). This gives
\begin{align*}
\frac{\beta_1}{1-\beta_1}\mathbb E_{t-1}\Delta_{t,i}
|\nabla_i f(\q_t)||\mom_{t-1,i}|
&\le\tfrac1{16}\lambda_{t-1,i}(\nabla_i f(\q_t))^2
+\frac{4\beta_1^2}{(1-\beta_1)^2}
\frac{(\mathbb E_{t-1}\Delta_{t,i})^2}{\lambda_{t-1,i}}\mom_{t-1,i}^2\\
&\le\tfrac1{16}\lambda_{t-1,i}(\nabla_i f(\q_t))^2
+\frac{8\beta_1^2}{(1-\beta_1)^2}\mathbb E_{t-1}\Delta_{t,i}\,
\mom_{t-1,i}^2.
\end{align*}
The rationalization bound above and
\(\beta_1^2\le\beta_1(1+\beta_1)\) give
\begin{align*}
&\frac{8\beta_1^2}{(1-\beta_1)^2}\mathbb E_{t-1}\Delta_{t,i}\,
\mom_{t-1,i}^2
\le\frac{8\beta_1(1+\beta_1)}{(1-\beta_1)^2}\frac{T}{T-1}
\frac{\lambda_{t-1,i}^2\mom_{t-1,i}^2}{\bar\alpha}\\
&\qquad\times
\Bigl(|h_{t,i}|+\mathcal X^{1/2}+\frac{\epsilon}{\sqrt T+\sqrt{T-1}}\Bigr).
\end{align*}
Summing over \(i\), the remaining part is at most
\(\tfrac1{16}\sum_i\lambda_{t-1,i}(\nabla_i f(\q_t))^2+P_{t,6}\),
where \(P_{t,6}\) is defined in \eqref{eq:coefficientwise-seven-residuals}.
Finally, with \(E_t\) from \eqref{eq:coefficientwise-energy-processes},
\[
\tfrac1{16}\sum_i\lambda_{t-1,i}(\nabla_i f(\q_t))^2
\le\tfrac18E_t+
\frac{\beta_1^2\mathcal X^{2}\bar\alpha}{4(1-\beta_1)^2\sqrt{\binit}}
\|\lambda_{t-1}\odot\mom_{t-1}\|^2.
\]
Combining the three descent coefficients gives
\(-1/2+1/4+1/8=-1/8\), proving the claim. Multiplication by
\(\mathbf1_{\{t<\sigma_U\}}\) preserves the martingale-difference
property because this event is \(\mathscr F_{t-1}\)-measurable.

\subsection{Uniform bounds for the stopped gradients}

For the stopping index and gradients defined in
\eqref{eq:coefficientwise-stopping-index}--\eqref{eq:coefficientwise-energy-processes},
before \(\sigma_U\), with \(\mathcal X\) from
\eqref{eq:coefficientwise-scale-domination} and \(U=U_\delta\) from
\eqref{eq:coefficientwise-confidence-height},
\begin{equation}
\|h_t\|,\ \|\nabla f(\q_t)\|
\le6\mathcal X^{3}\sqrt U.
\label{eq:uniform-stopped-gradient}
\end{equation}
For \(p>0\), the last term of
\eqref{eq:coefficientwise-calibration-scale} and
\eqref{eq:coefficientwise-scale-domination} imply
\[
\sqrt{\frac{2G_{\mathrm q}(U)^2}{U+1}}\le\mathcal X.
\]
Since \(U+1\le2U\), the bound \(\|h_t\|,\|\nabla f(\q_t)\|\le G_{\mathrm q}(U)\)
before \(\sigma_U\) gives
\begin{align*}
\|h_t\|,\ \|\nabla f(\q_t)\|&\le\mathcal X\sqrt U\\
&\le6\mathcal X^{3}\sqrt U,
\end{align*}
which is \eqref{eq:uniform-stopped-gradient} for \(p>0\).
The definitions \eqref{eq:gradient-growth-constant}--\eqref{eq:gap-comparison-constant}
give \(K,C_{\mathrm q}\ge1\); with \(G_{\mathrm q}\) from
\eqref{eq:coefficientwise-local-scales}, this gives \(G_{\mathrm q}(U)\ge U\), hence
\begin{equation}
\mathcal X\ge\sqrt{\frac{2G_{\mathrm q}(U)^2}{U+1}}\ge\sqrt U
\qquad(p>0).
\label{eq:positive-p-excess}
\end{equation}
For \(p=0\), the local Lipschitz condition extends by partitioning any
segment to global \((L_0+L_p)\)-smoothness. Before \(\sigma_U\),
\(\Psi(\q_t)\le U\) by \eqref{eq:coefficientwise-stopping-index}, and
\eqref{eq:coefficientwise-scale-domination}
gives \(L_0+L_p\le\mathcal P\le\mathcal X\), with \(\mathcal P\) from
\eqref{eq:coefficientwise-parameter-scale}. Hence
\[
\|\nabla f(\q_t)\|\le\sqrt{2(L_0+L_p)U}
\le\sqrt{2\mathcal XU}.
\]
Since \((L_0+L_p)\beta_1/(1-\beta_1)\le\mathcal P^2
\le\mathcal X^{2}\),
\eqref{eq:stepsize-power-absorption} gives
\[
\|h_t-\nabla f(\q_t)\|
\le\frac{2(L_0+L_p)\beta_1}{1-\beta_1}\sqrt d\,\bar\alpha
\le2\sqrt d\,\bar\alpha\mathcal X^{2}\le1.
\]
With \(U\ge1\) (see \eqref{eq:coefficientwise-confidence-height}) and
\(\mathcal X\ge1\),
\begin{align*}
\|h_t\|&\le\sqrt{2\mathcal XU}+1\\
&\le6\mathcal X^{3}\sqrt U.
\end{align*}
The same bound holds for \(\|\nabla f(\q_t)\|\). This proves
\eqref{eq:uniform-stopped-gradient} in this case. The all-path displacement bounds
\[
\|\w_t-\w_1\|\le2\sqrt d\,\bar\alpha T,
\qquad
\|\q_t-\q_1\|\le\left(1+\frac{6\beta_1}{1-\beta_1}\right)
\sqrt d\,\bar\alpha T
\]
also give, before \(\sigma_U\),
\begin{equation}
\|h_t\|^2,\ \|\nabla f(\q_t)\|^2
\le36\mathcal X^{6}\min\{U,1+d\bar\alpha^2T^2\}
\qquad(p=0).
\label{eq:classical-gradient-time-bound}
\end{equation}
The elementary inequality
\begin{equation}
\frac{\min\{U,d\bar\alpha^2T^2\}}{\sqrt T}
\le d^{1/4}\sqrt{\bar\alpha}\,U^{3/4}
\label{eq:time-interpolation}
\end{equation}
follows by considering \(T\le\sqrt U/(\sqrt d\,\bar\alpha)\) and its complement.

\subsection{Four stopped concentration estimates}
\label{app:stopped-concentration}

Throughout this subsection \(U=U_\delta\) is the height
\eqref{eq:coefficientwise-confidence-height}, and every closing step
absorbs the factor \(1+\log(16/\delta)\) into \(U\) only through
\eqref{eq:confidence-log-absorption}. Use \(\sigma=\sigma_U\), the stopping index defined in
\eqref{eq:coefficientwise-stopping-index}, in
\eqref{eq:stopped-logarithmic-energy} and
\eqref{eq:normalized-three-energies}.

First, \eqref{eq:uniform-stopped-gradient} and the noise envelope
\eqref{eq:coefficientwise-local-scales}, bounded through
\eqref{eq:coefficientwise-scale-domination}, show that the stopped
raw second-moment sum has expectation at most
\(T[36\mathcal X^{6}U+\mathcal X]\). Markov's inequality and Jensen's
inequality across coordinates give, for \(Z_T\) from
\eqref{eq:stopped-logarithmic-energy}, except on an event of probability
\(e^{-(1+\log(16/\delta))}\),
\begin{equation}
\begin{aligned}
Z_T&\le d\{(1+\log(16/\delta))+\log[1+37\mathcal X^{7}U]\}\\
&\le2^7d\mathcal X^{2}(1+\log(16/\delta)).
\end{aligned}
\label{eq:log-good-event}
\end{equation}
For the last inequality, bound \(1+37\mathcal X^{7}U\) by
\(38\mathcal X^{7}U\) and split its logarithm into three parts.
Since \(\mathcal X\ge2\) by \eqref{eq:coefficientwise-scale-domination},
\(\log(38\mathcal X)\le4\mathcal X\) and
\(6\log(\mathcal X)\le6\mathcal X\), while
\(\log U\le2U^{1/4}\) and \eqref{eq:confidence-height-facts} give
\(\log U\le2^6\mathcal X^{2}(1+\log(16/\delta))\).
Using \(\mathcal X\ge2\) and \(1+\log(16/\delta)\ge1\) once more,
\(1\le\mathcal X^{2}/4\) and \(10\mathcal X\le5\mathcal X^{2}\), so the braces are at most
\((1/4+5+2^6)\mathcal X^{2}(1+\log(16/\delta))\le2^7\mathcal X^{2}(1+\log(16/\delta))\).

Second, Lemma~\ref{lem:cond_to_direct}, at exponent
\(1+\log(16/\delta)\), and
\eqref{eq:direct-variation-bound} for
\(\Delta_{t,i}\) from \eqref{eq:coefficientwise-step-variation-definition}
imply, with \(\sigma_U\) from \eqref{eq:coefficientwise-stopping-index},
except on an event of probability
\(e^{-(1+\log(16/\delta))}\),
\begin{equation}
\sum_{t<T}\mathbf1_{\{t<\sigma_U\}}
\sum_i\mathbb E_{t-1}\Delta_{t,i}
\le\frac{2e(1+\log(16/\delta))d\bar\alpha}{\sqrt{T\binit}}.
\label{eq:compensator-good-event}
\end{equation}
Indeed the \(L^{1+\log(16/\delta)}\) norm of the left side is at most \((1+\log(16/\delta))\) times the
deterministic upper bound for the corresponding realized stopped sum;
Markov's inequality at \(e\) times that norm gives the stated probability.

Third, the first two martingales in
\eqref{eq:coefficientwise-centered-terms} combine exactly as
\begin{equation*}
D_{t,12}:=D_{t,1}+D_{t,2}
=\mathbb E_{t-1}\sum_i\lambda_{t,i}h_{t,i}\xi_{t,i}
-\sum_i\lambda_{t,i}h_{t,i}\xi_{t,i}.
\end{equation*}
Here \(\lambda_{t,i}\) is from Algorithm~\ref{alg:adam} and
\(h_t,\xi_t\) are defined in \eqref{eq:coefficientwise-energy-processes}.
Using the gradient estimate \eqref{eq:uniform-stopped-gradient},
the noise envelope \eqref{eq:coefficientwise-local-scales}, and
the coefficient bound \eqref{eq:coefficientwise-scale-domination},
the stopped predictable quadratic variation is at most
\[
144\bar\alpha^2\mathcal X^{8}U.
\]
To see this, bound each conditional variance by the uncentered square,
then use Cauchy--Schwarz, \(\lambda_{t,i}\le2\bar\alpha/
\sqrt{T\binit}\), and \(\mathbb E_{t-1}\|\xi_t\|^2\le \mathcal X\).
The stopped increment magnitude is at most
\begin{align*}
&12\sqrt d\,\bar\alpha\mathcal X^{4}\sqrt U\\
&\qquad+144\bar\alpha\mathcal X^{7}\frac U{\sqrt T}.
\end{align*}
For \(p=0\), \eqref{eq:classical-gradient-time-bound} permits replacing
\(U\) in the second summand by \(\min\{U,1+d\bar\alpha^2T^2\}\).
The one-sided Freedman inequality \eqref{eq:freedman-tool} applies at failure level
\(e^{-(1+\log(16/\delta))}\). Its square-root contribution is at
most \(17\bar\alpha\mathcal X^{4}(1+\log(16/\delta))^{1/2}U^{1/2}\),
since \(\sqrt{2\cdot144}\le17\); the first jump contribution is at most
\(8\sqrt d\,\bar\alpha\mathcal X^{4}(1+\log(16/\delta))U^{1/2}\).
For \(p>0\), \eqref{eq:positive-p-excess} gives
\(U\le\mathcal X\sqrt U\) and therefore bounds the remaining jump
contribution by \(96\bar\alpha\mathcal X^{8}(1+\log(16/\delta))U^{1/2}\).
For \(p=0\), use the replacement above,
\(\min\{U,1+d\bar\alpha^2T^2\}\le1+\min\{U,d\bar\alpha^2T^2\}\),
\eqref{eq:time-interpolation},
\(d^{1/4}\sqrt{\bar\alpha}\le\sqrt{d\bar\alpha}\le1\) and \(U\ge1\); they bound it by
\(192\bar\alpha\mathcal X^{7}(1+\log(16/\delta))U^{3/4}\).
Since \(d,U,1+\log(16/\delta)\ge1\) and
\(\mathcal X\ge2\) by \eqref{eq:coefficientwise-scale-domination},
in both cases the three contributions sum to at most
\[
217\sqrt d\,\bar\alpha\mathcal X^{8}(1+\log(16/\delta))U^{3/4}<U/32,
\]
the last step by \(\sqrt d\le d\) and \eqref{eq:confidence-log-absorption}, for the height
\(U=U_\delta\) of \eqref{eq:coefficientwise-confidence-height}. Therefore, except on an event
of probability \(e^{-(1+\log(16/\delta))}\),
\begin{equation*}
\sum_{t<T}\mathbf1_{\{t<\sigma_U\}}D_{t,12}\le U/32.
\end{equation*}

Fourth, for \(Z_n\) from \eqref{eq:stopped-logarithmic-energy},
introduce the predictable crossing
\begin{equation}
\begin{aligned}
\tau&=\inf\{1\le n\le T:Z_n>2^7d\mathcal X^{2}(1+\log(16/\delta))\},\\
\inf\varnothing&=\infty.
\end{aligned}
\label{eq:coefficientwise-log-crossing}
\end{equation}
Stop \(D_{t,3}\) from \eqref{eq:coefficientwise-centered-terms}
at \(t<\sigma_U\) and \(t<\tau\). Conditional on
the past, \(\nabla f(\q_t)\) and \(\mom_{t-1}\) are fixed (\(\q_t\) is the
auxiliary sequence \eqref{eq:auxiliary-sequence} and \(\mom_{t-1}\) the
momentum of Algorithm~\ref{alg:adam}), while
\(|\lambda_{t-1,i}-\lambda_{t,i}|\le3\lambda_{t-1,i}\). Its quadratic
variation is bounded, using \eqref{eq:uniform-stopped-gradient}
and \eqref{eq:normalized-three-energies}, by
\begin{align*}
&\frac{9\beta_1^2}{(1-\beta_1)^2}[36\mathcal X^{6}U]
\sum_{t<T}\mathbf1_{\{t<\sigma_U,t<\tau\}}
\|\lambda_{t-1}\odot\mom_{t-1}\|^2\\
&\qquad\le2^{22}d\bar\alpha^2\mathcal X^{10}(1+\log(16/\delta))U.
\end{align*}
The summand at time \(t\) uses \(\mom_{t-1}\). Let \(t\) be an
included time, that is, \(t<T\), \(t<\sigma_U\) and \(t<\tau\).
Since \(\mom_0=\mathbf 0\),
\(\sum_{s\le t}\|\lambda_{s-1}\odot\mom_{s-1}\|^2
=\sum_{s<t}\|\lambda_s\odot\mom_s\|^2\).
The momentum bound in \eqref{eq:normalized-three-energies}, with
\(\sigma=\sigma_U\) and \(n=t\), for which \(n\wedge\sigma_U=t\), bounds
this sum by \(64\bar\alpha^2Z_t\). Because \(t<\tau\),
\eqref{eq:coefficientwise-log-crossing} gives
\(Z_t\le2^7d\mathcal X^{2}(1+\log(16/\delta))\). Taking \(t\) to be the
largest included time bounds the stopped sum above.
The increment magnitude is at most
\(72\sqrt d\,\bar\alpha\mathcal X^{5}\sqrt U\). Freedman's inequality
\eqref{eq:freedman-tool}
therefore gives an upper deviation at most
\[
2^{12}\sqrt d\,\bar\alpha\mathcal X^{5}(1+\log(16/\delta))U^{1/2}<U/32
\]
with failure probability \(e^{-(1+\log(16/\delta))}\); here
\(\sqrt{2\cdot2^{22}}+\tfrac23\cdot72\le2^{12}\), and the last step uses
\(\sqrt d\le d\) and \eqref{eq:confidence-log-absorption}. On \eqref{eq:log-good-event},
\(\tau>T\) by \eqref{eq:coefficientwise-log-crossing}, so this is the desired estimate for the stopped
\(D_{t,3}\) sum.

A union bound gives all four events, namely \eqref{eq:log-good-event},
\eqref{eq:compensator-good-event} and the two Freedman bounds for the
stopped \(D_{t,12}\) and \(D_{t,3}\) sums, simultaneously with probability
at least \(1-4e^{-(1+\log(16/\delta))}>1-\delta/2\).

\subsection{Closing the stopped descent}

\begin{lemma}[Stopped energy and localization]
\label{lem:stopped-energy-localization}
Assume Assumptions~\ref{ass:nonneg}--\ref{ass:abc}, \(T\ge10\),
\(\beta_2=1-1/T\), and \(\alpha=\bar\alpha/\sqrt T\).
For \(0<\delta<1\), take \(U=U_\delta\) from
\eqref{eq:coefficientwise-confidence-height} and \(\bar\alpha\)
in the interval \eqref{eq:coefficientwise-stepsize}.
With the stopping index and energies defined in
\eqref{eq:coefficientwise-stopping-index}--\eqref{eq:coefficientwise-energy-processes},
with probability at least \(1-\delta/2\),
\[
\mathcal L_{\sigma_U\wedge T}
+\frac18\sum_{t<\sigma_U\wedge T}E_t\le U/2,
\qquad \sigma_U>T,\qquad
\sum_{t=1}^{T-1}E_t\le4U.
\]
\end{lemma}
\begin{proof}
On \eqref{eq:log-good-event} and \eqref{eq:compensator-good-event},
\eqref{eq:normalized-three-energies},
\eqref{eq:uniform-stopped-gradient} and the residual definitions
\eqref{eq:coefficientwise-seven-residuals}
give the following cumulative bounds, with \(\mathcal X\) from
\eqref{eq:coefficientwise-scale-domination} and all sums being over
\(t<\sigma_U\wedge T\) for \(\sigma_U\) from \eqref{eq:coefficientwise-stopping-index}:
\begin{align*}
\sum_tP_{t,1}&\le2^{13}d\bar\alpha^2\mathcal X^{7}(1+\log(16/\delta)),\\
\sum_tP_{t,2}&\le2^{10}d\bar\alpha^2\mathcal X^{3}(1+\log(16/\delta)),\\
\sum_tP_{t,3}&\le2^{17}d\bar\alpha^2\mathcal X^{5}(1+\log(16/\delta)),\\
\sum_tP_{t,4}&\le2^{21}d^2\bar\alpha^3\mathcal X^{6}(1+\log(16/\delta)),\\
\sum_tP_{t,5}&\le6d\bar\alpha\mathcal X^{3}(1+\log(16/\delta)),\\
\sum_tP_{t,6}&\le2^{22}d\bar\alpha\mathcal X^{7}(1+\log(16/\delta))U^{1/2},\\
\sum_tP_{t,7}&\le12\sqrt d\,\bar\alpha\mathcal X^{5}U^{1/2}.
\end{align*}
For example, the coefficient in \(P_{t,6}\) from
\eqref{eq:coefficientwise-seven-residuals} is at most
\(512\mathcal X^{5}\sqrt U/\bar\alpha\), and its lagged-momentum
sum is at most \(64\bar\alpha^2 Z_T\), with \(Z_T\) defined in
\eqref{eq:stopped-logarithmic-energy}. This gives exactly its displayed
bound. For \(P_{t,5}\), use \eqref{eq:compensator-good-event}; for
\(P_{t,7}\), use \eqref{eq:coefficientwise-momentum-displacement} directly.
The other four use \eqref{eq:normalized-three-energies} together with
\(64\bar\alpha^2Z_T\le2^{13}d\bar\alpha^2\mathcal X^{2}(1+\log(16/\delta))\)
from \eqref{eq:log-good-event}. By
\eqref{eq:coefficientwise-scale-domination},
\(\mathcal P\le\mathcal R(U)\le\mathcal X\) (with \(\mathcal P\) and \(\mathcal R\)
from \eqref{eq:coefficientwise-parameter-scale} and
\eqref{eq:coefficientwise-calibration-scale}) and
\(\mathcal X\ge2\), so
the power \(\mathcal X^{k}\) is nondecreasing in
\(k\ge0\); this is used below to enlarge exponents.
For \(P_{t,1}\), \(\beta_1^2/(1-\beta_1)^2\le\mathcal P
\le\mathcal X\),
\(\binit^{-1/2}\le1+\binit^{-1}\le\mathcal P\) and
\(\mathcal X^{4}\bar\alpha\le2^{-100}\) from
\eqref{eq:stepsize-power-absorption} bound its coefficient by
\begin{align*}
&\frac{\mathcal X^{2}}2
+\frac{\mathcal X^{4}\bar\alpha}4\\
&\qquad\le\mathcal X^{2},
\end{align*}
so \(\sum_tP_{t,1}\le2^{13}d\bar\alpha^2\mathcal X^{4}(1+\log(16/\delta))\);
enlarging the exponent gives the displayed bound. For \(P_{t,2}\), with
\(Z_T\) from \eqref{eq:stopped-logarithmic-energy},
\begin{align*}
\sum_tP_{t,2}&\le\frac{3\mathcal X}2\cdot4\bar\alpha^2Z_T\\
&\le768d\bar\alpha^2\mathcal X^{3}(1+\log(16/\delta)),
\end{align*}
and \(768<2^{10}\) gives the displayed bound.
For \(P_{t,3}\), with \(\lambda_t\) from Algorithm~\ref{alg:adam},
\(|\lambda_{t-1}-\lambda_t|\le3\lambda_{t-1}\) and
\(\beta_1^2/(1-\beta_1)^2\le\mathcal X\) give
\begin{align*}
\sum_tP_{t,3}&\le9\mathcal X^{2}\cdot64\bar\alpha^2Z_T\\
&\le9\cdot2^{13}d\bar\alpha^2\mathcal X^{4}(1+\log(16/\delta)),
\end{align*}
and \(9\cdot2^{13}<2^{17}\) with an enlarged exponent gives the displayed
bound. For \(P_{t,4}\), \(\binit^{-1/2}\le\mathcal P
\le\mathcal X\) gives
\begin{align*}
\sum_tP_{t,4}&\le140d\bar\alpha\mathcal X^{3}
\cdot64\bar\alpha^2Z_T\\
&\le140\cdot2^{13}d^2\bar\alpha^3\mathcal X^{5}(1+\log(16/\delta)),
\end{align*}
and \(140\cdot2^{13}<2^{21}\) with an enlarged exponent gives the
displayed bound.
By \eqref{eq:stepsize-power-absorption}, \(d\bar\alpha\le1\). Together with
\(d,U,1+\log(16/\delta)\ge1\) and the monotonicity in \(k\) noted above,
each of the seven bounds is at most its constant times
\(d\bar\alpha\mathcal X^{7}(1+\log(16/\delta))U^{3/4}\). The seven constants
sum to less than \(2^{23}\), so the seven bounds sum to at most
\[
2^{23}d\bar\alpha\mathcal X^{7}(1+\log(16/\delta))U^{3/4}<U/32,
\]
the last step by \eqref{eq:confidence-log-absorption}.
The energy definition \eqref{eq:coefficientwise-energy-processes},
the parameter scale \eqref{eq:coefficientwise-parameter-scale}, and
the height \eqref{eq:coefficientwise-confidence-height} give
\[
\mathcal L_1\le2\mathcal P<U/32.
\]
Summing
\eqref{eq:localized-augmented-descent} to \((\sigma_U\wedge T)-1\)
on all four good events of Appendix~\ref{app:stopped-concentration} yields
\[
\mathcal L_{\sigma_U\wedge T}
+\frac18\sum_{t<\sigma_U\wedge T}E_t\le U/2.
\]
The energy definition \eqref{eq:coefficientwise-energy-processes}
gives \(\mathcal L_t\ge\Psi(\q_t)\), with \(\Psi\) from
\eqref{eq:auxiliary-objective}, so a crossing
\(\sigma_U\le T\), for the stopping index \(\sigma_U\) of
\eqref{eq:coefficientwise-stopping-index}, contradicts this inequality. Thus
\begin{equation}
\sigma_U>T,\qquad \sum_{t=1}^{T-1}E_t\le4U
\label{eq:closed-preconditioned-energy}
\end{equation}
on an event of probability at least \(1-\delta/2\).
\end{proof}

\subsection{De-preconditioning}
\label{app:de-preconditioning}

For the stopping index \eqref{eq:coefficientwise-stopping-index}
and variance envelope \eqref{eq:coefficientwise-local-scales},
the same predictable stopping gives
\[
\mathbb E\sum_{t<T}\mathbf1_{\{t<\sigma_U\}}\|\xi_t\|^2
\le T[C+\mathcal V_{\mathrm q}(U)].
\]
Here \(C\) is the constant of Assumption~\ref{ass:abc} and \(\xi_t\) is
defined in \eqref{eq:coefficientwise-energy-processes}.
With an additional failure probability at most \(\delta/2\), this
stopped sum divided by \(T\) is at most
\[
\frac2\delta[C+\mathcal V_{\mathrm q}(U)].
\]
On its intersection with \eqref{eq:closed-preconditioned-energy},
the accumulator recursion of Algorithm~\ref{alg:adam}, the variables in
\eqref{eq:coefficientwise-energy-processes}, and
\(\|\dir_t\|^2\le2\|h_t\|^2+2\|\xi_t\|^2\) imply
\[
\max_{t<T,i}\acc_{t,i}\le\binit+
\frac2T\sum_{s=1}^{T-1}\|\nabla f(\w_s)\|^2
+\frac4\delta[C+\mathcal V_{\mathrm q}(U)].
\]
With \(U=U_\delta\) from \eqref{eq:coefficientwise-confidence-height}
and \(\mathcal V_{\mathrm q}\) from \eqref{eq:coefficientwise-local-scales},
\eqref{eq:closed-preconditioned-energy} consequently gives
\begin{align*}
\frac1T\sum_{s=1}^{T-1}\|\nabla f(\w_s)\|^2
\le{}&\frac{4U(\epsilon+\sqrt{\binit})}{\bar\alpha\sqrt T}
+\frac{4U}{\bar\alpha\sqrt T}
\sqrt{\frac2T\sum_{s=1}^{T-1}\|\nabla f(\w_s)\|^2}\\
&+\frac{8U\sqrt{C+\mathcal V_{\mathrm q}(U)}}{\bar\alpha\sqrt{\delta T}}.
\end{align*}
The elementary inequality
\[
\frac{4U}{\bar\alpha\sqrt T}
\sqrt{\frac2T\sum_{s=1}^{T-1}\|\nabla f(\w_s)\|^2}
\le\frac1{2T}\sum_{s=1}^{T-1}\|\nabla f(\w_s)\|^2
+\frac{16U^2}{\bar\alpha^2T}
\]
gives
\eqref{eq:coefficientwise-rate}. The union bound gives total failure at
most \(\delta\), proving
Theorem~\ref{thm:coefficientwise-convergence}.

\paragraph{Confidence dependence of the admissible stepsize.}
For the decomposition \eqref{eq:stepsize-order-split}, the coefficient
in \eqref{eq:stepsize-excess-main} is
\begin{equation}
\eta_{\mathrm{gen}}
=\mathbf1_{\{p>0\}}
+\sqrt{A+[B-1]_++B\mathbf1_{\{\rho_2\ne2\}}}.
\label{eq:stepsize-excess-appendix}
\end{equation}
Here \(A,B,\rho_2\) are the constants of Assumption~\ref{ass:abc}.
It vanishes exactly when \(p=0\), \(A=0\), \(B\le1\), and \(B=0\) or
\(\rho_2=2\). In that case \(Bz^{\rho_2}-z^2\le0\) for all \(z\ge0\), so
\eqref{eq:coefficientwise-local-scales}--\eqref{eq:coefficientwise-calibration-scale}
give \(\mathcal V_{\mathrm q}(U)=0\) and, since \(p=0\),
\(\mathcal R(U)=\mathcal P\) for every \(U\ge1\). If \(\eta_{\mathrm{gen}}>0\),
then \eqref{eq:coefficientwise-confidence-height} gives
\(U_\delta=2^{20}\mathcal P^8\) for every \(\delta\).
Choose the upper endpoint of \eqref{eq:coefficientwise-stepsize}. Its
reciprocal is therefore
\begin{equation}
\bar\alpha^{-1}=
\begin{cases}
2^{100}d\,\mathcal P^{48}, & \eta_{\mathrm{gen}}=0,\\
2^{100}d\,\mathcal R(2^{20}\mathcal P^8)^{48}(1+\log(16/\delta)),
& \eta_{\mathrm{gen}}>0,
\end{cases}
\label{eq:stepsize-exact-split}
\end{equation}
where \(\mathcal P\) and \(\mathcal R\) are defined in
\eqref{eq:coefficientwise-parameter-scale} and
\eqref{eq:coefficientwise-calibration-scale}. Neither \(\mathcal P\) nor
\(\mathcal R(2^{20}\mathcal P^8)\) depends on \(T,\delta,d\). Since
\(1+\log(16/\delta)=(1+\log16)+\log(1/\delta)\), the second case of
\eqref{eq:stepsize-exact-split} is the sum of a term independent of
\(\delta\) and \(2^{100}d\,\mathcal R(2^{20}\mathcal P^8)^{48}\log(1/\delta)\).
This proves \eqref{eq:stepsize-order-split}; its logarithmic term is present
only when \(\eta_{\mathrm{gen}}>0\).

\subsection{Proof of Claim~\ref{clm:squared-standard-rate}}
\label{app:proof-standard-specialization}

\begin{proof}
When \(p=0,A=0,B=1,\rho_2=2\), \eqref{eq:stepsize-excess-appendix} gives
\(\eta_{\mathrm{gen}}=0\), so \eqref{eq:coefficientwise-confidence-height} reads
\(U_\delta=2^{20}\mathcal P^8(1+\log(16/\delta))^4\) and
\eqref{eq:coefficientwise-stepsize} carries no logarithmic factor. The definitions
\eqref{eq:coefficientwise-local-scales}--\eqref{eq:coefficientwise-calibration-scale} give
\(\mathcal V_{\mathrm q}(U)=0\) and \(\mathcal R(U)=\mathcal P\) for every
\(U\ge1\), with \(\mathcal P\) given by
\eqref{eq:coefficientwise-parameter-scale}.
The same fixed prefactor is therefore admissible for all confidence
levels. The bound \eqref{eq:coefficientwise-rate} is a decreasing
function of \(\delta\). For \(1\le r<2\), quantile integration bounds
\[
\mathbb E\!\left[
\left(\frac1T\sum_{s=1}^{T-1}\|\nabla f(\w_s)\|^2\right)^r
\right]
\]
by the integral of the \(r\)-th power of the right-hand side of
\eqref{eq:coefficientwise-rate} over \(0<\delta<1\).
The only singular endpoint integral is of the form
\[
\int_0^1\delta^{-r/2}(1+\log(16/\delta))^{4r}\,d\delta<\infty.
\]
All remaining integrals are finite moments of \(\log(1/\delta)\), and
thus
\[
\mathbb E\!\left[
\left(\frac1T\sum_{s=1}^{T-1}\|\nabla f(\w_s)\|^2\right)^r
\right]=\mathcal{O}(T^{-r/2}).
\]
For example,
\begin{align*}
\mathbb E\!\left[
\frac1T\sum_{s=1}^{T-1}\|\nabla f(\w_s)\|^2\right]\le{}&
\frac8{\bar\alpha\sqrt T}\int_0^1
U_\delta(\epsilon+\sqrt{\binit}+2\sqrt C\,\delta^{-1/2})\,d\delta\\
&+\frac{32}{\bar\alpha^2T}\int_0^1U_\delta^2\,d\delta.
\end{align*}
All parameter dependence in these finite integrals is specified by
\eqref{eq:coefficientwise-confidence-height} and
\eqref{eq:coefficientwise-parameter-scale}.
For each integer \(j\ge0\) and real \(s>0\), substitution
\(u=\log(1/\delta)\) followed by integration by parts gives
\begin{equation}
\int_0^1\delta^{s-1}\log^j(1/\delta)\,d\delta
=\int_0^\infty e^{-su}u^j\,du
=\frac{j!}{s^{j+1}}.
\label{eq:logarithmic-quantile-integral}
\end{equation}
Expand the fourth and eighth powers in
\eqref{eq:coefficientwise-confidence-height} by the binomial theorem.
Applying \eqref{eq:logarithmic-quantile-integral} with \(s=1\)
to the deterministic terms and with \(s=1/2\) to the noise term
yields \eqref{eq:coefficientwise-standard-mean}, proving
Claim~\ref{clm:squared-standard-rate}.
\end{proof}

\subsection{Proof of Theorem~\ref{thm:coefficientwise-expectation}}
\label{app:proof-expectation}

Use \(\delta_T,\bar\alpha_T,A_{\mathrm{tail}}\) from
\eqref{eq:coefficientwise-tail-calibration} and the confidence height
\eqref{eq:coefficientwise-confidence-height}.
The value \(\mathcal R(U_\delta)\), with \(\mathcal R\) from
\eqref{eq:coefficientwise-calibration-scale}, does not depend on \(\delta\): if
\(\eta_{\mathrm{gen}}>0\), then \(U_\delta=2^{20}\mathcal P^8\) by
\eqref{eq:coefficientwise-confidence-height}, and if \(\eta_{\mathrm{gen}}=0\),
then \(\mathcal R(U)=\mathcal P\) for every \(U\ge1\) (with \(\mathcal P\) from
\eqref{eq:coefficientwise-parameter-scale}), as shown after
\eqref{eq:stepsize-excess-appendix}. Since
\((1+\log(16/\delta))^{-\mathbf 1_{\{\eta_{\mathrm{gen}}>0\}}}\) is
nondecreasing in \(\delta\), the fixed prefactor \(\bar\alpha_T\) satisfies
\eqref{eq:coefficientwise-stepsize} for every
\(\delta\in[\delta_T,1)\).
Theorem~\ref{thm:coefficientwise-convergence} consequently applies to
the same run with prefactor \(\bar\alpha_T\) at every
\(\delta\in[\delta_T,1)\). These individual tail bounds can
therefore be integrated over the confidence level.

The bound in \eqref{eq:coefficientwise-rate}, with the denominator fixed
at \(\bar\alpha_T\) from \eqref{eq:coefficientwise-tail-calibration},
is decreasing in \(\delta\). By
\eqref{eq:coefficientwise-confidence-height} and
\eqref{eq:coefficientwise-local-scales},
\(U_\delta\) and \(C+\mathcal V_{\mathrm q}(U_\delta)\) grow at most polynomially in
\(1+\log(16/\delta)\). Thus integrating the bound over
\([\delta_T,1]\) gives
\[
\mathcal{O}\!\left(\frac1{\bar\alpha_T\sqrt T}
+\frac1{\bar\alpha_T^2T}\right).
\]
For the remaining quantiles \(\delta<\delta_T\), with \(\delta_T\) from
\eqref{eq:coefficientwise-tail-calibration}, Lemma~\ref{lem:gradient_distance_poly} at
\(\w_1=\q_1\) and
\(\|\w_t-\w_1\|\le2\sqrt d\,\bar\alpha_TT\le2T\) give the deterministic
bound
\begin{align*}
\frac1T\sum_{s=1}^{T-1}\|\nabla f(\w_s)\|^2
\le{}&\left[(1+\|\nabla f(\w_1)\|)^{1-p}+(1-p)(L_0+L_p)\right]^{2/(1-p)}\\
&\qquad\times(1+2T)^{2/(1-p)}.
\end{align*}
Their integral is at most this bound times \(\delta_T\). Since
\(1+2T\le3T\),
\(\delta_T(1+2T)^{2/(1-p)}\le3^{2/(1-p)}T^{2/(1-p)-A_{\mathrm{tail}}}\), which is
\(o(T^{-1/2})\) by the choice \(A_{\mathrm{tail}}>1/2+2/(1-p)\) in
\eqref{eq:coefficientwise-tail-calibration}.
Adding this contribution to the integral of
\eqref{eq:coefficientwise-rate} over \([\delta_T,1]\), with
\(U_\delta\), \(\mathcal V_{\mathrm q}\) and \(\bar\alpha_T\)
defined in \eqref{eq:coefficientwise-confidence-height},
\eqref{eq:coefficientwise-local-scales} and
\eqref{eq:coefficientwise-tail-calibration}, respectively, gives
the complete inequality \eqref{eq:coefficientwise-expectation-explicit}.
Finally, \eqref{eq:coefficientwise-tail-calibration} and
\eqref{eq:stepsize-exact-split} at \(\delta=\delta_T\) give
\(\bar\alpha_T^{-1}=2^{100}d\,\mathcal P^{48}\) when \(\eta_{\mathrm{gen}}=0\) and
\(\bar\alpha_T^{-1}=2^{100}d\,\mathcal R(2^{20}\mathcal P^8)^{48}
(1+\log16+A_{\mathrm{tail}}\log T)\) when \(\eta_{\mathrm{gen}}>0\).
The integrals in \eqref{eq:coefficientwise-expectation-explicit} are at most
their values over \((0,1]\), which are finite and independent of \(T\) and \(d\).
Hence the three integral terms of
\eqref{eq:coefficientwise-expectation-explicit} are
\(\mathcal{O}(d/\sqrt T+d^2/T)\) when \(\eta_{\mathrm{gen}}=0\) and
\(\mathcal{O}(d(1+\log T)/\sqrt T+d^2(1+\log T)^2/T)\) when
\(\eta_{\mathrm{gen}}>0\), while the rare-event term is \(o(T^{-1/2})\).
This proves Theorem~\ref{thm:coefficientwise-expectation}.

\section{Proof of Proposition~\ref{prop:delta_sharpness}}
\label{app:confidence-lower-proof}
\label{subsec:proof-prop-delta-sharpness}

\begin{proof}
We give the construction for the one-dimensional case. Let
\begin{equation}
\label{eq:confidence-lower-scales}
\chi_\star:=(\delta T)^{-1/4},
\qquad
H:=\sqrt{\frac{T}{2\delta}},
\qquad
L:=16\bar\alpha\sqrt{8}\,(\delta T)^{1/4}.
\end{equation}
Using the height \(\chi_\star\) and width \(L\) from
\eqref{eq:confidence-lower-scales}, define the continuous
piecewise-affine function
\begin{equation}
\label{eq:confidence-lower-ramp}
h(u)=
\begin{cases}
0,&u\le-2L,\\
\chi_\star(u+2L)/L,&-2L\le u\le-L,\\
\chi_\star,&-L\le u\le-\bar\alpha/8,\\
-8\chi_\star u/\bar\alpha,&-\bar\alpha/8\le u\le0,\\
0,&u\ge0,
\end{cases}
\qquad f(u)=\int_0^u h(r)\,dr.
\end{equation}
The pieces in \eqref{eq:confidence-lower-ramp} agree at their
endpoints. Since \(\delta T\ge1\), the scales in
\eqref{eq:confidence-lower-scales} satisfy
\(\chi_\star\le1\) and \(L\ge16\bar\alpha\sqrt8\).
Thus \(f\in C^1(\mathbb R)\), its derivative is globally Lipschitz
with constant at most \(8\chi_\star/\bar\alpha\le8/\bar\alpha\),
and it satisfies the smoothness assumption (Assumption~\ref{ass:smooth}) with
\(p=0\) and \(L_0=L_p=4/\bar\alpha\).
Moreover, \(h\) in \eqref{eq:confidence-lower-ramp} is nonnegative
and vanishes outside \([-2L,0]\). Using \(L\chi_\star\) from
\eqref{eq:confidence-lower-scales} gives
\[
f^*=-\int_{-2L}^{0}h(u)\,du> -\infty,\qquad
f(0)-f^*\le2L\chi_\star=32\sqrt8\,\bar\alpha.
\]

Set \(\w_1=0\), \(\mom_0=0\), and \(\acc_0=\binit>0\) in Algorithm~\ref{alg:adam}. Let the oracle be
\[
\dir_t=h(\w_t)+\omega_t,
\]
where the independent noise variables, with spike height \(H\) from
\eqref{eq:confidence-lower-scales}, satisfy
\begin{equation}
\label{eq:confidence-lower-noise}
\omega_t
=
\begin{cases}
H, & \text{with probability }\delta/T,\\
-H, & \text{with probability }\delta/T,\\
0, & \text{with probability }1-2\delta/T.
\end{cases}
\end{equation}
The noise law \eqref{eq:confidence-lower-noise} and the spike height
in \eqref{eq:confidence-lower-scales} give
\(\mathbb E[\omega_t]=0\) and \(\mathbb E[\omega_t^2]=1\), so
\[
\mathbb E[\dir_t\mid\mathscr F_{t-1}]=h(\w_t)=\nabla f(\w_t),
\qquad
\mathbb E[\dir_t^2\mid\mathscr F_{t-1}]
=h(\w_t)^2+1.
\]
Thus Assumption~\ref{ass:abc} holds with \(A=0\), \(B=1\),
\(\rho_1=0\), \(\rho_2=2\), and \(C=1\), independently of \(T,\delta\).

For the noise in \eqref{eq:confidence-lower-noise}, let
\begin{equation}
\label{eq:confidence-lower-spike-event}
\mathcal S:=\bigcup_{k=1}^{\lfloor T/2\rfloor}
\left\{\omega_k=H,\quad
\omega_t=0\ \text{for }1\le t<T,\ t\ne k\right\}.
\end{equation}
Thus exactly one positive spike occurs in the first
\(\lfloor T/2\rfloor\) iterations, and all other noise values before
\(T\) vanish. Since \(\delta\le1/8\),
\[
\mathbb P(\mathcal S)
=
\left\lfloor\frac{T}{2}\right\rfloor
\frac{\delta}{T}
\left(1-\frac{2\delta}{T}\right)^{T-2}
\ge
\frac\delta4.
\]
Indeed, \(\lfloor T/2\rfloor/T\ge1/3\), and Bernoulli's inequality
bounds the survival factor below by \(1-2\delta\ge3/4\).
Figure~\ref{fig:confidence-dynamics} illustrates the entry into the
constant-gradient region and the ensuing slow motion.
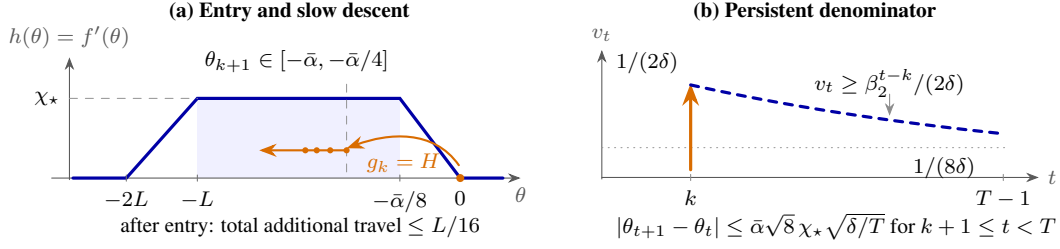
\begin{figure}[H]
\centering
\begin{minipage}[t]{0.49\linewidth}
\centering
\textbf{\footnotesize (a) Entry and slow descent}\par\smallskip
\begin{tikzpicture}[x=0.94cm,y=0.68cm,font=\fontsize{8}{9}\selectfont,
  line cap=round,>=Stealth]
\path[use as bounding box] (-0.10,-1.24) rectangle (7.00,2.85);
\fill[blue!6] (2.15,0) rectangle (5.00,1.55);
\draw[->,black!65] (0.10,0) -- (6.70,0) node[below] {$\theta$};
\draw[->,black!65] (0.35,-0.03) -- (0.35,2.35) node[above] {$h(\theta)=f'(\theta)$};
\draw[very thick,blue!65!black] (0.40,0) -- (1.15,0) -- (2.15,1.55)
  -- (5.00,1.55) -- (5.85,0) -- (6.45,0);
\draw[black!40,dashed] (0.35,1.55) -- (2.15,1.55);
\node[left] at (0.35,1.55) {$\chi_\star$};
\foreach \x/\lab in {1.15/{-2L},2.15/{-L},5.00/{-\bar\alpha/8},5.85/{0}}
  {\draw[black!55] (\x,0.06) -- (\x,-0.06);
   \node[below] at (\x,-0.07) {$\lab$};}
\node[align=center] at (3.55,2.27) {$\theta_{k+1}\in[-\bar\alpha,-\bar\alpha/4]$};
\draw[black!40,dashed] (4.25,0.12) -- (4.25,1.88);
\fill[orange!85!black] (5.85,0) circle (1.5pt);
\draw[->,thick,orange!85!black] (5.83,0.25)
  to[out=125,in=20] (4.25,0.58);
\node[orange!85!black] at (5.05,0.30) {$g_k=H$};
\draw[->,thick,orange!85!black] (4.25,0.54) -- (3.02,0.54);
\foreach \x in {4.25,4.02,3.83,3.67}
  \fill[orange!85!black] (\x,0.54) circle (1.15pt);
\node at (3.60,-0.95)
  {after entry: total additional travel $\le L/16$};
\end{tikzpicture}
\end{minipage}\hfill
\begin{minipage}[t]{0.49\linewidth}
\centering
\textbf{\footnotesize (b) Persistent denominator}\par\smallskip
\begin{tikzpicture}[x=0.94cm,y=0.68cm,font=\fontsize{8}{9}\selectfont,
  line cap=round,>=Stealth]
\path[use as bounding box] (-0.10,-1.24) rectangle (7.00,2.85);
\draw[->,black!65] (0.45,0) -- (6.60,0) node[right] {$t$};
\draw[->,black!65] (0.45,-0.03) -- (0.45,2.48) node[above] {$v_t$};
\draw[black!40,dotted] (0.45,0.58) -- (6.15,0.58);
\node[below right] at (4.70,0.57) {$1/(8\delta)$};
\draw[->,very thick,orange!85!black] (1.70,0.15) -- (1.70,1.83);
\node[above left] at (1.66,1.84) {$1/(2\delta)$};
\draw[very thick,blue!65!black,dashed,domain=1.70:6.10,samples=50]
  plot (\x,{1.8*exp(-0.75*(\x-1.70)/4.40)});
\node[align=center,fill=white,inner sep=2pt] at (4.50,1.84)
  {$v_t\ge\beta_2^{t-k}/(2\delta)$};
\draw[black!45,->] (4.50,1.62) -- (4.50,1.20);
\draw[black!45] (1.70,0.06) -- (1.70,-0.06);
\node[below] at (1.70,-0.07) {$k$};
\draw[black!45] (6.10,0.06) -- (6.10,-0.06);
\node[below] at (6.10,-0.07) {$T-1$};
\node at (3.75,-0.95)
  {$|\theta_{t+1}-\theta_t|\le
    \bar\alpha\sqrt8\,\chi_\star\sqrt{\delta/T}$ for $k+1\le t<T$};
\end{tikzpicture}
\end{minipage}
\caption{Confidence lower-bound dynamics on the single-spike event
\(\mathcal S\) from \eqref{eq:confidence-lower-spike-event}, with
\(\beta_1=0\). The scales \(\chi_\star,H,L\) are defined in
\eqref{eq:confidence-lower-scales}, and (a) shows the derivative
in \eqref{eq:confidence-lower-ramp}. At the spike time \(k\), the positive
oracle output sends the iterate left into the shaded plateau.
Subsequent steps remain leftward but cannot leave it before the horizon.
Panel (b) shows a lower envelope of the accumulator, not an exact
accumulator trajectory: the spike contribution decays only by
\(\beta_2^{t-k}\) and stays above \(1/(8\delta)\).
Both panels are schematic and not to scale.}
\label{fig:confidence-dynamics}
\end{figure}

On the event \(\mathcal S\) in
\eqref{eq:confidence-lower-spike-event}, let \(k\le T/2\) denote the
spike time. Before time \(k\), the iterate remains at \(0\), because
\(h(0)=0\) in \eqref{eq:confidence-lower-ramp} and the noise is zero.
At time \(k\), \(\dir_k=H\), with \(H\) from
\eqref{eq:confidence-lower-scales}, and
\[
\acc_k
\le
\binit+\frac{H^2}{T}
=
\binit+\frac{1}{2\delta},
\qquad
\acc_k
\ge
\frac{H^2}{T}
=
\frac{1}{2\delta}.
\]
The confidence range on \(\delta\) in Proposition~\ref{prop:delta_sharpness} ensures
\[
\sqrt{1+2\delta \binit}+\epsilon\sqrt{2\delta}
\le\frac{\sqrt5+1}{2}<4.
\]
Substituting \(H\) from \eqref{eq:confidence-lower-scales}, the
resulting Adam step satisfies
\[
\w_{k+1}
=
-\frac{\bar\alpha}{\sqrt T}
\frac{H}{\sqrt{\acc_k}+\epsilon}
=
-\frac{\bar\alpha}{\sqrt{2\delta\acc_k}+\epsilon\sqrt{2\delta}}.
\]
Because \(\acc_k\ge1/(2\delta)\), we have
\(|\w_{k+1}|\le\bar\alpha\). The upper accumulator bound gives
\[
\sqrt{2\delta\acc_k}+\epsilon\sqrt{2\delta}
\le\sqrt{1+2\delta\binit}+\epsilon\sqrt{2\delta}\le4,
\]
and hence \(|\w_{k+1}|\ge\bar\alpha/4\). Therefore
\[
\w_{k+1}\in[-\bar\alpha,-\bar\alpha/4]\subset[-L,-\bar\alpha/8],
\]
so the iterate enters the constant-gradient part of
\eqref{eq:confidence-lower-ramp}.

We claim that \(\w_t\in[-L,-\bar\alpha/8]\) for all
\(t=k+1,\ldots,T-1\), where \(L\) is defined in
\eqref{eq:confidence-lower-scales}. We prove this by induction.
On \(\mathcal S\) in \eqref{eq:confidence-lower-spike-event}, there
are no further noise spikes. As long as the iterates remain in the
constant part of \eqref{eq:confidence-lower-ramp},
\(\dir_t=h(\w_t)=\chi_\star\), the updates move monotonically to the
left, and the accumulator remains large:
\[
\acc_t
\ge
\beta_2^{t-k}\acc_k
\ge
\frac{1}{8\delta},
\qquad
k\le t\le T-1.
\]
Here we used \((1-1/T)^T\ge1/4\) for \(T\ge2\).
Therefore each later displacement in the ramp
\eqref{eq:confidence-lower-ramp}, whose height \(\chi_\star\) is
given in \eqref{eq:confidence-lower-scales}, is bounded by
\[
|\w_{t+1}-\w_t|
=
\frac{\bar\alpha}{\sqrt T}
\frac{\chi_\star}{\sqrt{\acc_t}+\epsilon}
\le
\bar\alpha\sqrt{8}\,
\frac{\chi_\star\sqrt\delta}{\sqrt T}.
\]
Summing over at most \(T\) steps and using the definitions of
\(\chi_\star,L\) in \eqref{eq:confidence-lower-scales} gives
\[
\sum_{t=k+1}^{T-1}|\w_{t+1}-\w_t|
\le
\bar\alpha\sqrt{8}\,\chi_\star\sqrt{\delta T}
=
\bar\alpha\sqrt{8}\,(\delta T)^{1/4}
\le
\frac{L}{16}.
\]
With \(L\) from \eqref{eq:confidence-lower-scales}, for every later
time covered by the induction,
\[
\w_t\ge \w_{k+1}-\frac{L}{16}\ge-\bar\alpha-\frac{L}{16}\ge -L,
\qquad
\w_t\le \w_{k+1}\le-\frac{\bar\alpha}{4}\le-\frac{\bar\alpha}{8}.
\]
This closes the induction and proves \(\w_t\in[-L,-\bar\alpha/8]\).
The definition \eqref{eq:confidence-lower-ramp} then gives
\[
|\nabla f(\w_t)|=\chi_\star,
\qquad
t=k+1,\ldots,T-1.
\]
Because \(k\le T/2\), substituting \(\chi_\star\) from
\eqref{eq:confidence-lower-scales} yields on the event
\eqref{eq:confidence-lower-spike-event}
\[
\frac1T\sum_{t=1}^{T-1}|\nabla f(\w_t)|^2
\ge
\frac{T/2-1}{T}\chi_\star^2
\ge
\frac1{4\sqrt{\delta T}},
\]
where \(T\ge4\) was used in the last step. Together with
\(\mathbb P(\mathcal S)\ge\delta/4\), this proves the probability
and stationarity assertions.

It remains to compare the construction on a common parameter class with
the calibrated upper bound. If
\(\delta T\ge(8/\bar\alpha)^4\), then the height in
\eqref{eq:confidence-lower-scales} satisfies
\(\chi_\star\le\bar\alpha/8\), so the Lipschitz constant
\(8\chi_\star/\bar\alpha\) of \(h\) from \eqref{eq:confidence-lower-ramp} is at most one.
Thus all these instances satisfy the common constants
\[
p=0,\quad L_0=L_p=1,\quad A=0,\quad B=1,\quad
\rho_1=0,\quad\rho_2=2,\quad C=1.
\]
Their initial gradient is zero because \(h(0)=0\) in
\eqref{eq:confidence-lower-ramp}, and their initial gap is at most
\(32\sqrt8\,\bar\alpha\le91\). In this specialization and with
\(d=1,\beta_1=0\), the baseline parameter sum in
\eqref{eq:coefficientwise-stepsize} is therefore bounded by
\[
100+\epsilon+\binit^{-1}+\binit^{-3/2}.
\]
Every fixed positive prefactor at most
\(2^{-100}(100+\epsilon+\binit^{-1}+\binit^{-3/2})^{-48}\)
is consequently admitted by Theorem~\ref{thm:coefficientwise-convergence} for the entire constructed
class, independently of \(T\) and \(\delta\).

Now use \(\delta=8\eta\). For each sufficiently small \(\eta>0\),
choose any integer \(T\ge10\) large enough that
\(8\eta T\ge(8/\bar\alpha)^4\). The already proved estimates on the event
\(\mathcal S\) from \eqref{eq:confidence-lower-spike-event} give
\[
\mathbb P\!\left(
\frac1T\sum_{t=1}^{T-1}|f'(\theta_t)|^2
\ge\frac1{8\sqrt2\sqrt{\eta T}}
\right)\ge2\eta.
\]
For every fixed \(k\ge0\),
\(\eta^{-1/2}/\log^k(e/\eta)\to\infty\) as \(\eta\downarrow0\).
Hence the lower threshold eventually exceeds any uniform bound of
order \(\log^k(e/\eta)/\sqrt T\), while the probability of exceeding
that bound is strictly larger than \(\eta\). Any additional term of
order \(\operatorname{polylog}(1/\eta)/T\) can also be made smaller
by increasing \(T\). This establishes the asserted confidence
obstruction for the same stationarity statistic and an admissible fixed
prefactor.
\end{proof}

\section{Proof of Proposition~\ref{prop:expected_stationarity_lower_bound}}
\label{app:expectation-lower-proof}
In the one-dimensional constructions below, \(\theta_t,m_t,v_t\)
denote the scalar iterate, momentum and accumulator of Algorithm~\ref{alg:adam}.
\begin{proof}
Let \(\mathcal C_p\) consist of the deterministic quadratic instance
\begin{equation}
\label{eq:expected-lower-quadratic-instance}
f_{\mathrm q}(x)=x^2/2,\qquad
\mathsf G_{\mathrm q}(x)=x,\qquad\theta_1=1,
\end{equation}
and the following countable family, indexed by integers \(n\ge1\):
\begin{equation}
\label{eq:expected-lower-staircase-gradient}
h_n(x)=
\begin{cases}
0,&x\le-2,\\
x+2,&-2\le x\le0,\\
2^{j+1},&2j\le x\le2j+1,\quad 0\le j<n,\\
2^{j+1}(x-2j),&2j+1\le x\le2j+2,\quad0\le j<n,\\
2^{n+1},&x\ge2n,
\end{cases}
\end{equation}
\begin{equation}
\label{eq:expected-lower-staircase-objective}
f_n(x)=\int_{-2}^{x}h_n(y)\,dy,\qquad\theta_1=0.
\end{equation}
The pieces in \eqref{eq:expected-lower-staircase-gradient} agree at
their common endpoints. In particular, \(h_n\) is
continuous and nondecreasing, \(f_n\in C^1(\mathbb R)\),
\(f_n^*=0\), \(f_n(0)=2\), and \(f_n'(0)=2\).

We first verify the exact smoothness condition for
\eqref{eq:expected-lower-staircase-gradient}--\eqref{eq:expected-lower-staircase-objective}.
Suppose \(u<v\) and
\(v-u\le1\). The interval intersects the interior of at most one ramp
\((2j+1,2j+2)\). If it intersects that ramp, then \(u>2j\), so
\(h_n(u)\ge2^{j+1}\). The slope on that ramp is \(2^{j+1}\), which is at
most \(h_n(u)^p\) because \(p\ge1\). If the interval instead intersects
\((-2,0)\), the slope there is one; all remaining slopes on the interval
are zero. Thus
\[
0\le h_n(v)-h_n(u)
\le \bigl(1+h_n(u)^p\bigr)(v-u).
\]
Monotonicity gives the same bound with \(h_n(v)^p\) on the right, proving
Assumption~\ref{ass:smooth} in both orders. The quadratic satisfies the
same assumption since \(|f_{\mathrm q}'(u)-f_{\mathrm q}'(v)|=|u-v|\)
by \eqref{eq:expected-lower-quadratic-instance}.

For the staircase gradient \(h_n\) in
\eqref{eq:expected-lower-staircase-gradient}, write \(h=h_n(x)\)
and define the oracle by
\begin{equation}
\label{eq:expected-lower-staircase-oracle}
\mathsf G_n(x)=
\begin{cases}
h,&0\le h<2,\\
-h\quad\text{with probability }1-h^{-1},&h\ge2,\\
2h^2-h\quad\text{with probability }h^{-1},&h\ge2.
\end{cases}
\end{equation}
At different calls use independent uniform random variables to select the
branch. For \(h=h_n(x)\ge2\), the law in
\eqref{eq:expected-lower-staircase-oracle} gives
\[
\mathbb E\mathsf G_n(x)=h,\qquad
\mathbb E\mathsf G_n(x)^2=4h^3-3h^2\le4h^3.
\]
For \(0\le h<2\), the deterministic branch of
\eqref{eq:expected-lower-staircase-oracle} has second moment
\(h^2\le4h^3+1\): when \(h\le1\), use \(h^2\le1\), and when
\(h\ge1\), use \(h^2\le h^3\). The same two cases, with \(|x|\), show
\(x^2\le4|x|^3+1\) for the quadratic in
\eqref{eq:expected-lower-quadratic-instance}. Hence all members of
\(\mathcal C_p\), given by
\eqref{eq:expected-lower-quadratic-instance},
\eqref{eq:expected-lower-staircase-objective}, and
\eqref{eq:expected-lower-staircase-oracle}, satisfy the stated
unbiasedness and ABC conditions.

For the quadratic in \eqref{eq:expected-lower-quadratic-instance},
as long as the preceding iterates lie in \([1/2,1]\),
the momentum recursion of Algorithm~\ref{alg:adam} gives \(0\le m_t\le1\). Therefore
\[
0\le\theta_t-\theta_{t+1}
=\frac{\alpha m_t}{\sqrt{v_t}+\epsilon}
\le\frac{\alpha}{\epsilon}.
\]
By induction, the first
\(\min\{T-1,\lfloor\epsilon/(2\alpha)\rfloor+1\}\) iterates lie in
\([1/2,1]\). For \(T\ge2\), with \(\alpha=\bar\alpha_T/\sqrt T\) as in
Proposition~\ref{prop:expected_stationarity_lower_bound}, this proves
\begin{equation}
\label{eq:expected_lower_quadratic}
\frac1T\sum_{t=1}^{T-1}|f_{\mathrm q}'(\theta_t)|^2
\ge\min\left\{\frac18,
\frac{\epsilon}{8\bar\alpha_T\sqrt T}\right\}.
\end{equation}
In particular, when
\begin{equation}
\label{eq:expected_lower_small_step}
\bar\alpha_T\le
\frac{2^{16}(1+\sqrt{\binit}+\epsilon)^2}{(1-\beta_1)^2}\,T^{-1/6},
\end{equation}
the bound in \eqref{eq:expected_stationarity_lower_bound} follows from
\eqref{eq:expected_lower_quadratic}. Indeed,
\(\epsilon(1-\beta_1)^2/[2^{19}(1+\sqrt{\binit}+\epsilon)^2]<1/8\).

For larger stepsizes, namely those violating
\eqref{eq:expected_lower_small_step}, consider the deterministic path obtained by always
selecting the negative branch \(-h_n(\theta_t)\) of
\eqref{eq:expected-lower-staircase-oracle} until the path reaches
\(2n\). Call the intervals
\begin{equation}
\label{eq:expected-lower-blocks}
[2j,2j+2),\qquad 0\le j<n,
\end{equation}
the staircase blocks. Extend the
same deterministic recursion for as many steps as needed to define its
hitting times
\begin{equation}
\label{eq:expected-lower-hitting-times}
\widehat\tau_j:=\inf\{t\ge1:\theta_t\ge2j\},\qquad 0\le j\le n.
\end{equation}
The infimum of the empty set is interpreted as infinity. We next prove that these times are finite and that
\begin{equation}
\label{eq:expected_lower_residence}
\widehat\tau_{j+1}-\widehat\tau_j
\le
\frac{2^{16}(1+\sqrt{\binit}+\epsilon)^2}{(1-\beta_1)^2}
\left(
\frac{1}{\min\{1,\bar\alpha_T\}^2}
+\frac{\sqrt T}{2^{j+1}\min\{1,\bar\alpha_T\}}
\right),\quad 0\le j<n.
\end{equation}
Along the negative-branch path of
\eqref{eq:expected-lower-staircase-oracle}, \(m_t<0\),
\(|m_t|\ge(1-\beta_1)h_n(\theta_t)\), and \(\theta_{t+1}>\theta_t\).
The bound for the hitting times in
\eqref{eq:expected-lower-hitting-times} is proved inductively.
A block in \eqref{eq:expected-lower-blocks} that is skipped has
zero residence time, so suppose \(\widehat\tau_j<\widehat\tau_{j+1}\).
In each preceding occupied block \(k\), the definition
\eqref{eq:expected-lower-staircase-gradient} gives
\(h_n(\theta_s)<2^{k+2}\). The induction hypothesis, namely
\eqref{eq:expected_lower_residence} for the earlier blocks, and geometric sums give
\begin{align}
\frac{1}{4^{j+1}}\sum_{s<\widehat\tau_j}h_n(\theta_s)^2
&\le
\frac{2^{18}(1+\sqrt{\binit}+\epsilon)^2}{(1-\beta_1)^2}
\left(
\frac{1}{3\min\{1,\bar\alpha_T\}^2}
+\frac{\sqrt T}{2^{j+1}\min\{1,\bar\alpha_T\}}
\right).
\label{eq:expected_lower_previous_blocks}
\end{align}
For \(j=0\), the sum on the left is empty, which starts the induction.

Let
\begin{equation}
\label{eq:expected-lower-block-update-count}
N_j:=\left\lfloor
\frac{2^{16}(1+\sqrt{\binit}+\epsilon)^2}{(1-\beta_1)^2}
\left(
\frac{1}{\min\{1,\bar\alpha_T\}^2}
+\frac{\sqrt T}{2^{j+1}\min\{1,\bar\alpha_T\}}
\right)\right\rfloor,
\end{equation}
the integer part of the right side of
\eqref{eq:expected_lower_residence}. Suppose the path has not left block
\(j\) after \(N_j\) further updates. For all those updates,
\(2^{j+1}\le h_n(\theta_t)<2^{j+2}\), and the accumulator recursion of
Algorithm~\ref{alg:adam} and
\eqref{eq:expected_lower_previous_blocks} imply
\begin{equation*}
v_t\le\binit+
\frac{2^{19}(1+\sqrt{\binit}+\epsilon)^2\,4^{j+1}}
     {(1-\beta_1)^2T}
\left(
\frac{1}{\min\{1,\bar\alpha_T\}^2}
+\frac{\sqrt T}{2^{j+1}\min\{1,\bar\alpha_T\}}
\right).
\end{equation*}
Here we used \(v_t\le\binit+T^{-1}\sum_{s\le t}h_n(\theta_s)^2\).

To verify that the \(N_j\) updates in
\eqref{eq:expected-lower-block-update-count} traverse a block of
\eqref{eq:expected-lower-blocks}, consider their total displacement. Since the
quantity whose integer part defines \(N_j\) is at least \(2^{16}\), its
integer part is at least half that quantity. Also
\(\sqrt{\binit}+\epsilon\le1+\sqrt{\binit}+\epsilon\). Thus the total
displacement is at least
\begin{equation}
\label{eq:expected_lower_displacement}
\frac{
\dfrac{2^{15}(1+\sqrt{\binit}+\epsilon)^2}{1-\beta_1}
\left(\dfrac{1}{\min\{1,\bar\alpha_T\}}
      +\dfrac{\sqrt T}{2^{j+1}}\right)
}{
\dfrac{(1+\sqrt{\binit}+\epsilon)\sqrt T}{2^{j+1}}
+\dfrac{2^8\sqrt8(1+\sqrt{\binit}+\epsilon)}{1-\beta_1}
\sqrt{\dfrac{1}{\min\{1,\bar\alpha_T\}^2}
      +\dfrac{\sqrt T}{2^{j+1}\min\{1,\bar\alpha_T\}}}
}.
\end{equation}
The square root in the denominator satisfies
\[
\sqrt{\frac{1}{\min\{1,\bar\alpha_T\}^2}
      +\frac{\sqrt T}{2^{j+1}\min\{1,\bar\alpha_T\}}}
\le
\frac{1}{\min\{1,\bar\alpha_T\}}+\frac{\sqrt T}{2^{j+1}}.
\]
Indeed, the square of the right side exceeds the radicand by a nonnegative
quantity. The other term in the denominator is bounded using
\(\sqrt T/2^{j+1}\le1/\min\{1,\bar\alpha_T\}+\sqrt T/2^{j+1}\).
Consequently \eqref{eq:expected_lower_displacement} is at least
\[
\frac{2^{15}(1+\sqrt{\binit}+\epsilon)}
     {(1-\beta_1)+2^8\sqrt8}
\ge\frac{2^{15}}{1+2^{10}}>2.
\]
This contradicts residence in a block of width two (see
\eqref{eq:expected-lower-blocks}) and proves
\eqref{eq:expected_lower_residence}, including finiteness of the next
hitting time. This induction also allows arbitrary overshoots and skipped
blocks.

Summing \eqref{eq:expected_lower_residence} over the blocks in
\eqref{eq:expected-lower-blocks} and using
\(\sum_{j\ge0}2^{-j-1}=1\) gives, for the hitting time defined in
\eqref{eq:expected-lower-hitting-times},
\begin{equation}
\label{eq:expected_lower_hitting}
\widehat\tau_n-1\le
\frac{2^{16}(1+\sqrt{\binit}+\epsilon)^2}{(1-\beta_1)^2}
\left(
\frac{n}{\min\{1,\bar\alpha_T\}^2}
+\frac{\sqrt T}{\min\{1,\bar\alpha_T\}}
\right).
\end{equation}
The adverse path and its hitting times in
\eqref{eq:expected-lower-hitting-times} are deterministic. Let
\begin{equation}
\label{eq:expected-lower-adverse-event}
E_n:=\left\{\dir_t=-h_n(\theta_t)\ \text{for every }1\le t<\widehat\tau_n\right\},
\end{equation}
the event that the oracle in \eqref{eq:expected-lower-staircase-oracle}
selects its negative branch up to that hitting time. Since
\(\log(1-x)\ge-2x\) for \(0\le x\le1/2\), the residence bound
\eqref{eq:expected_lower_residence} gives
\begin{align}
\log\mathbb P(E_n)
&\ge-2\sum_{j=0}^{n-1}
\frac{\widehat\tau_{j+1}-\widehat\tau_j}{2^{j+1}}\notag\\
&\ge-
\frac{2^{18}(1+\sqrt{\binit}+\epsilon)^2}{(1-\beta_1)^2}
\left(
\frac{1}{\min\{1,\bar\alpha_T\}^2}
+\frac{\sqrt T}{\min\{1,\bar\alpha_T\}}
\right).
\label{eq:expected_lower_probability}
\end{align}

Choose the objective in \eqref{eq:expected-lower-staircase-objective}
with
\begin{equation}
\label{eq:expected-lower-staircase-height-index}
n=\left\lfloor
\frac{(1-\beta_1)^2T\min\{1,\bar\alpha_T\}^2}
     {2^{18}(1+\sqrt{\binit}+\epsilon)^2}
\right\rfloor.
\end{equation}
When \eqref{eq:expected_lower_small_step} fails and \(T\) is sufficiently
large, the integer in \eqref{eq:expected-lower-staircase-height-index}
is positive,
\begin{equation}
\label{eq:expected-lower-large-prefactor}
\min\{1,\bar\alpha_T\}
\ge\frac{2^{16}(1+\sqrt{\binit}+\epsilon)^2}
         {(1-\beta_1)^2}T^{-1/6},
\end{equation}
and \eqref{eq:expected_lower_hitting} yields
\(\widehat\tau_n\le T/2+1\le T-1\).
On \(E_n\) from \eqref{eq:expected-lower-adverse-event},
\eqref{eq:expected-lower-staircase-gradient} gives true gradient
\(2^{n+1}\) at the deterministic index \(\widehat\tau_n\).
Therefore
\[
\frac1T\sum_{t=1}^{T-1}\mathbb E|f_n'(\theta_t)|^2
\ge\frac{4^{n+1}}{T}\mathbb P(E_n).
\]
Figure~\ref{fig:expectation-dynamics} depicts the initial slow descent
on the quadratic and the conditional climb to this arrival gradient.
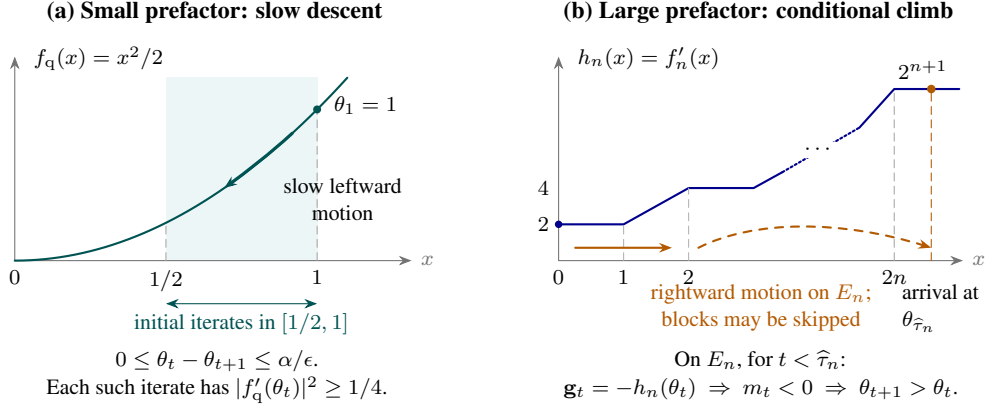
\begin{figure}[H]
\centering
\begin{minipage}[t]{0.485\linewidth}
\centering
{\small\bfseries (a) Small prefactor: slow descent}\par\smallskip
\begin{tikzpicture}[
  x=1cm,y=1cm,
  font=\fontsize{8.2}{10}\selectfont,
  >={Stealth[length=1.7mm]},
  line cap=round,line join=round]
  \path[use as bounding box] (-0.55,-1.03) rectangle (5.85,3.0);
  \fill[teal!7] (2,0) rectangle (4,2.43);
  \draw[->,black!55] (0,0) -- (5.25,0) node[right] {$x$};
  \draw[->,black!55] (0,0) -- (0,2.7);
  \node[anchor=west] at (0.13,2.68) {$f_{\mathrm q}(x)=x^2/2$};
  \draw[thick,teal!65!black,domain=0:1.1,samples=45]
    plot ({4*\x},{2*\x*\x});
  \draw[densely dashed,black!35] (2,0) -- (2,0.5);
  \draw[densely dashed,black!35] (4,0) -- (4,2);
  \node[below] at (0,0) {$0$};
  \node[below] at (2,0) {$1/2$};
  \node[below] at (4,0) {$1$};
  \fill[teal!65!black] (4,2) circle (1.5pt);
  \node[anchor=west] at (4.1,2.06) {$\theta_1=1$};
  \draw[->,very thick,teal!65!black]
    (3.67,1.684) .. controls (3.37,1.409) and (3.06,1.173) .. (2.77,0.959);
  \node[align=center,anchor=west] at (3.43,0.81)
    {slow leftward\\motion};
  \draw[<->,teal!65!black] (2,-0.57) -- (4,-0.57);
  \node[below,teal!65!black] at (3,-0.57) {initial iterates in $[1/2,1]$};
\end{tikzpicture}\par
{\fontsize{8.2}{10.5}\selectfont
\(0\le\theta_t-\theta_{t+1}\le\alpha/\epsilon\).\par
Each such iterate has \(|f'_{\mathrm q}(\theta_t)|^2\ge1/4\).\par}
\end{minipage}\hfill
\begin{minipage}[t]{0.485\linewidth}
\centering
{\small\bfseries (b) Large prefactor: conditional climb}\par\smallskip
\begin{tikzpicture}[
  x=1cm,y=1cm,
  font=\fontsize{8.2}{10}\selectfont,
  >={Stealth[length=1.7mm]},
  line cap=round,line join=round]
  \path[use as bounding box] (-0.55,-1.03) rectangle (5.85,3.0);
  \draw[->,black!55] (0,0) -- (5.35,0) node[right] {$x$};
  \draw[->,black!55] (0,0) -- (0,2.7);
  \node[anchor=west] at (0.13,2.68) {$h_n(x)=f'_n(x)$};
  \draw[thick,blue!55!black]
    (0,0.48) -- (0.86,0.48) -- (1.72,0.96) -- (2.58,0.96)
    -- (2.96,1.17);
  \draw[densely dotted,thick,blue!55!black] (2.96,1.17) -- (3.98,1.76);
  \node[fill=white,inner sep=1pt] at (3.45,1.47) {$\cdots$};
  \draw[thick,blue!55!black] (3.98,1.76) -- (4.44,2.27) -- (5.3,2.27);
  \node[left] at (0,0.48) {$2$};
  \node[left] at (0,0.96) {$4$};
  \node[above] at (4.82,2.28) {$2^{n+1}$};
  \draw[densely dashed,black!30] (0.86,0) -- (0.86,0.48);
  \draw[densely dashed,black!30] (1.72,0) -- (1.72,0.96);
  \draw[densely dashed,black!35] (4.44,0) -- (4.44,2.27);
  \node[below] at (0,0) {$0$};
  \node[below] at (0.86,0) {$1$};
  \node[below] at (1.72,0) {$2$};
  \node[below] at (4.44,0) {$2n$};
  \fill[blue!55!black] (0,0.48) circle (1.4pt);
  \fill[orange!70!black] (4.93,2.27) circle (1.6pt);
  \draw[densely dashed,orange!70!black] (4.93,0) -- (4.93,2.27);
  \draw[->,thick,orange!70!black] (0.22,0.17) -- (1.5,0.17);
  \draw[->,thick,dashed,orange!70!black]
    (1.85,0.17) .. controls (2.58,0.54) and (3.74,0.54) .. (4.93,0.17);
  \node[align=center,orange!70!black] at (2.68,-0.62)
    {rightward motion on $E_n$;\\blocks may be skipped};
  \node[anchor=west,align=left] at (4.42,-0.63)
    {arrival at\\$\theta_{\widehat\tau_n}$};
\end{tikzpicture}\par
{\fontsize{8.2}{10.5}\selectfont
On \(E_n\), for \(t<\widehat\tau_n\):\par
\(\dir_t=-h_n(\theta_t)\ \Rightarrow\ m_t<0\ \Rightarrow\ \theta_{t+1}>\theta_t\).\par}
\end{minipage}
\par\medskip
\begin{minipage}{0.97\linewidth}
\centering\fontsize{8.2}{10.5}\selectfont
For \(h=h_n(x)\ge2\), the oracle selects \(-h\) with probability
\(1-h^{-1}\), and \(2h^2-h\) with probability \(h^{-1}\), so its mean is \(h\).
\end{minipage}
\caption{Dynamics of the two instances in
Proposition~\ref{prop:expected_stationarity_lower_bound}.
The quadratic \eqref{eq:expected-lower-quadratic-instance} retains an
initial segment of iterates in \([1/2,1]\), as quantified by
\eqref{eq:expected_lower_quadratic}.
The staircase is the derivative \(h_n\) in
\eqref{eq:expected-lower-staircase-gradient}; its plateaus are joined
continuously by linear ramps. The rightward arrows are conditional on
the negative-oracle event \(E_n\) in
\eqref{eq:expected-lower-adverse-event}, with the two branches given by
\eqref{eq:expected-lower-staircase-oracle}.
They end at the first arrival \(\widehat\tau_n\) from
\eqref{eq:expected-lower-hitting-times}, where
\(h_n(\theta_{\widehat\tau_n})=2^{n+1}\).
Choosing \(n\) as in \eqref{eq:expected-lower-staircase-height-index}
places this index before \(T\); its single squared-gradient contribution
gives the lower bound \(4^{n+1}\mathbb P(E_n)/T\).
Arrows indicate directions schematically; intermediate staircase blocks
are omitted.}
\label{fig:expectation-dynamics}
\end{figure}

Using the choice \eqref{eq:expected-lower-staircase-height-index},
which gives \(n+1\ge
(1-\beta_1)^2T\min\{1,\bar\alpha_T\}^2/
[2^{18}(1+\sqrt{\binit}+\epsilon)^2]\), together with the event-probability
bound \eqref{eq:expected_lower_probability}, the logarithm of the
right side is bounded below by
\begin{align*}
&\frac{(1-\beta_1)^2\log4}
      {2^{18}(1+\sqrt{\binit}+\epsilon)^2}
  T\min\{1,\bar\alpha_T\}^2\\
&\qquad-
\frac{2^{18}(1+\sqrt{\binit}+\epsilon)^2}{(1-\beta_1)^2}
\left(
\frac{\sqrt T}{\min\{1,\bar\alpha_T\}}
+\frac{1}{\min\{1,\bar\alpha_T\}^2}
\right)-\log T.
\end{align*}
This expression is increasing in \(\min\{1,\bar\alpha_T\}>0\).
Substituting \eqref{eq:expected-lower-large-prefactor} gives
\[
\left(
\frac{2^{14}(1+\sqrt{\binit}+\epsilon)^2\log4}{(1-\beta_1)^2}-4
\right)T^{2/3}
-\frac{(1-\beta_1)^2}{2^{14}(1+\sqrt{\binit}+\epsilon)^2}T^{1/3}
-\log T.
\]
The coefficient of \(T^{2/3}\) is positive, so this is nonnegative for all
sufficiently large \(T\), uniformly over stepsizes in this case. Hence
the expected stationarity average is at least one for this member of
\(\mathcal C_p\), the objective \(f_n\) from
\eqref{eq:expected-lower-staircase-objective} with \(n\) from
\eqref{eq:expected-lower-staircase-height-index}. Together with the
quadratic case \eqref{eq:expected_lower_quadratic} and the fact that
\(\epsilon(1-\beta_1)^2/[2^{19}(1+\sqrt{\binit}+\epsilon)^2]<1\), this
proves \eqref{eq:expected_stationarity_lower_bound}.
\end{proof}

\end{document}